\documentclass[10pt,conference]{IEEEtran}

\usepackage[utf8]{inputenc} % allow utf-8 input
\usepackage[T1]{fontenc}    % use 8-bit T1 fonts
\usepackage{hyperref}       % hyperlinks
\usepackage{url}            % simple URL typesetting
\usepackage{booktabs}       % professional-quality tables
\usepackage{amsfonts}       % blackboard math symbols
\usepackage{nicefrac}       % compact symbols for 1/2, etc.
\usepackage{microtype}      % microtypography
\usepackage{wrapfig}

\usepackage[numbers]{natbib}
\usepackage{graphicx}
\usepackage{siunitx}
\usepackage{url}

\usepackage{support-caption}
\usepackage{subcaption}

\usepackage{complexity}
\usepackage{dsfont}
\usepackage{nicefrac}
\usepackage{multirow}
\usepackage{microtype}
\usepackage{graphicx}
\usepackage{amsthm, amsmath, amssymb}
\usepackage{xspace}
\usepackage{mathtools}
\usepackage{bbm}
\usepackage{fancyhdr}
\usepackage{cleveref}
\usepackage{xcolor}
\usepackage{url}
\usepackage{xr}
\usepackage{float}
\usepackage{threeparttable}
\usepackage{enumitem}
\usepackage{algorithm}
\usepackage{algorithmic}

\usepackage{dsfont}

\def\ddefloop#1{\ifx\ddefloop#1\else\ddef{#1}\expandafter\ddefloop\fi}
\def\ddef#1{\expandafter\def\csname bb#1\endcsname{\ensuremath{\mathbb{#1}}}}
\ddefloop ABCDEFGHIJKLMNOPQRSTUVWXYZ\ddefloop
\def\ddef#1{\expandafter\def\csname c#1\endcsname{\ensuremath{\mathcal{#1}}}}
\ddefloop ABCDEFGHIJKLMNOPQRSTUVWXYZ\ddefloop

\DeclareMathOperator*{\argmin}{arg\,min}
\DeclareMathOperator*{\argmax}{arg\,max}

\def\E{\mathbb{E}}

\def\1{\mathds{I}}

\def\supp{\textup{supp}}

\newtheorem{theorem}{Theorem}

\theoremstyle{definition}

\theoremstyle{remark}
\newtheorem*{remark}{Remark}

\newcommand{\Att}{\mathcal{A}}

\newtheorem{lem}{Lemma}

\theoremstyle{plain}

\newcommand{\para}[1]{{\vspace{1pt} \bf \noindent #1 \hspace{3pt}}}

\theoremstyle{plain}
\newtheorem{Thm}{Theorem}
\crefname{Thm}{Theorem}{Theorems}
\newtheorem{Df}{Definition}

\newtheorem*{theorem*}{Theorem}

\newtheorem{Lem}[Thm]{Lemma}

\newcommand{\term}[1]{\emph{#1}}

\renewcommand{\Pr}[1]{\mathrm{Pr}\left(#1\right)}

\newcommand{\Attack}[1]{\mathcal{A}\left(#1\right)}

\newcommand{\model}{{\mathcal{F}}}
\newcommand{\modelt}{{\mathcal{G}}}

\newcommand{\el}{{\ell_{\model}}}
\newcommand{\lt}{{\ell_{\modelt}}}

\newcommand{\xA}{x^{\mathcal{A}}}
\usepackage{enumitem}

\newcount\Comments  % 1 suppresses notes to selves in text
\usepackage{color}
\definecolor{darkgreen}{rgb}{0,0.5,0}
\definecolor{darkblue}{rgb}{0,0,0.5}
\definecolor{purple}{rgb}{1,0,1}
\newcommand{\kibitz}[2]{\ifnum\Comments=0\textcolor{#1}{#2}\fi}

\title{
    TRS: Transferability Reduced Ensemble via Encouraging Gradient Diversity and Model Smoothness
}

\author{%
}

\begin{document}

\maketitle

% add page number for the draft version
\ifnum\Comments=0
\thispagestyle{plain}
\pagestyle{plain}
\fi

\begin{abstract}
\emph{Adversarial Transferability} is an intriguing property --  adversarial perturbation crafted against one model is also effective against another model, while these models are from different model families or training processes.
% Such a property enables blackbox attacks that do not require direct access to the target victim model.
To better protect ML systems against adversarial attacks, several questions are raised: \textit{what are the sufficient conditions for adversarial transferability and how to bound it? }
% Is it possible to bound such transferability? 
Is\textit{ there a way to reduce the adversarial transferability in order to improve the robustness of an ensemble ML model?}
To answer these questions, in this work we first theoretically analyze and outline sufficient conditions for  adversarial transferability between models; then propose a practical algorithm to reduce the transferability between base models within an ensemble to improve its robustness.
Our theoretical analysis shows that only promoting the orthogonality between gradients of base models is not enough to ensure low transferability; in the meantime, the model smoothness is an important factor to control the transferability.
We also provide the lower and upper bounds of adversarial transferability under certain conditions. 
% We demonstrate that under the condition of gradient orthogonality, smoother classifiers will guarantee lower adversarial transferability. 
Inspired by our theoretical analysis, we propose an effective {\textbf{T}ransferability \textbf{R}educed \textbf{S}mooth} (TRS) ensemble training strategy to train a robust ensemble with low transferability by enforcing both gradient orthogonality and model smoothness between base models.
We conduct extensive experiments on TRS and compare with 6 state-of-the-art ensemble baselines against 8 whitebox attacks on different datasets, 
demonstrating that the proposed TRS outperforms all baselines significantly. 
% We believe our analysis on adversarial transferability will not only provide further understanding on predictions of ML models, but also
% inspire future research towards developing robust ML models taking these adversarial transferability properties into account.
\end{abstract}

% %%
% %% The code below is generated by the tool at http://dl.acm.org/ccs.cfm.
% %% Please copy and paste the code instead of the example below.
%   <concept>
%       <concept_id>10002978.10003022.10003023</concept_id>
%       <concept_desc>Security and privacy~Software security engineering</concept_desc>
%       <concept_significance>300</concept_significance>
%       </concept>
%   <concept>
%       <concept_id>10010147.10010257.10010293.10010294</concept_id>
%       <concept_desc>Computing methodologies~Neural networks</concept_desc>
%       <concept_significance>500</concept_significance>
%       </concept>
%   <concept>
%       <concept_id>10010147.10010257.10010293.10003660</concept_id>
%       <concept_desc>Computing methodologies~Classification and regression trees</concept_desc>
%       <concept_significance>300</concept_significance>
%       </concept>
%  </ccs2012>

% \ccsdesc[300]{Security and privacy~Software security engineering}
% \ccsdesc[500]{Computing methodologies~Neural networks}
% \ccsdesc[300]{Computing methodologies~Classification and regression trees}

% %%
% %% Keywords. The author(s) should pick words that accurately describe
% %% the work being presented. Separate the keywords with commas.
% \keywords{Robustness, Ensemble, Adversarial Transferability}

% \maketitle

% ===================== Introduction =====================
\section{Introduction}
    
    Machine learning systems, especially those based on deep neural networks (DNNs),
    have been widely applied in numerous applications~\cite{krizhevsky2012imagenet,hannun2014deep,sutskever2014sequence,choi2020attention}.
    % , including image recognition, speech recognition~\cite{}, natural language processing~\cite{}, 
    % and malware detection~\cite{}.
    However, recent studies show that DNNs are vulnerable to adversarial examples, which are able to mislead DNNs by adding small magnitude of perturbations to the original instances~\cite{szegedy2014intriguing,goodfellow2014explaining,xiao2018spatially,xiao2018generating}. 
    % There have also been a number of efforts exploring adversarial examples in general machine learning systems beyond those on DNNs~\cite{barreno2010security,biggio2013evasion,li2014feature,li2015scalable, fawzi2015analysis}. 
    Several attack strategies have been proposed so far to generate such adversarial examples in both digital and physical environments~\cite{moosavi2016universal,lin2017tactics,chaowei2018,chaowei2018spatially,ivan2018robust,kurakin2016adversarial}.
    % % Goodfellow \emph{et al.}~\cite{goodfellow2014explaining} introduced the fast gradient sign method which is able to generate these adversarial perturbations with minimal computational requirements on the adversary. Moosavi-Dezfooli \emph{et al.}~\cite{moosavi2016universal} have also succeeded in creating `universal perturbations' for image data which can be added to various images to mislead the learner. Attacks have also demonstrated the effectiveness for reinforcement learning systems~\cite{lin2017tactics,kos2017delving}.
    % For instance, several work has shown that such adversarial examples are effective in the physical world~\cite{kurakin2016adversarial,ivan2018robust}. It is thus clear that such adversarial attacks are a common risk across different learning systems and may exist in different domains. 
    Intriguingly, though most attacks require access to the target models (whitebox attacks), several studies show that adversarial examples generated against one model are able to  \emph{transferably} attack another target model with high probability, giving rise to blackbox attacks~\cite{papernot2016transferability,papernot2016practical,li2021nonlinear,li2020qeba,zhang2021progressive}. 
    % The universal adversarial perturbations~\cite{moosavi2015deepfool} mentioned earlier have also been found to transfer across classifiers. 
    This property of \textit{adversarial transferability} poses great threat to DNNs.
    
    Some work have been conducted to understand \textit{adversarial transferability}~\cite{tramer2017space,liu2016delving,demontis2019adversarial}. However, a rigorous theoretical analysis or explanation for transferability is still lacking in the literature. In addition, although developing robust ensemble models to limit transferability shows great potential towards practical robust learning systems, only \textit{empirical} observations have been made in this line of research~\cite{pang2019improving,kariyappa2019improving,yang2020dverge}.
    \emph{Can we deepen our 
    theoretical understanding on 
    transferability? Can we  take advantage 
    of rigorous theoretical understanding 
    to reduce the adversarial transferability and therefore generate robust ensemble ML models?}

    \begin{wrapfigure}{r}{0.45\textwidth}
        \centering
        \includegraphics[width=0.48\textwidth]{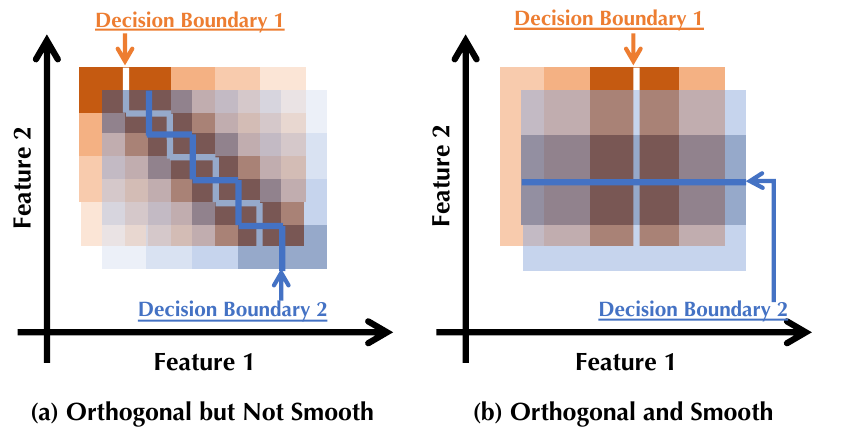}
        \vspace{-2em}
        \caption{\small An illustration of the relationship between \emph{adversarial transferability}, \emph{gradient orthogonality}, and \emph{model smoothness}.
        (a) Gradient orthogonality alone
        cannot minimize transferability
        as the decision boundaries between
        two classifiers can be arbitrarily close yet  orthogonal
        almost everywhere; (b)
        Gradient orthogonality
        with model smoothness provides a stronger guarantee on 
       model diversity, as our 
        theorems will show.
        }
        \label{fig:transfer}
        \vspace{-0.5em}
    \end{wrapfigure}

    In this paper, we focus on these two questions.
    From the theoretical side, we 
    are interested in the sufficient conditions under
    which the adversarial transferability can be
    \textit{lower bounded} and \textit{upper bounded}.
    % , both of which could 
    % lead to insights that could have
    % profound empirical implications: 
    % \textit{An upper bound on transferability could potentially deliver a new optimization
    % objective when training robust ensemble models, while a lower bound on 
    % transferability could help avoid 
    % ``doomed'' scenarios}.
% 
% 
    % \para{Intuition.} 
    Our theoretical arguments provides the \textit{first} theoretical interpretation for the sufficient conditions of transferability. Intuitively, as illustrated in Figure~\ref{fig:transfer}, we show that the commonly used gradient orthogonality (low cosine similarity) between learning models~\cite{demontis2019adversarial} cannot directly imply low adversarial transferability; on the other hand, orthogonal and smoothed models would limit the transferability.
    In particular,
    % with 
    % a focus on understanding the impact 
    % of \textit{model smoothness} and \textit{gradient cosine similarity}
    % on transferability.
        % In particular,
        % the existing work on understanding the adversarial transferability focuses on the empirical evaluation, which mainly rely on minimizing the gradient cosine similarity;
        % while other factors such as model smoothness are not outlined or studied.
        % For instance, \citeauthor{pang2019improving}~\cite{pang2019improving} propose to diversify the models' output confidence and empirically shows that it limits the transferability to some extent.
        % \citeauthor{yang2020dverge}~\cite{yang2020dverge} propose to diversify the distilled features learned from the models by penalizing feature overlaps during training.
        % However, there is no explanation on how the models evolve to be diversified quantitatively.
        % A large-scale empirical study~\cite{demontis2019adversarial}
        % summarizes three main factors that affect the transferability: the magnitude of input gradients, the gradient similarity, and the variability of the loss landscape, while the theoretical justifications of these factors are still open.
        % In this paper, 
        we prove that the \textit{gradient similarity} and \textit{model smoothness} are the key factors that both contribute to the adversarial transferability, and smooth models with orthogonal gradients can guarantee low transferability. 
        % We  verify our theoretical analysis based on extensive experiments.
        % Interestingly, from another point of view, we identify that the model smoothness we analyzed here is actually the precise key factor covered by the broad definition based on the large-scale empirical studies in~\cite{demontis2019adversarial}: the variability for the loss landscape or the so-called ``complexity'' of the model, since well-trained smoother models have smaller variance in its loss landscape and smaller model complexity.

    % \para{Lower Bounds.} 
    % In addition,
    % based on the Taylor expansion and Markov's inequality,
    % we derive a lower bound 
    % on transferability
    % between two low risk classifiers based on model smoothness and loss gradient similarity for the $\ell_2$ norm bounded adversarial examples. 
    % Beyond such common attack types, we 
    % also derive a lower bound for transferability from the perspective of data distribution distance, \ie under the condition that the distance between the benign and adversarial distributions is bounded. 
    % Both lower bounds illuminate mechanisms of transferability under different scenarios. 
    % % 
    % % 
    % % to mitigate ``doomed" scenarios where robustness is impossible.
    % % From the model perspective, we prove that the lower bound of transferability is bounded by the learning risks and the loss gradient similarity between different models. 
    % % 
    % % \para{Upper Bounds.} 
    % Besides, we also prove an upper bound on transferability, demonstrating that transferability 
    % can be upper bounded by model smoothness, model risks, and the similarity of their loss gradients.
    
    Under an empirical lens,
    inspired by our theoretical analysis, we propose a simple yet effective approach, {\textbf{T}ransferability \textbf{R}educed \textbf{S}mooth} (TRS) ensemble to limit adversarial transferability between base models within an ensemble and therefore improve its robustness. In particular, we reduce the loss gradient similarity between models as well as enforce the smoothness of models to introduce global model orthogonality.

    % Extensive experiments have shown that the proposed approach produces a robust ensemble which outperforms the state-of-the-art robust methods on multiple datasets against various adversarial attacks.
    
        We conduct extensive experiments to evaluate TRS in terms of the model robustness against different strong white-box and blackbox attacks following the robustness evaluation procedures~\cite{carlini2017towards,carlini2019evaluating,tramer2020adaptive},
        % \bo{cite carlini's work on evaluating blackbox}
        as well as its
        ability to limit transferability across the base models.
        We compare the proposed TRS with existing state-of-the-art baseline ensemble approaches  such as ADP~\cite{pang2019improving}, GAL~\cite{kariyappa2019improving}, and DVERGE~\cite{yang2020dverge}
        % We apply multiple strong whitebox attacks 
        % such as PGD~\cite{madry2017towards}, C\&W~\cite{carlini2017towards}, and EAD~\cite{chen2018ead} 
        % to simulate a powerful adversary and provide an empirically tighter upper bound of ensemble robustness. 
        on MNIST, CIFAR-10, and CIFAR-100 datasets, and we show that (1)~TRS achieves the state-of-the-art ensemble robustness, outperforming others by a large margin; (2)~TRS achieves efficient training; (3)~TRS effectively reduces the transferability among base models within an ensemble which indicates its robustness against whitebox and blackbox attacks; (4)~Both loss terms in TRS contribute to the ensemble robustness by constraining different  sufficient conditions of  adversarial transferability.

    \noindent\underline{\textbf{Contributions.}}
    In this paper, we make the first attempt towards theoretical understanding of adversarial transferability, and provide practical approach for developing robust ML ensembles.
    \begin{enumerate}[label=(\emph{\arabic*}),leftmargin=*,itemsep=-0.5mm]
        \vspace{-0.8em}
        \item 
        We provide a general theoretical analysis framework for adversarial transferability. We prove the lower and upper bounds of adversarial transferability. Both bounds show that the gradient similarity and model smoothness are the key factors contributing to the adversarial transferability, and smooth models with orthogonal gradients can guarantee low transferability.
        % between low risk classifiers for both $\ell_p$ norm bounded and distribution enabled adversarial examples.
    
        % \item
        % We prove an upper bound of transferability based on model similarity and smoothness, emphasizing the importance of model smoothness in decreasing the transferability between models, which is aligned with with the conclusions of existing large-scale empirical studies. We show that with smoother models, both the lower and upper bounds of transferability are tighter.
    
        \item
        We propose a simple yet effective approach TRS to train a robust ensemble by jointly reducing the loss gradient similarity between base models and enforcing the model smoothness.
        The code is publicly available\footnote{ \texttt{\href{https://github.com/AI-secure/Transferability-Reduced-Smooth-Ensemble}{https://github.com/AI-secure/Transferability-Reduced-Smooth-Ensemble}}}.

        \item
        We conduct extensive experiments to evaluate TRS in terms of model robustness under different attack settings, showing that TRS achieves the state-of-the-art ensemble robustness and outperforms other baselines by a large margin. We also conduct ablation studies to further understand the contribution of different loss terms and verify our theoretical findings.
        % \vspace{-0.8em}
    \end{enumerate}

    % \vspace{-0.5em}
\subsection*{Related Work}
    % \vspace{-0.7em}
    % \para{Adversarial vulnerabilities of deep neural networks.}
    % Since the early studies on adversarial examples against deep neural networks~\cite{szegedy2014intriguing}, the adversarial vulnerabilities of deep neural networks has raised great concerns these years, especially for security-critical scenarios, such as autonomous driving and drug design~\cite{kurakin2016adversarial,ivan2018robust,sitawarin2018darts}.
    % There has been a series of work on adversarial evasion attacks, including both whitebox and blackbox attacks~\cite{goodfellow2014explaining,carlini2016towards,athalye2018obfuscated,chaowei2018}. Corresponding defenses have also been extensively studied, including both empirical defenses~\cite{madry2017towards,samangouei2018defense} which are usually been adaptive attacked again~\cite{athalye2018obfuscated,tramer2020adaptive}, and provable robust ML defenses~\cite{kolter2017provable,cohen2019certified,li2020sok}). Besides, the explanation of such  vulnerabilities have also provided further understanding on the potential reasons and ways towards improving ML robustness ~\cite{tramer2017space,liu2016delving,demontis2019adversarial}.

    % \para{Adversarial transferability.}
    The adversarial transferability between different ML models is an intriguing research direction.  \citeauthor{papernot2016limitations}~\cite{papernot2016limitations} explored the limitation of adversarial  examples and showed that, while some instances are more difficult to manipulate than the others, these adversarial examples usually transfer from one model to another. 
    % It has also been found that the transferability property holds across different classification models, such as support vector machines, decision trees, and neural networks. This property exists not only across classifiers with the same underlying model, but also in the cross-model cases~\cite{papernot2016transferability}. 
    \citeauthor{demontis2019adversarial}~\cite{demontis2019adversarial} later analyzed transferability for both evasion and poisoning attacks.
    \citeauthor{tramer2017space}~\cite{tramer2017space} empirically investigated the subspace of adversarial examples that enables transferability between different models: though their results provide a non-zero probability guarantee on the transferability, they did not quantify the probability of adversarial  transferability. 
    % as well as the distance between the decision boundaries of classifiers in benign and adversarial directions.
    % They found that these decision boundaries often lie close to each other, which they hypothesize as the mechanism for transferability. They also provide a sufficient condition for transferability of adversarial examples generated using perturbations aligned with the difference of class means between classifiers which use `pseudo-linear' feature transformations. 
    % However, this category of transformations does not include CNNs. 
    % Further, their result only provides a non-zero probability guarantee on the transferability of these samples and does not quantify the probability with which adversarial examples transfer.
    
    % \para{Transferability based attacks and diversified robust ensemble as defenses.}
    % Given that the adversarial transferability implies the vulnerability of machine learning models even without being accessed by attackers,
    Leveraging the transferability, different blackbox attacks have been proposed~\cite{papernot2016practical,kurakin2016adversarial,ivan2018robust,cheng2019improving}. 
    % In addition, different query based 
    To defend against these transferability based attacks, \citeauthor{pang2019improving}~\cite{pang2019improving} proposed a class entropy based adaptive diversity promoting approach to enhance the ML ensemble robustness. 
        Recently, \citeauthor{yang2020dverge}~\cite{yang2020dverge} proposed DVERGE, a robust ensemble training approach that diversifies the non-robust features of base models via an adversarial  training  objective function.
        However, these approaches do not provide theoretical justification for adversarial transferability, and there is still room to improve the ML ensemble robustness based on in-depth understanding on the sufficient conditions of transferability.  In this paper, we aim to provide a theoretical understanding of transferability, and empirically compare the proposed robust ML ensemble inspired by our theoretical analysis  with existing approaches to push for a tighter empirical upper bound for the ensemble robustness.

% ===================== Theory =====================

\section{Transferability of Adversarial Perturbation}
    
    \label{sec:transferability-analysis}
    
    In this section, we first introduce preliminaries,  and then provide the upper and lower bounds of adversarial transferability by connecting adversarial transferability with 
    different characteristics of models theoretically, which, in the next section, will allow us to explicitly minimize 
    transferability by enforcing (or rewarding) certain properties of models.
    
    %     In this section, we will connect adversarial transferability with 
    % % different characteristics of models theoretically, which, in the next section, will allow us to explicitly minimize 
    % transferability by enforcing (or rewarding) certain properties of models.
    % Our theoretical analysis in \Cref{sec:transferability-analysis} will follow the same threat model defined here.

    % \subsection{Preliminaries and Threat Model}
        \noindent\textbf{Notations.} We consider neural networks for classification tasks.
        Assume there are $C$ classes, and
        let $\cX$ be the \emph{input space} of the model with $\cY = \{1,2,\dots,C\}$ the set of prediction classes~(i.e., labels).
        We model the neural network by a mapping function $\cF: \cX \to \cY$.
        % (e.g. the set of all possible images for an image classification task).
        % Let $\cY$ denote the set of classes~(i.e., labels) the models would output.
        % If there are $C$ classes, $\cY = \{1,2,\dots,C\}$.
        % The neural network is modeled as a mapping function $\cF$ (or $\cG$) from $\cX$ to $\cY$.
        We will study the transferability between two models $\cF$ and $\cG$.
        For brevity, hereinafter we mainly show the derived notations for $\cF$ and notations for $\cG$ are similar.
        % which are similar $\cG$.
        Let the \emph{benign} data $(x,y)$ follow an unknown distribution $\cD$ supported on $(\cX, \cY)$, and $\cP_{\cX}$ denote the marginal distribution on $\cX$.
        % Under this distribution $\cD$, we use $\cP_{\cX}$ to represent the marginal distribution on $\cX$, i.e., marginal distribution of input data. 
        % The $\Pr{E}$ and $\E[v]$ represent the probability of an event $E$ and the expected value of a random variable $v$ under distribution $\cD$, respectively.
        
        For a given input $x\in\cX$, the classification model $\cF$ first predicts the confidence score for each label $y\in \cY$, denoted as $f_y(x)$.
        These confidence scores sum up to $1$, i.e., $\sum_{y\in \cY} f_y(x) = 1, \forall x\in \cX$. 
        The model $\cF$ will predicts the label with highest confidence score: $\cF(x) = \argmax_{y\in \cY} f_y(x)$.
        
        For model $\cF$, there is usually a model-dependent loss function $\ell_\cF: \cX \times \cY \to \bbR_+$, which is the composition of a differentiable training loss~(e.g., cross-entropy loss) $\ell$ and the model's confidence score $f(\cdot)$: $\ell_\cF(x,y) := \ell(f(x), y), (x,y) \in (\cX, \cY)$.
        % For instance, when $\ell$ is the cross-entropy loss, $\ell_\cF(x,y) = -\log f_y(x)$.
        We further assume that $\cF(x) = \argmin_{y\in \cY} \ell_{\cF}(x,y)$, i.e., the model predicts the label with minimum loss.
        This holds for common training losses.
        % Neural network models are usually differentiable, i.e., for any $(x,y) \in (\cX,\cY)$, $\nabla_x f_y(x)$ exists, and so as $\nabla_x \ell(f(x),y)$.
        
        In this paper, by default we will focus on models that are well-trained on the benign dataset, and such models are the most commonly encountered in practice, so their robustness is paramount. This means we will focus on the \textit{low risk} classifiers, which we will formally define in \Cref{subsec:model-characteristics}.

        % Throughout the paper, we focus on the data evasion attacks, where the attacker tries to craft an adversarial example that fools the model to make wrong predictions.
        % It is worth mentioning that through optimizing framework in \cite{demontis2019adversarial} it is also possible to extend our analysis to data poisoning attacks where the attacker maliciously manipulates the training data to cause model failures, and we leave this as interesting future directions.
        
        % \noindent\textbf{Transferability based attacks.}
        
        {\em How should we define an adversarial attack?}
        For the threat model, we consider the attacker that adds an $\ell_p$ norm bounded perturbation to data instance $x \in \cX$.
        In practice, there are two types of attacks, \textit{untargeted attacks} and \textit{targeted
        attacks}. 
        % As previous work observed, the
        The definition of adversarial transferability is slightly different under these
         attacks~\cite{liu2016delving}, and we consider
        both in our analysis.
        
        \begin{Df}[Adversarial Attack]
            \label{def:untarget}
            \label{def:target}
            Given an input $x\in\cX$ with true label $y\in\cY$, $\model(x) = y$. 
            (1)~An untargeted attack crafts $\Att_U(x) =  x + \delta$ to maximize
                % $$
                %   \Att_U(x) =  x + \delta \in \underset{\delta: \lVert\delta\rVert_p \le \epsilon}{\argmax}\,\, \ell_\cF(x+\delta,y).
                % $$
                $ 
                    \ell_\cF(x+\delta,y)
                $ where $\|\delta\|_p \le \epsilon$.
            (2)~A targeted attack with target label $y_t \in \cY$ crafts $\Att_T(x) = x + \delta$ to minimize
                % satisfies that
                % $$
                %   \Att_T(x)=  x + \delta \in \underset{{\delta: \lVert\delta\rVert_p \le \epsilon}}{\argmin}\,\, \ell_\cF(x+\delta,y_t).
                % $$
                $
                    \ell_\cF(x+\delta,y_t)
                $ where $\|\delta\|_p \le \epsilon$.
            
            % \begin{itemize}
            %     \item 
            %     An untargeted attack crafts $\Att_U(x) =  x + \delta$ to maximize
            %     % $$
            %     %   \Att_U(x) =  x + \delta \in \underset{\delta: \lVert\delta\rVert_p \le \epsilon}{\argmax}\,\, \ell_\cF(x+\delta,y).
            %     % $$
            %     $ 
            %         \ell_\cF(x+\delta,y)
            %     $ where $\|\delta\|_p \le \epsilon$.
                
            %     \item 
            %     A targeted attack with adversarial target $y_t \in \cY$ crafts $\Att_T(x) = x + \delta$ to minimize
            %     % satisfies that
            %     % $$
            %     %   \Att_T(x)=  x + \delta \in \underset{{\delta: \lVert\delta\rVert_p \le \epsilon}}{\argmin}\,\, \ell_\cF(x+\delta,y_t).
            %     % $$
            %     $
            %         \ell_\cF(x+\delta,y_t)
            %     $ where $\|\delta\|_p \le \epsilon$.
            % \end{itemize}
        \end{Df}

        In this definition, 
        % $\|\delta\|_p$ represents the $\ell_p$ norm of $\delta$, and 
        $\epsilon$ is a pre-defined \emph{attack radius} that limits the power of the attacker.
        We may refer to $\{\delta: \|\delta\|_p \le \epsilon\}$ as the perturbation ball.
        The goal of the untargeted attack is to maximize the loss of the target model against its true label $y$.
        The goal of the targeted attack is to minimize the loss towards its adversarial target label $y_t$.

        {{\em How do we formally define that an attack is effective?}}
        % We define the attack effectiveness for both targeted and untargeted attacks based on the statistical probability of successful attacks.
        
        \begin{Df}[($\alpha,\model$)-Effective Attack]
            \label{def:adv_effect}
            Consider a input $x\in \cX$ with true label $y\in \cY$.
            %  $x$ is drawn uniformly from the distribution $\cX$.
            An attack is ($\alpha, \model$)-effective in {untargeted} scenario if $\Pr{\model(\Att_U(x)) \ne y} \ge 1 - \alpha$.
            An attack is ($\alpha, \model$)-effective in {targeted} scenario (with class target $y_t$) if $\Pr{\model(\Att_T(x)) = y_t} \ge 1 - \alpha$. 
            %  both untargeted and targeted (with class target $y_t$) scenarios if:
            % %   \vspace{-2mm}
            % \begin{itemize}
            % \item {Untargeted}: $\Pr{\model(\Att_U(x)) \ne y} \ge 1 - \alpha$. 
            % \item {Targeted}:
            %   $\Pr{\model(\Att_T(x)) = y_t} \ge 1 - \alpha$. 
            % \end{itemize}
        \end{Df}

% \begin{wrapfigure}{r}{0.675\linewidth}
%     \centering
%     \vspace{-1em}
%     \includegraphics[width=1.0\linewidth]{figures/trs-method-overview.png}
%     % \includegraphics[width=0.75\textwidth]{figures/trs-pip.png}
%     \vspace{-1em}
%     \caption{\small The overview for the Transferability Reduced Smoothing ensemble (TRS).\TODO \linyi{need to update section references} }
%     \label{fig:method-overview}
%     \vspace{-1.0em}
% \end{wrapfigure}

        This definition captures the requirement that an adversarial instance generated by an effective attack strategy is able to mislead the target classification model~(e.g. $\cF$) with certain probability $(1-\alpha)$.
        % In untargeted attacks, provided the prediction on an adversarial instance  differs from the ground truth, the attack is considered successful.
        % In targeted attacks, the attack succeeds only when the adversarial instance can be mis-recognized as the specific adversarial target label $y_t$.
        The smaller the $\alpha$ is, the more effective the attack is.
        In practice, this implies that on a finite sample of targets, the attack success is frequent but not absolute.
        % 
% 
        % Note that in the literature, the attack strategy is further divided into whitebox attack and blackbox attack based on the degree of information known by the attacker~\cite{athalye2018obfuscated,demontis2019adversarial,carlini2017towards}, where the \emph{whitebox attack} assumes the full-knowledge of the target model including model structure and weights, and the \emph{blackbox attack} only knows the confidence score~\cite{papernot2017practical} or the label decision~\cite{brendel2017decision}.
        % While in our analysis, we only need to assume the effectiveness of the attack on some surrogate model $\cF$, then we can derive the upper and lower bound on its transferability on the target model $\cG$, no matter whether it is a whitebox or blackbox attack.
        % 
        Note that the definition is general for both whitebox~\cite{athalye2018obfuscated,demontis2019adversarial,carlini2017towards} and blackbox attacks~\cite{papernot2017practical,brendel2017decision}. 
        % In {whitebox attack}, the attacker is assumed to have full-knowledge about the target model's structure and weights~\cite{athalye2018obfuscated,demontis2019adversarial,carlini2017towards}; while in \emph{blackbox attack} only the prediction probability score~\cite{papernot2017practical} or the label decision can be queried by the attacker~\cite{brendel2017decision}.
        % % In both types of attacks, we can use such successful attack probability to quantify the attack effectiveness naturally.
        % % 
        % In this paper, we will focus on both targeted and untargeted adversarial effective attacks, as well as their transferability between different models.% (e.g. \cF~and \cG)

        \vspace{-0.5em}
%  \subsection{Method Overview}
       
        % \paragraph{Method overview.}
        % In \Cref{sec:transferability-analysis}, 
        % we will present the theoretical lower and upper bounds of adversarial transferability for both untargeted and targeted attacks.
        % From our theoretical analysis, we find that the diversified gradients and smoother model jointly contribute to low transferability across models.
        % In \Cref{sec:algo}, we encode these two factors into the proposed TRS training to minimize the transferability between base models during training.
        % We combine the trained base models as an ensemble---TRS ensemble, which is extensively evaluated in \Cref{sec:exp}.

% \section{Transferability of Adversarial Perturbation}

    % \label{sec:transferability-analysis}

    \subsection{Model Characteristics}
        \label{subsec:model-characteristics}
    % We first introduce some
    % necessary notations and definitions, then present the theoretical analysis of adversarial transferability.
    
        {\em Given two models $\mathcal{F}$ and 
        $\mathcal{G}$, what are the characteristics of
        $\mathcal{F}$ and $\mathcal{G}$ that have
        impact on transferability under a given 
        attack strategy?} Intuitively, the more
        similar these two classifers are, the larger
        the transferability would be. 
        However, \textit{how can 
        we define ``similar'' and how can we 
        rigorously connect it to transferability?}
        To answer these questions, we will first define the risk and empirical risk for a given model to measure its performance on benign test data.
        Then, as the DNNs are differentiable, we will define model similarity based on their gradients.
        We will then derive the lower and upper bounds of adversarial transferability based on the defined model risk and similarity measures.

        \begin{Df}[Risk and Empirical Risk]
            For a given model $\cF$, we let $\ell_{\cF}$ be its model-dependent loss function.
            Its \textbf{risk} is defined as $\eta_{\cF} = \Pr{\cF(x) \neq y}$;
            and its \textbf{empirical risk} is defined as $\xi_{\cF} = \E\left[\ell_{\cF}(x,y)\right]$.
            \label{def:risk}
        \end{Df}
        
        The \emph{risk} represents the model's error rate on benign test data, while the \emph{empirical risk} is a non-negative value that also indicates the inaccuracy.
        For both of them, higher value means worse performance on the benign test data.
        The difference is that, the risk has more intuitive meaning, while the empirical risk is differentiable and is actually used during model training. % in practice. 

            \begin{Df}[Loss Gradient Similarity]
                \label{def:similarity}
                The {lower loss gradient similarity} $\underline{\mathcal{S}}$ and upper loss gradient similarity $\overline{\mathcal{S}}$ between two differentiable loss functions $\el$ and
                $\lt$ is defined as:
                $$ 
                    \resizebox{\linewidth}{!}{
                    $
                        \displaystyle
                     \underline{\mathcal{S}}(\el, \lt) = \inf_{x \in \cX, y \in \mathcal{Y}} \frac{\nabla_x \el(x,y) \cdot 
                       \nabla_x \lt(x,y)}{\lVert \nabla_x \el(x,y) \rVert_2 \cdot \lVert \nabla_x
                       \lt(x,y) \rVert_2},
                    \overline{\cS}(\ell_\model, \ell_\modelt) = \sup_{x\in\cX, y\in\cY}\frac{\nabla_x\ell_\model(x, y)\cdot \nabla_x\ell_\modelt(x, y)}{\|\nabla_x\ell_\model(x, y)\|_2\cdot\|\nabla_x\ell_\modelt(x,y)\|_2}.
                    $
                    }
                $$
                \label{defn:loss-gradient-similarity}
                \label{defn:upper-loss-gradient}
            \end{Df}
            \vspace{-1.5em}
            The $\underline{\mathcal{S}}(\el, \lt)$~($\overline{\mathcal{S}}(\el, \lt)$) is the minimum~(maximum) cosine
            similarity between the gradients of the two loss
            functions for an input $x$ drawn from $\cX$ with any label $y \in \mathcal{Y}$. 
            % Great $\underline{\mathcal{S}}$ indicates a larger and tighter lower bound on loss gradient similarity.
            % 
            % Analogous with the lower loss gradient similarity, we define the upper loss gradient similarity, which is the maximum cosine similarity between the gradients of the two loss functions.
            % \begin{Df}[Upper Loss Gradient Similarity]
            %     The {upper loss gradient similarity} $\overline{\cS}$ between two differentiable loss functions $\ell_\model$ and $\ell_\modelt$ is defined as:
            %     $$
            %         \overline{\cS}(\ell_\model, \ell_\modelt) = \sup_{x\in\cX, y\in\cY}\frac{\nabla_x\ell_\model(x, y)\cdot \nabla_x\ell_\modelt(x, y)}{\|\nabla_x\ell_\model(x, y)\|_2\cdot\|\nabla_x\ell_\modelt(x,y)\|_2}
            %     $$
            % \end{Df}
            % \vspace{-1em}
            % For upper loss gradient similarity we care about the maximum cosine similarity between the two loss function gradients.
            % Smaller $\overline{\cS}$ indicates a lower and tighter upper bound for the loss gradient similarity.
            % In contrary to gradient similarity, here smaller $\overline{\cS}$ means larger gradient dissimilarity.
            Besides the loss gradient similarity, in our analysis we will also show that the \textit{model smoothness} is another key characteristic of ML models that affects the model transferability.
            % We define the model smoothness as below.
            \begin{Df}
                \label{def:smoothness}
                We call a model $\cF$ $\beta$-smooth if 
                % The model pair $(\cF,\cG)$ is said to be $\beta$-smooth if $\beta$ satisfies
                $ \displaystyle
                    \sup_{x_1, x_2 \in \cX, y \in \cY} \frac{\| \nabla_x \ell_\cF(x_1,y) - \nabla_x \ell_\cF(x_2,y) \|_2}{\| x_1-x_2 \|_2} \le \beta.
                $
            \end{Df}
            \vspace{-0.5em}
            % Note that in the following theorems, $\beta$-smooth is defined on the differentiable loss $\ell_\cF$ and $\ell_\cG$ instead of the 0-1 loss. 
            This smoothness definition is commonly used in deep learning theory and optimization literature~\cite{Boyd04,borwein2010convex}, and is also named curvature bounds in certified robustness literature~\cite{singla2020second}.
            It could be interpreted as the Lipschitz bound for the model's loss function gradient.
            We remark that \emph{larger} $\beta$ indicates that the model is less smoother, while \emph{smaller} $\beta$ means the model is smoother.
            Particularly, when $\beta = 0$, the model is linear in the input space $\cX$.

            % \textcolor{red}{Moreover, in the following analysis we use $\eta_\cF := \bbE \left[ \ell_\cF(x,y) \right]$, $\eta_\cG := \bbE \left[ \ell_\cG(x,y) \right]$ as the proxy of \emph{risk}.
            % In this way, we have a differentiable risk which simplifies theoretical analysis.}

            %similarity as the relative alignment of loss functions
            %between the target and surrogate models. We will then show the upper bound of the model induced transferability is bounded as a function of the gradient similarity as well as the model smoothness in Section~\ref{sec:model-upperbound}.
            
            %When the data distribution density is low, other factors besides the data distribution distance would become important for inducing transferability such as the ML model itself. Here we will focus on the ML models including the loss function and model smoothness. 
            %The model induced transferability is extrinsically caused by the fact that both classifiers are low risk and non-smooth. 
            %As we have described, adversarial transfer attacks
            %successfully occur when an adversarial example crafted against a
            %surrogate model fools a similar target model. From the model perspective, we defined this
            %similarity as the relative alignment of loss functions
            %between the target and surrogate models. We will then show the upper bound of the model induced transferability is bounded as a function of the gradient similarity as well as the model smoothness in Section~\ref{sec:model-upperbound}.

    \subsection{Definition of Adversarial Transferability}
        \label{subsec:def-transferability}
    
        Based on the model characteristics we explored above, next we will ask:
        {\em Given two models, what is the natural and precise definition of adversarial transferability?} 
        % We propose the following definition.

        % Before we present our result, we first need to formally define transferability.
        % % and relevant characteristics of 
        % % models. 
        % % We then connect transferability
        % % with these properties.
        
        \begin{Df}[Transferability]
        \label{def:trans}
        Consider an adversarial instance $\Att_U(x)$ or $\Att_T(x)$ constructed against a surrogate
        model $\model$. With a given benign input $x\in \cX$, The {transferability} $T_r$ between
        $\model$ and a target model $\modelt$ is defined as follows (adversarial target $y_t \in \cY$):
            \begin{itemize}%[leftmargin=*]
            \item Untargeted: 
            % \vspace{-0.25em}
            $
                % \begin{aligned}
                    T_r(\model, \modelt, x) = \1[ \model(x) = \modelt(x) = y \; \wedge 
                     \model(\Att_U(x)) \neq y \; \wedge \; \modelt(\Att_U(x)) \neq y].
                % \end{aligned}
            $
            % \vspace{-0.5em}
            \item Targeted: 
            % \vspace{-0.25em}
            $
                % \begin{aligned}
                T_r(\model, \modelt, x, y_t) = \1[ \model(x) = \modelt(x) = y \; \wedge 
                 \model(\Att_T(x)) = \modelt(\Att_T(x)) = y_t].
                % \end{aligned}
            $
            \end{itemize}
        \end{Df}

        Here we define the transferability at instance level, showing several conditions are required to satisfy for a transferable instance.
        For the untargeted attack, it requires that: (1)~both the surrogate model and target model make correct prediction on the benign input; and (2)~both of them make incorrect predictions on the adversarial input $\Att_U(x)$.
        The $\Att_U(x)$ is generated via the untargeted attack against the surrogate model $\cF$.
        For the targeted attack, it requires that: (1)~both the surrogate and target model make correct prediction on benign input; and (2)~both output the adversarial target $y_t\in\cY$ on the adversarial input $\Att_T(x)$. 
        The $\Att_T(x)$ is crafted against the surrogate model $\cF$.
        The predicates themselves do not require $\Att_U$ and $\Att_T$ to be explicitly constructed against the surrogate model $\cF$. It will be implied by attack effectiveness~(\Cref{def:adv_effect}) on $\model$ in theorem statements. 
        Note that the definition here is a predicate for a specific input $x$, and in the following analysis we will mainly use its distributional version: $\Pr{T_r(\cF,\cG,x)=1}$ and $\Pr{T_r(\cF,\cG,x,y_t)=1}$.
        % , which measures the overall average transferability over a finite number of inputs for a given classification task.
        % % 
        % This definition of transferability directly encodes the probability that the crafted adversarial examples against a surrogate model transfer to the target model.
        % In contrast, the previous work~(e.g., \cite{demontis2019adversarial}) usually defines the transferability over the loss functions $\ell_\cF$ and $\ell_\cG$, which draws an indirect connection with the observed transferability phenomenon.

    \subsection{Lower Bound of Adversarial Transferability}
        \label{sec:lower-bound}
        Based on the general definition of transferability, in this section we will analyze how to lower bound the transferability for targeted attack.
        The analysis for untargeted attack has a similar form and is deferred to \Cref{thm:untarget} in \Cref{adx:sec-untargeted attack}.
        % between low risk classifiers.
        
        % Specifically, we show that if we measure the model transferability by
        % purely computing the cosine similarity of
        % the models' loss gradients evaluated on a particular input without constraining the model smoothness, the lower bound of transferability can be computed. 
        
        \begin{Thm}[Lower Bound on Targeted Attack Transferability]
        	\label{thm:target}

            Assume both models $\cF$ and $\cG$ are $\beta$-smooth. 
            Let $\Att_T$ be an ($\alpha, \model$)-effective targeted attack with perturbation ball $\lVert \delta \rVert_2 \le \epsilon$ and target label $y_t \in \cY$.
            The transferabiity can be lower bounded by
        % 	Let an instance $x \in \bbR^n$ with  true label $y$ and adversarial target $y_t$.
        % 	An ($\alpha, \model$)-effective (targeted) attack $\xA = \Att_T(x)$ with perturbation ball $\lVert \delta \rVert_2 \le \epsilon$ is transferable to $\modelt$ with bounded probability
        % 	
            % 	Let the surrogate $\model$ and the target model $\modelt$ be models with risk $\eta_\model$ and $\eta_\modelt$.
            % 	
                % \vspace{-1mm}
            % 	Suppose the loss function $\el$ and $\lt$ are both in differentiability class $C^2$.
            % 	Let $\gamma_\model$ and $\gamma_\modelt$ denote the largest absolute value of Hessian matrix eigenvalues of loss function $\el$ and $\lt$ within the perturbation ball respectively.
            % 	The probability that $x^*$ transfers from $\model$ to $\modelt$
            $$
                % \begin{aligned}
            	   % & 
            	    \Pr{T_r(\model, \modelt, x, y_t) = 1}  \ge (1-\alpha) - (\eta_\model + \eta_\modelt) - 
            	   % \\
            	   % & \hspace{2em} 
            	    \dfrac{\epsilon(1+\alpha) +  c_\model(1-\alpha)}{c_\modelt + \epsilon} - \dfrac{ \epsilon (1-\alpha)}{c_\modelt + \epsilon}\sqrt{2 - 2 \underline{\mathcal{S}}(\el, \lt)},
            % 	\end{aligned}
            $$
        	where 
        	\vspace{-0.75em}
        	$$
        	   % \begin{aligned}
        	   \resizebox{\linewidth}{!}{
        	   $
        	    \displaystyle
                	c_\model = \max_{x \in \cX} \dfrac{ \underset{y\in \cY}{\min}\, \el(\Att_T(x), y) - \el(x, y_t) + \beta\epsilon^2 / 2 }{ \lVert \nabla_x \el(x, y_t) \rVert_2 }, 
                % 	\\
        	        c_\modelt = \min_{x \in \cX} \dfrac{ \underset{y\in \cY}{\min}\, \lt(\Att_T(x), y) - \lt(x, y_t) - \beta\epsilon^2 / 2 }{ \lVert \nabla_x \lt(x, y_t) \rVert_2 }.
        	   $
        	   }
        	   %\end{aligned}
        	$$
            Here $\eta_\cF,\eta_\cG$ are the \emph{risks} of models $\cF$ and $\cG$ respectively.%, relative to the 0-1 loss.
            % Here the risk of $\model$ and $\modelt$ are $\eta_\model$ and $\eta_\modelt$ with loss functions $\el$ and $\lt$, respectively.
            % are both in differentiability class $C^2$.
            % $\gamma_\model$ and $\gamma_\modelt$ denote the largest absolute value of Hessian matrix eigenvalues of $\el$ and $\lt$ within the perturbation ball.
            % $\el$ and $\lt$ are the \emph{risk} of model $\cF$ and $\cG$ respectively.
        \end{Thm}
        % Due to space limitations we defer the proof to Appendix~D.
        % Similarly, we have the following theorem for untargeted attack.
        % Though seems pessimistic, this provides insights about what model properties affect the transferability and on the countrary encourage us to derive upper bound for it to improve model robustness. We will discuss the upper bound in Section~\ref{sec:model-upperbound}.
        
        We defer the complete proof in \Cref{aplowerbound}.
        In the proof, we first use a Taylor expansion to introduce the gradient terms, then relate the dot product with cosine similarity of the loss gradients, and finally use Markov's inequality to derive the misclassification probability of $\modelt$ to complete the proof.
        % \begin{proof}[Proof (Sketch)]
        %   The main challenge for lower bounding  transferability is to connect the misclassification of $\model$ to the misclassification of $\modelt$ purely using their loss gradient similarities.
        %   We first use a Taylor expansion to introduce the gradient terms and eliminate higher-order terms by $\beta$-bounded Hessian matrix eigenvalues.
        %   We then connect the gradient terms with the misclassification probability.
        %   To do so, we prove and apply \Cref{lem:cosineBound}~(\Cref{aplowerbound}) which relates the dot product with cosine similarities of the loss gradients.  Finally, we calculate the upper bound of the expectation of $\delta \cdot \frac{\nabla_x \lt(x, y_t)}{\lVert \nabla_x \lt(x, y_t) \rVert_2}$, and use Markov's inequality to derive the misclassification probability of $\modelt$, which complete the proof. The full proof of this result can be found in
        %   \Cref{aplowerbound}.%\ref{aplowerbound}. 
        % \end{proof}
      
        % In the above theorems, we use \emph{risk}  $\eta_{\cF}$ and $\eta_{\cG}$ to represent the model risk.
        % Compared with \emph{empirical risk}, the \textit{risk} is directly related to the model inaccuracy, and thus more informative and provides intuitive understanding. It is also possible to derive the lower bound based on the empirical risk which will be of slightly more complicated form.
        % , and we will not discuss it in details here.

        \para{Implications.}
        In \Cref{thm:target}, the only term which correlates both $\model$ and $\modelt$ is $\underline{\mathcal{S}}(\el, \lt)$, while all other terms depend on individual models $\model$ or $\modelt$.
        Thus, we study the relation between $\underline{\mathcal{S}}(\el, \lt)$ and $\Pr{T_r(\model, \modelt, x, y_t) = 1}$.
        % Thus, we can view all other terms as constant.
        
        Note that since $\beta$ is small compared with the perturbation radius $\epsilon$ and the gradient magnitude $\|\nabla_x \ell_{\cG}\|_2$ in the denominator is relatively large, the quantity $c_{\cG}$ is small. 
        Moreover, $1 - \alpha$ is large since the attack is typically effective against $\cF$.
        Thus, $\Pr{T_r(\model, \modelt, x, y_t) = 1}$ has the form $C - k\sqrt{1 - \underline{\mathcal{S}}(\el, \lt)}$, where $C$ and $k$ are both positive constants.
        We can easily observe the \emph{positive} correlation between the loss gradients similarity
        $\underline{\mathcal{S}}(\el, \lt)$, and
        lower bound of adversarial transferability $\Pr{T_r(\model, \modelt, x, y_t) = 1}$.
        % or $T_r(\model, \modelt,
        % x)$.
        
        In the meantime, note that when $\beta$ increases~(i.e., model becomes less smooth), 
        in the transferability lower bound $C - k\sqrt{1 - \underline{\mathcal{S}}(\el, \lt)}$, the $C$ decreases and $k$ increase.
        As a result, the lower bounds in \Cref{thm:target} decreases, which implies that when model becomes less smoother~(i.e., $\beta$ becomes larger), the transferability lower bounds become looser for both targeted and untargeted attacks.
        In other words, \emph{when the model becomes smoother, the correlation between loss gradients similarity and lower bound of transferability becomes stronger}, which motivates us to constrain the model smoothness to increase the effect of limiting loss gradients similarity.
        % The theorems formalize the intuition that \textit{low risk classifiers would potentially have high transferability} if their loss gradients similarity and smoothness are sufficiently high as observed in exiting work~(cf. \cite{demontis2019adversarial}).

        In addition to the $\ell_p$-bounded attacks, we 
        also derive a transferability lower bound for general attacks whose magnitude is bounded by total variance distance of data distributions.
        We defer the detail analysis and discussion to \Cref{adx:sec-tv-transferability-proofs}.
        % These lower bounds illuminate the mechanisms of transferability under different scenarios. 

        % Note that in the above theorems, $c_\modelt$ is close to $0$.
        % Therefore the transferability lower bound is positively correlated with lower loss gradient similarity.
        % Moreover, we observe that when the models become smoother~($\beta$ becomes smaller), the lower bound increases.
        % We attribute this to the increased correlation between transferability lower bound and model gradient similarity.
        
        % smoothness. 
        %This is further discussed in Section~\ref{sec:model-upperbound}.
        % there is direct proportionality between the lower bound probability of 
        % transferability and the loss gradient similarity. 
        % This direct proportionality of two measurable quantities
        % allows us to experimentally validate this theoretical result.
        % \para{Discussion}
        % It is obvious that the lower bound of transferability between two low risk models depends on the cosine similarity between the gradient of two loss functions if we do not constrain the model smoothness.
        
        % Though it seems pessimistic given the lower bound of transferability, we would ask another question: {\em Is it possible to upper bound the transferability given additional constraints?}
        % % To answer this question, we next analyze the upper bound of the model induced transferability given model smoothness.
    
    \vspace{-0.3em}
    \subsection{Upper Bound of Adversarial Transferability}
        \vspace{-0.3em}
        \label{sec:model-upperbound}
    
        We next aim to upper bound the adversarial transferability. The upper bound for target attack is shown below; and the one for untargeted attack has a similar form in \Cref{thm:untarget-upper-bound} in \Cref{adx:sec-untargeted attack}.
        % Suppose we have two models $\model$ and $\modelt$, then for any targeted adversarial attack, we can also upper bound the transferability of the two models by their loss gradients similarity and model smoothness. The upper bound for untargeted attack has a similar form and is deferred to \Cref{thm:untarget-upper-bound} in \Cref{adx:sec-untargeted attack}.
        
        \begin{Thm}[Upper Bound on Targeted Attack Transferability]
        	\label{thm:target-upper-bound}

            Assume both models $\cF$ and $\cG$ are $\beta$-smooth with gradient magnitude bounded by $B$, i.e., $\|\nabla_x \ell_{\cF}(x,y)\| \le B$ and $\|\nabla_x \ell_{\cG}(x,y)\| \le B$ for any $x \in \cX, y \in \cY$.
            Let $\Att_T$ be an $(\alpha,\cF)$-effective targeted attack with perturbation ball $\|\delta\|_2 \le \epsilon$ and target label $y_t\in \cY$.
            When the attack radius $\epsilon$ is small such that 
            $\ell_{\min} - \epsilon B \left(1 + \sqrt{\frac{1+\overline{\cS}(\el,\lt)}{2}}\right) - \beta\epsilon^2 > 0$,
            the transferability can be upper bounded by
            % 
            % Consider an instance $x \in \bbR^n$ with true label $y$ and adversarial target $y_t$.
            % % Assume models $\cF$ and $\cG$ are $\beta$-smooth with gradients bounded by $B$.
            % An $(\alpha,\cF)$-effective (targeted) attack $\xA = \Att_T(x)$ with perturbation ball $\|\delta\|_2 \le \epsilon$ is transferable to $\cG$ with bounded probability 
            % % Consider low risk models $\model$ and $\modelt$, which are $\beta$-smooth with gradient bounded by $B$.
            % % % $$\eta_\model := \bbE\left[ \ell(x,y) \right], \eta_\modelt := \bbE\left[ \ell(x,y) \right]$$ 
            % % % Suppose loss function $\ell_\model$ and $\ell_\modelt$ of the models is $\beta$-smooth with gradient bounded by $B$.
            % % The transferability of a targeted adversarial attack $x^* = \cA_T(x)$ with perturbation ball $||\delta||_2 \le \epsilon$ can be upper bounded as:
            % % \vspace{-1mm}
            $$
                % \resizebox{0.485\textwidth}{!}{
                % $
                   \Pr{T_r(\cF, \cG, x, y_t) = 1} \le 
                    \dfrac{\xi_\model + \xi_\modelt}{\ell_{\min} - \epsilon B \left(1 + \sqrt{\frac{1+\overline{\cS}(\el,\lt)}{2}}\right) - \beta\epsilon^2},
                    \label{eq:upper-bound}
                % $
                % }
            $$
            where 
            $ \displaystyle
                \ell_{\min} = \min_{x \in \cX} \, (\ell_\model(x, y_t), \ell_\modelt(x, y_t)).
            $
            Here $\xi_\cF$ and $\xi_\cG$
            are the \emph{empirical risks} of models $\cF$ and $\cG$ respectively, defined relative to a differentiable loss.
        \end{Thm}

        We defer the complete proof to \Cref{apupperbound}.
        In the proof, we first take a Taylor expansion on the loss function at $(x, y)$, then use the fact that the attack direction will be dissimilar with at least one of the model gradients to upper bound the transferability probability.

        \para{Implications.}
        In \Cref{thm:target-upper-bound}, we  observe that along with the increase of $\overline{\cS}(\el,\lt)$, the denominator decreases and henceforth the upper bound increases.
        Therefore, $\overline{S}(\el,\lt)$---upper loss gradient similarity and the upper bound of transferability probability is positively correlated.
        This tendency is the same as that in the lower bound.
        % This establishes that the gradient similarity between models is highly correlated with transferability. 
        Note that $\alpha$ does not appear in upper bounds since only completely successful attacks ($\alpha = 0\%$) needs to be considered here to upper bound the transferability.
        
        Meanwhile, when the model becomes smoother~(i.e., $\beta$ decreases), the transferability upper bound decreases and becomes tighter.
        This implication again motivates us to constrain the model smoothness.
        % to increase the effect of constraining gradient similarity.
        We further observe that smaller magnitude of gradient, i.e., $B$, also helps to tighten the upper bound.
        We will regularize both $B$ and $\beta$ to increase the effect of constraining loss gradients similarity.
        % A new term, the magnitude of gradient $B$, appears in the upper bounds.
        % The smaller the magnitude of gradient implies larger denominator and thus smaller transferability.
        % This aligns with previous findings~\cite{demontis2019adversarial}, and our proposed TRS will also optimize the model toward smaller $B$ for robustness purpose.

        % In the above theorems, we notice that increasing $\overline{\cS}(\ell_\model, \ell_\modelt)$ derives a smaller denominator and henceforth larger transferability upper bound.
        % In other words, the tendency is the same as the lower bound, where larger $\underline{\cS}(\ell_\model, \ell_\modelt)$ introduces a higher lower bound.
        
        % Decreasing gradient similarity could lead to smaller transferability upper bound.
        
        % \para{Impact of model smoothness $\beta$.}
        % Moreover, when the models become smoother with smaller $\beta$, the transferability upper bound decreases.
        Note that the lower bound and upper bound jointly show smaller $\beta$ leads to a reduced gap between lower and upper bounds and thus a stronger correlation between loss gradients similarity and transferabiltiy.
        Therefore, it is important to \emph{both} constrain gradient similarity and increase model smoothness~(decrease $\beta$) to reduce model transferability and improve ensemble robustness.

    \vspace{-0.5em}
\section{Improving Ensemble Robustness via Transferability Minimization}
    \vspace{-0.5em}
    \label{sec:algo}
    
    % Built upon our theoretical analysis on adversarial transferability, 
    Motivated by our theoretical analysis,
    we propose a lightweight yet effective robust ensemble training approach, {\textbf{T}ransferability \textbf{R}educed \textbf{S}mooth} (TRS), to reduce the transferability among base models by enforcing \textit{low loss gradient similarity} and \textit{model smoothness} at the same time. 
    % In particular, inspired by the factors that affect the upper bound of transferability, we develop {\textbf{T}ransferability \textbf{R}educed \textbf{S}mooth} (TRS) ensemble training approach to reduce adversarial transferability among base models, and thus enhance the ensemble robustness.
    
    \vspace{-0.5em}
    \subsection{TRS Regularizer}
    \vspace{-0.5em}
        
        In practice, it is challenging to directly regularize the model smoothness.
        % to ensure the trained model smoothness, regularizing the model directly %ing the smoothness to be small 
        %  is challenging. %to be achieved. %, since it is indeed a bound on the hard-to-compute second-order derivative.
        Luckily, inspired from deep learning theory and optimization~\cite{drucker1992improving,oberman2018lipschitz,sinha2017certifying},
        succinct $\ell_2$ regularization on the gradient terms $\|\nabla_x \ell_{\cF}\|_2$ and $\|\nabla_x \ell_{\cG}\|_2$ can reduce the magnitude of gradients and thus improve \textbf{model smoothness}.
        For example, for common neural networks, the smoothness can be upper bounded via bounding the $\ell_2$ magnitude of gradients~\cite[Corollary 4]{sinha2017certifying}.
        An intuitive explanation is that, the $\ell_2$ regularization on the gradient terms reduces the magnitude of model's weights, thus limits its changing rate when non-linear activation functions are applied to the neural network model.
        However,  we find that  directly regularizing the loss gradient magnitude with $\ell_2$ norm is not enough, since a vanilla $\ell_2$ regularizer such as $\|\nabla_x \ell_{\cF}\|_2$ will only focus on the local region at data point $x$, while it is required to ensure the model smoothness over a large decision region to control the adversarial transferability based on our theoretical analysis. % of its upper bound.
        
       To address this challenge, we propose a min-max framework 
       to regularize the ``support'' instance $\hat x$ with ``worst'' smoothness in the neighborhood region of data point $x$, which results in the following model smoothness loss:
    %   to search for the ``support" instances which maximizes the \textit{model smoothness loss} first, and then focus on minimizing the loss on the neighborhood regions of the decision boundary. 
    %   In particular, we compute the model smoothness loss as: 
        \begin{equation}
        %\begin{align*}
            % \small
            \mathcal{L}_{\text{smooth}}(\cF, \cG, x, \delta) =
           \max_{\|\hat{x}-x\|_\infty \leq \delta}  \|\nabla_{\hat{x}} \ell_{\cF}\|_2 + \|\nabla_{\hat{x}} \ell_{\cG}\|_2
        %\end{align*}
        \end{equation}
        where $\delta$ refers to the radius of the $\ell_\infty$ ball around instance $x$ within which we aim to ensure the model to be smooth. 
        In practice, we leverage projection gradient descent optimization to search for support instances $\hat{x}$ for optimization. 
        This model smoothness loss can be viewed as promoting margin-wise smoothness, i.e., improving the margin between nonsmooth decision boundaries and data point $x$.
        Another option is to promote point-wise smoothness that only requires the loss landscape at data point $x$ itself to be smooth.
        We compare the ensemble robustness of the proposed min-max framework which promotes the margin-wise smoothness with the na\"ive baseline which directly applies $\ell_2$ regularization on each model loss gradient terms to promote the point-wise smoothness
        % for model smoothing 
        (i.e. \cosll) in \Cref{sec:exp}.

        Given trained ``smoothed" base models, we also %control and reduce 
        decrease the model \textbf{loss gradient similarity} to reduce the overall adversarial transferability between base models.
       %We will first measure the similarity between model $\cF$ and $\cG$ by various metrics. 
       Among various metrics which measure the similarity between the loss gradients of base model $\cF$ and $\cG$, we find that the vanilla cosine similarity metric, which is also used in \cite{kariyappa2019improving}, may lead to certain concerns. By minimizing the cosine similarity between $\nabla_x \ell_{\cF}$ and $\nabla_x \ell_{\cG}$, the optimal case implies $\nabla_x \ell_{\cF} = -\nabla_x \ell_{\cG}$, which means two models have contradictory (rather than diverse) performance on instance $x$ and thus results in turbulent model functionality. Considering this challenge, we leverage the absolute value of cosine similarity 
       %$\mathcal{L}_{sim}$ in our proposed TRS loss 
       between $\nabla_x \ell_{\cF}$ and $\nabla_x \ell_{\cG}$ as \emph{similarity loss} $\mathcal{L}_{\text{sim}}$
       and its optimal case implies orthogonal loss gradient vectors. For simplification, we will always use the absolute value of the gradient cosine similarity as the indicator of \textit{gradient similarity}  in our later description and evaluation.
        
        %To decrease model loss gradient similarity, the cosine similarity may not be counted into one of the brilliant choice. similarity measurements due to complicated scenarios emerged in high-dimensional space. That is, two high dimensional vectors could be orthogonal (i.e. cosine similarity equals to zero) even though they could be identical under the linear transformation. \cite{kornblith2019similarity} proposed a more general but reasonable metric to measure two vector's similarly by considering the \emph{Central Kernel Alignment} (CKA). 
        
    %   Inspired by the principles we mentioned above, we define the TRS regularizer term $\mathcal{L}_{\text{TRS}}$ between model pair $(\cF, \cG)$ on input $x$ as the following:
       Based on our theoretical analysis and particularly the model \textit{loss gradient similarity} and \textit{model smoothness} optimization above, we propose TRS regularizer for model pair $(\cF, \cG)$ on input $x$ as:
        % %\begin{align*}
        % %    \mathcal{L}_{\text{TRS}}(\cF, \cG, x) &= 
        % $$    \lambda_a \cdot \bigg|\frac{(\nabla_x \ell_{\cF})^\top (\nabla_x \ell_{\cG})}{\|\nabla_x \ell_{\cF}\|_2 \cdot \|\nabla_x \ell_{\cG}\|_2}\bigg|\\
        %     %\frac{\nabla_x \ell_{\cF} \otimes \nabla_x \ell_{\cG}}{\sqrt{\nabla_x \ell_{\cF} \otimes \nabla_x \ell_{\cF} \cdot \nabla_x \ell_{\cG} \otimes \nabla_x \ell_{\cG}}} \\
        %      + \lambda_b \cdot (\|\nabla_x \ell_{\cF}\|_2 + \|\nabla_x \ell_{\cG}\|_2)
        % $$
        % %\end{align*}
          \begin{align*}
            \mathcal{L}_{\text{TRS}}(\cF, \cG, x, \delta)
            & = \lambda_a \cdot \mathcal{L}_{\text{sim}} + \lambda_b \cdot \mathcal{L}_{\text{smooth}} \\
            &= 
            \lambda_a \cdot \bigg|\frac{(\nabla_x \ell_{\cF})^\top (\nabla_x \ell_{\cG})}{\|\nabla_x \ell_{\cF}\|_2 \cdot \|\nabla_x \ell_{\cG}\|_2}\bigg|
            %\frac{\nabla_x \ell_{\cF} \otimes \nabla_x \ell_{\cG}}{\sqrt{\nabla_x \ell_{\cF} \otimes \nabla_x \ell_{\cF} \cdot \nabla_x \ell_{\cG} \otimes \nabla_x \ell_{\cG}}} \\
             + \lambda_b \cdot \bigg[\max_{\|\hat{x}-x\|_\infty \leq \delta}  \|\nabla_{\hat{x}} \ell_{\cF}\|_2 + \|\nabla_{\hat{x}} \ell_{\cG}\|_2 \bigg].
        \end{align*}
        Here $\nabla_x \ell_{\cF}$ and $\nabla_x \ell_{\cG}$ refer to the loss gradient vectors of base models $\cF$ and $\cG$ on input $x$, and $\lambda_a, \lambda_b$ the weight balancing parameters. %The first term refer to the absolute value of cosine similarity between model loss gradient vectors and the second term is designed to restrain models' gradient magnitude thus enforce TRS model to be smooth.

        In Section~\ref{sec:exp}, backed up by extensive empirical evaluation,
        we will systematically show that the local min-max training and the absolute value of the cosine similarity between the model loss gradients significantly improve the ensemble model robustness with negligible performance drop on  benign accuracy, as well as reduce the adversarial transferability among base models. 
        % We will provide the detailed empirical evaluation on the ensemble robustness and adversarial transferability in Section~\ref{sec:exp}.

    \vspace{-0.5em}
   \subsection{TRS Training}
    \vspace{-0.5em}
  %  \zhuolin{One issue: should we name all the TRS as GST-TRS (or something) now or later? since starting from this section we are all talking about advanced TRS}
    %With the TRS loss defined above, 
    We integrate the proposed TRS regularizer with the standard ensemble training loss, such as Ensemble Cross-Entropy (ECE) loss, to maintain both ensemble model's classification utility and robustness by varying the balancing parameter $\lambda_a$ and $\lambda_b$. Specifically, for an ensemble model consisting of $N$ base models $\{\mathcal{F}_i\}_{i=1}^N$, given an input $(x, y)$, our final training loss $\mathcal{\text{train}}$ is defined as:
    $$\mathcal{L}_{\text{train}} = \frac{1}{N}\sum_{i=1}^{N}\mathcal{L}_{\text{CE}}(\mathcal{F}_i(x), y) + \frac{2}{N(N-1)}\sum_{i=1}^{N}\sum_{j=i+1}^{N}\mathcal{L}_{\text{TRS}}(\cF_i, \cF_j, x, \delta)$$
    where $\mathcal{L}_{\text{CE}}(\mathcal{F}_i(x), y)$ refers to the cross-entropy loss between $\mathcal{F}_i(x)$, the output vector of model $\cF_i$ given $x$, and the ground-truth label $y$. The weight of $\mathcal{L}_{\text{TRS}}$ regularizer could be adjusted by the tuning $\lambda_a$ and $\lambda_b$ internally. We present one-epoch training pseudo code in \Cref{algo:trs-training} of \Cref{adx:sec-experiment-details}. 
    % In our experiments, we follow the designed training algorithm mentioned here to train our ensemble. 
    The detailed hyper-parameter setting and training criterion are  discussed in \Cref{adx:sec-experiment-details}.%Due to the space limitation, we leave more training details into the \Cref{adx:sec-experiment-details}.
    
    %The detailed TRS training algorithm for one epoch is illustrated in \Cref{algo:trs-training} in \Cref{adx-sec:trs-alg}. We apply the mini-batch training strategy and train the TRS ensemble for $M$ epochs ($M=120$ for MNIST and $M=200$ for CIFAR-10) in our experiments. To decide the $\delta$ within the local min-max procedure, we use the \textbf{Warm-up} strategy by linearly increasing the local $\ell_\infty$ ball's radius $\delta$ from small initial $\delta_0$ to the final $\delta_M$ along with the increasing of training epochs. Due to the space limitation, We put more training details into the \Cref{adx:sec-experiment-details}. % and discuss training algorithm's convergence later.
    %\subsection{Discussion on Other Design Choices}

    %We leverage the regularizer of the loss gradient to constrain model smoothness. 
    %From Theorems~\ref{thm:target}--\ref{thm:untarget-upper-bound} we can see that when the models are forced to be smooth, the lower and upper bound of transferability would be tighter and therefore as long as the model loss gradient similarity is minimized, the transferability will be largely constrained. We validate this observation empirically in Section~\ref{sec:exp}.

%\zhuolin{should we clarify the PGD function's details inside the pseudo code?}

%\bo{pseudocode here}

%\zhuolin{done but not too adorable. Going to polish this pseudo code}

%\bo{motivate the inner products, and its rational}

% algorithm box

% ===================== Experiments =====================
\vspace{-0.7em}
\section{Experimental Evaluation}
\vspace{-0.7em}
    \label{sec:exp}
    %In this section we provide experimental evaluation of the proposed TRS-ensemble comparing with other baseline works to assess its robustness against different attacks and attack transferability among base models. To demonstrate our points, we conduct experiments on MNIST, CIFAR-10 and CIFAR-100~\cite{lecun1998mnist, krizhevsky2009learning}.
        %\linyi{Need to describe MNIST and CIFAR10 in detail.}
        % MNIST is a hand-written digits dataset containing gray-scale images of digits from 0 to 9 and CIFAR-10 contains colorful images divided into 10 classes. There are 50000 images for training and 10000 images for testing in both datasets. 
    In this section, we evaluate the robustness of the proposed TRS-ensemble model under both strong \underline{whitebox} attacks, as well as \underline{blackbox} attacks considering the gradient obfuscation concern~\cite{athalye2018obfuscated}. We compare TRS with six state-of-the-art ensemble approaches. In addition, we evaluate the adversarial transferability among base models within an ensemble and empirically show that the TRS regularizer can indeed reduce  transferability effectively.
    We also conduct extensive ablation studies to explore the effectiveness of different loss terms in TRS, as well as
    visualize the trained decision boundaries of different ensemble models to provide intuition on the model properties.
    % We will open source our implementation for reproducibility purpose and provide a large-scale benchmark for the community of robustness and ensemble learning. 
    We open source the code\footnote{\href{https://github.com/AI-secure/Transferability-Reduced-Smooth-Ensemble}{\texttt{https://github.com/AI-secure/Transferability-Reduced-Smooth-Ensemble}}} and provide a large-scale benchmark.
    % between models. 
    % show the low attack transferability among base models for TRS-ensemble. 

    %We show that: (1) on benign instances, the trained TRS-ensemble achieves similar performance compared with vanilla models; (2) against different adversarial \underline{whitebox} attacks, TRS-ensemble outperforms other baselines in terms of robustness; (3) we specifically analyze the transferability among base models within the TRS-ensemble and show that intra-ensemble transferability is indeed reduced significantly, across a range of attacks.

    % \vspace{-1em}

    %We show that: (1) on benign instances, the trained TRS-ensemble achieves similar performance compared with vanilla models; (2) against different adversarial \underline{whitebox} attacks, TRS-ensemble outperforms other baselines in terms of robustness; (3) we specifically analyze the transferability among base models within the TRS-ensemble and show that intra-ensemble transferability is indeed reduced significantly, across a range of attacks.

    % \vspace{-1em}
    \vspace{-0.6em}
    \subsection{Experimental Setup}
    \vspace{-0.6em}
        % dataset
        % model

        \noindent\textbf{Datasets.} We conduct our experiments on widely-used image datasets including hand-written dataset MNIST~\cite{lecun1998mnist}; and colourful image datasets CIFAR-10 and CIFAR-100~\cite{krizhevsky2009learning}. %MNIST~\cite{lecun1998mnist} is a gray-scale hand-written digits dataset with 60,000 training images and 10,000 testing images with label from 0 to 9. CIFAR-10~\cite{krizhevsky2009learning} consists of 10 categorical colour images with 5,000 training images and 1,000 testing images on each class. CIFAR-100~\cite{krizhevsky2009learning}, which is similar to CIFAR-10, consists of 100 classes colour images with 500 training and 100 testing images for each class. 
        
        \noindent\textbf{Baseline ensemble approaches.} We mainly consider the standard  ensemble, as well as the state-of-the-art robust ensemble methods that claim to be resilient against adversarial attacks. 
        Specifically, we consider the following baseline ensemble methods which aim to promote the diversity between base models: \textbf{AdaBoost}~\cite{hastie2009multi}; \textbf{GradientBoost}~\cite{friedman2001greedy}; \textbf{CKAE}~\cite{kornblith2019similarity}; \textbf{ADP}~\cite{pang2019improving}; \textbf{GAL}~\cite{kariyappa2019improving}; \textbf{DVERGE}~\cite{yang2020dverge}. The detailed description about these approaches are in \Cref{adx:sec-omitted-introduction}. DVERGE, which has achieved the state-of-the-art ensemble robustness to our best knowledge, serves as the strongest baseline.

        %\noindent\textbf{Training details.} 

        \begin{table*}[!t]
    % \vspace{-1.5em}
            \centering
            \caption{\small Robust accuracy$(\%)$ of different ensembles against whitebox attacks on MNIST/CIFAR-10. ``para.''  refers to the attack parameter ($\epsilon$ is the $\ell_\infty$ perturbation budget for the attack and $c$ the constant to balance the attack stealthiness and effectiveness). The first 6 methods 
            % (Adaboost, $\cdots$, DVERGE) 
            are baseline ensembles, and  the last 3 columns  (Cos-only, Cos-$\ell_2$, TRS) the variants of TRS-ensemble.}
            
            \noindent
            \scalebox{0.62}{
            \begin{tabular}{c||c||c|c|c|c|c|c||c|c|c}
            \toprule
        \bf MNIST                     & para.        & AdaBoost & GradientBoost     & CKAE & ADP  & GAL  & DVERGE & Cos-only & Cos-$\ell_2$ & \textbf{TRS}  \\ \hline
            \rowcolor{tabgray}      & $\epsilon=0.1$  & 70.2 & 73.2     & 72.6          & 71.7 & 35.7 & \textbf{95.8}   & 66.2 & 91.2 & 95.6 \\
                                     \rowcolor{tabgray} \multirow{-2}{*}{FGSM} & $\epsilon=0.2$  & 39.4 & 34.2    & 42.5        & 20.0 & 7.8 & 91.6  & 30.7 & 72.5 & \textbf{91.7}  \\ \hline
            \multirow{2}{*}{BIM (50)} & $\epsilon=0.1$  & 2.6 & 2.4     & 4.2          & 7.7 & 4.6 & 74.9 & 0.4 & 76.2 & \textbf{93.3}  \\
                                      & $\epsilon=0.15$ & 0.0 & 0.2  & 0.4       & 0.1 & 2.5 & 47.7   & 0.0 & 47.9 & \textbf{85.7}\\ \hline
            \rowcolor{tabgray} & $\epsilon=0.1$  & 1.9 & 1.5    & 1.4        & 4.5 & 4.1 & 69.2   & 0.0 & 73.4 & \textbf{93.0}  \\
                                      \rowcolor{tabgray} \multirow{-2}{*}{PGD (50)}& $\epsilon=0.15$ & 0.0  & 0.0  & 0.5       & 1.0 & 0.6 & 28.8   & 0.0 & 30.2 & \textbf{85.1}  \\ \hline
            \multirow{2}{*}{MIM (50)} & $\epsilon=0.1$  & 1.9 & 1.6  & 1.2      & 13.8 &0.8 & 75.3   & 0.4 & 74.1 & \textbf{92.9}  \\
                                      & $\epsilon=0.15$ & 0.0 & 0.1   & 0.3        & 1.0 & 0.2 & 44.6  & 0.0 & 35.5 & \textbf{85.1}   \\ \hline
            \rowcolor{tabgray}       & $c=0.1$         & 81.2 & 80.5    & 83.4        & 97.9 & 97.4 & 97.3   & 85.6 & 89.2 & \textbf{98.1} \\
                                      \rowcolor{tabgray}\multirow{-2}{*}{CW}& $c=1.0$        & 66.3 & 65.8     & 69.5         & 90.1 & 68.3 & 79.2   & 58.6 & 54.4 & \textbf{92.6} \\ \hline
            % \multirow{2}{*}{JSMA}     & $\gamma=0.3$    & -    & -        & -             & - & - & -   & - \\
            %                           & $\gamma=0.6$    & -    & -        & -             & - & - & -   & - \\ \hline
            \multirow{2}{*}{EAD}      & $c=5.0$  & 0.2    & 0.1        & 0.1             & 2.2 & 0.2 & 0.0   & 4.1 & 6.9 & \textbf{23.3} \\
                                      & $c=10.0$ & 0.0    & 0.0        & 0.0             & 0.0 & 0.2 & 0.0   & 0.5 & 0.8 & \textbf{1.4} \\ \hline
            \rowcolor{tabgray}  & $\epsilon=0.1$  & 0.5    & 0.2        & 0.5             & 2.1  & 1.9  & 65.4   & 0.0 & 70.6 & \textbf{92.1}   \\
                                      \rowcolor{tabgray} \multirow{-2}{*}{APGD-DLR} & $\epsilon=0.15$ & 0.0    & 0.0        & 0.1             & 0.5  & 0.2  & 27.4   & 0.0 & 26.3 & \textbf{83.4} \\ \hline
            \multirow{2}{*}{APGD-CE}  & $\epsilon=0.1$  & 0.2    & 0.2        & 0.1             & 1.4  & 1.2 & 63.2   & 0.0 & 69.8 & \textbf{91.7}\\
                                      & $\epsilon=0.15$ & 0.0    & 0.0        & 0.1             & 0.4  & 0.2  & 26.1   & 0.0 & 25.4 & \textbf{82.8}\\ \bottomrule
            \end{tabular}}
            %\label{tab:result1-mnist}
            %\vspace{-1.0em}
        %\end{table*}
        \\

        %\begin{table*}[!t]
            \noindent
            %\caption{\small Robust accuracy $(\%)$ of different approaches against various whitebox attacks on CIFAR-10 dataset. \xiaojun{updated TRS result.}
            %}
            \scalebox{0.625}{
            \begin{tabular}{c||c||c|c|c|c|c|c||c|c|c}
            \toprule
            \bf CIFAR-10 & para.& AdaBoost & GradientBoost  & CKAE  & ADP  & GAL  & DVERGE & Cos-only & Cos-$\ell_2$ & \textbf{TRS} \\ \hline
            \rowcolor{tabgray}     & $\epsilon=0.02$ & 28.2 & 30.4 & 34.1 & 58.8 & 19.2 & \bf 63.8 & 56.1 & 35.8 & 44.2 \\
                                      \rowcolor{tabgray} \multirow{-2}{*}{FGSM} & $\epsilon=0.04$ & 15.4 & 15.2 & 18.5 & 39.4 & 12.6 & \bf 53.4 & 35.0 & 25.9 & 24.9  \\ \hline
            \multirow{2}{*}{BIM (50)} & $\epsilon=0.01$ & 4.2 & 4.4 & 5.1 & 13.8 & 13.0 & 39.1 & 0.0 & 17.1 & \textbf{50.6}  \\
                                      & $\epsilon=0.02$ & 0.2 & 0.1 & 0.2 & 0.9 & 2.5 & 13.0 & 0.0 & 1.2 & \textbf{15.8}  \\ \hline
            \rowcolor{tabgray} & $\epsilon=0.01$ & 2.1 & 1.9 & 1.9 & 9.0 & 8.3 & 37.1 & 0.0 & 15.7 & \textbf{50.5}  \\
                                      \rowcolor{tabgray} \multirow{-2}{*}{PGD (50)} & $\epsilon=0.02$ & 0.0 & 0.0 & 0.2 & 0.1 & 0.6 & 10.5 & 0.0 & 0.5 & \textbf{15.1}  \\ \hline
            \multirow{2}{*}{MIM (50)} & $\epsilon=0.01$ & 2.3 & 1.9 & 2.0 & 18.7 & 10.3 & 40.7 & 0.0 & 18.1 & \textbf{51.5} \\
                                      & $\epsilon=0.02$ & 0.1 & 0.0 & 0.1 & 1.7 & 0.8 & 14.4 & 0.0 & 0.5 & \textbf{17.2}  \\ \hline
            \rowcolor{tabgray}       & $c=0.01$ & 36.2 & 35.2 & 35.4 & 55.8 & 66.3 & 75.1 & 36.6 & 67.3 & \textbf{77.2} \\
                                      \rowcolor{tabgray} \multirow{-2}{*}{CW} & $c=0.1$  & 18.4 & 26.2 & 23.0 & 25.9 & 28.3 & 57.4 & 17.6 & 30.7 & \textbf{58.1}  \\ \hline
            % \multirow{2}{*}{JSMA}     & $\gamma=0.05$ & - & - & - &  &  &  & \textbf{} \\
            %                           & $\gamma=0.1$  & - & - & - &  &  &  & \textbf{} \\ \hline
            \multirow{2}{*}{EAD}      & $c=1.0$ & 0.2 & 0.0 & 0.0 & 9.0 & 0.0 & 0.2 & 0.0 & 0.0 & \bf 11.7 \\
                                      & $c=5.0$ & 0.0 & 0.0 & 0.0 & 0.0 & 0.0 & 0.0 & 0.0 & 0.0 & \bf 0.1 \\ \hline
            \rowcolor{tabgray}  & $\epsilon=0.01$ & 1.2 & 0.9 & 1.1 & 5.5 & 2.2 & 37.6 & 0.0 & 16.1 & \textbf{50.2}  \\
                                      \rowcolor{tabgray} \multirow{-2}{*}{APGD-DLR} & $\epsilon=0.02$ & 0.0 & 0.0 & 0.0 & 0.2 & 0.0 & 10.2 & 0.0 & 0.5 & \textbf{15.1}  \\ \hline
            \multirow{2}{*}{APGD-CE}  & $\epsilon=0.01$ & 0.9 & 0.2 & 0.4 & 3.9 & 1.6 & 37.5 & 0.0 & 15.9 & \textbf{48.6}  \\
                                      & $\epsilon=0.02$ & 0.0 & 0.0 & 0.0 & 0.1 & 0.0 & 10.2 & 0.0 & 0.5 & \textbf{15.0}\\ \bottomrule
            \end{tabular}}
            \label{tab:result1}
            \vspace{-1.5em}
        \end{table*}

        % During training we use the Adam optimizer~\cite{kingma2014adam} with initial learning rate $\alpha = 0.001$. 
        % We run 40 epochs on MNIST and 10 epochs on CIFAR making sure the loss has converged 
        % \vspace{1mm}
        \noindent\textbf{Whitebox robustness evaluation.} We consider the following adversarial attacks to measure ensembles' \underline{whitebox} robustness: \emph{Fast Gradient Sign Method} (\textbf{FGSM})~\cite{goodfellow2014explaining}; \emph{Basic Iterative Method} (\textbf{BIM})~\cite{madry2017towards}; \emph{Momentum Iterative Method} (\textbf{MIM}); \emph{Projected Gradient Descent} (\textbf{PGD}); \emph{Auto-PGD} (\textbf{APGD}); \emph{Carlini \& Wanger Attack} (\textbf{CW}); \emph{Elastic-net Attack} (\textbf{EAD})~\cite{chen2018ead}, and we leave the detailed description and parameter configuration of these attacks in \Cref{adx:sec-omitted-introduction}. 
        We use \emph{Robust Accuracy} as our \textbf{evaluation metric} for the whitebox setting, defined as the ratio of correctly predicted \textit{adversarial examples} generated by different attacks among the whole test dataset.

        \noindent\textbf{Blackbox robustness evaluation.} We also conduct blackbox robustness analysis in our evaluation since recent studies have shown that robust models which obfuscate gradients could still be fragile under blackbox attacks~\cite{athalye2018obfuscated}. In the \underline{blackbox} attack setting, we assume the attacker has no knowledge about the target ensemble, including the model architecture and parameters.
        In this case, the attacker is only able to craft adversarial examples based on several surrogate models and transfer them to the target victim ensemble. We follow the same blackbox  attack evaluation setting in \cite{yang2020dverge}: We choose three {ensembles} consisting of $3, 5, 8$ base models which are trained with standard Ensemble Cross-Entropy (ECE) loss as our surrogate models. We apply $50$-steps PGD attack with three random starts and two different loss functions (CrossEntropy and CW loss) on each surrogate model to generate adversarial instances (i.e. for each instance we will have 18 attack attempts). For each instance, among these attack attempts, as long as there is one that can successfully attack the victim model, we will count it as a successful attack. In this case, we use \emph{Robust Accuracy} as our \textbf{evaluation metric}, defined as the number of unsuccessful attack attempts divided by the number of all attacks. We also consider additional three strong blackbox attacks targeting on reducing transferability ($i.e.$, ILA~\cite{huang2019enhancing}, DI2-SGSM~\cite{xie2019improving}, IRA~\cite{wang2020unified}) in \Cref{adx:other-strong-attacks}, which leads to  similar observations. %of the argmax labels of the \textit{adversarial examples} generated by these attacks.

        \vspace{-0.5em}
    \subsection{Experimental Results}
        \vspace{-0.5em}

     In this section, we present both \underline{whitebox} and \underline{blackbox} robustness evaluation results, examine the adversarial transferability, and explore the impacts of different loss terms in TRS.
     Furthermore, in \Cref{adx:sec-decision-boundary}, we  visualize the decision boundary;
     in \Cref{adx:sec-adv-training}, we show results of further improving the robustness of the TRS ensemble by integrating adversarial training; in \Cref{adx:separate-effects-section}, we study the impacts of each of the regularization term $\mathcal{L}_{\text{sim}}$ and $\mathcal{L}_{\text{smooth}}$; in \Cref{adx:sec-convergence}, we show the convergence of robust accuracy under large attack iterations to demonstrate the robustness stability of TRS ensemble; in \Cref{adx:sec-pgd-inner-convergence}, we analyze the trade-off between the training cost and  robustness of TRS by varying  PGD step size and the total number of steps within $\mathcal{L}_{\text{smooth}}$ approximation.
     
    %  \vspace{1mm}
     \noindent\textbf{Whitebox robustness.}   
        %We evaluate ensemble models' robustness against various \underline{whitebox} attacks, where attacker has full access to the target model. All the ensemble models are consisting of three base models with ResNet20 architecture. 
        Table~\ref{tab:result1} presents the \emph{Robust Accuracy} of different ensembles against a range of whitebox attacks on MNIST and CIFAR-10 dataset. We defer results on CIFAR-100 in \Cref{adx:sec-cifar100-results}, and measure the statistical stability of our reported robust accuracy in \Cref{adx:sec-stability-results}. Results shows that the proposed TRS ensemble outperforms other baselines including the state-of-the-art DVERGE  \emph{significantly}, against a range of attacks and perturbation budgets, and such performance gap could be even larger under stronger adversary attacks (e.g. PGD attack). We note that TRS ensemble is slightly less robust than DVERGE under small perturbation  with weak attack FGSM. We investigate this based on the decision boundary analysis in \Cref{adx:sec-decision-boundary}, and find that DVERGE tends to
        be more robust along the gradient direction and thus more robust against weak attacks which only focus on the gradient direction (e.g., FGSM); while 
         TRS  yields a smoother model along different directions
       leading to more consistent predictions within a larger neighborhood of an input, and thus more robust against strong iterative attacks (e.g., PGD).
      This may be due to that DVERGE is essentially performing adversarial training for different base models and therefore it protects the adversarial (gradient) direction, while TRS optimizes to train a smooth ensemble with diverse base models. 
        We also analyze the convergence of attack algorithms in \Cref{adx:sec-convergence}, showing that when the number of attack iterations is large, both ADP and GAL ensemble achieve much lower robust accuracy against such iterative attacks; while both DVERGE and TRS remain robust.
        
        %We observe that when the number of attack iterations is large both ADP and GAL regularizer trained ensemble achieve much lower robust accuracy against iterative attacks (BIM, PGD, MIM) than the reported robustness in the original paper, which is estimated under a small number of attack iterations. This case implies the non-convergence of iterative attack evaluation mentioned in their paper, which is also confirmed by~\cite{tramer2020adaptive}. In the contrast, both DVERGE and TRS still remain highly robust against iterative attacks with large iterations. To show the stability of our model's robust accuracy, we evaluate it against PGD attack with $500$ and $1000$ attack iterations. Results are shown in Table~\ref{tab:convergence} where TRS ensemble's robust accuracy only dropped a little bit after increasing the attack iterations, and outperforms DVERGE by a large margin.  
        %We found that 1) slightly lower benign accuracy 2) Decision boundary shape of TRS model could be the major reason, and we'll discuss about this in the later chapters. For the detailed configuration of each evaluated attacks, we put them into Appendix. \zhuolin{need a ref}
        
        % \xiaojun{We also show the result with adversarial retraining and with different number of base models in the appendix.}

    % \vspace{1mm}
     \noindent\textbf{Blackbox robustness.}
        Figure~\ref{fig:blackbox-curve} shows the \emph{Robust Accuracy} performance of TRS compared with different baseline ensembles 
        % including \textit{Vanilla ensemble}, ADP, GAL, and DVERGE, against blackbox transfer attacks 
        under different perturbation budget $\epsilon$. As we can see, the TRS ensemble achieves competitive robust accuracy with DVERGE when $\epsilon$ is very small, and TRS beats \emph{all} the baselines significantly when $\epsilon$ is large. Precisely speaking, TRS ensemble achieves over $85\%$ robust accuracy against transfer attack with $\epsilon=0.4$ on MNIST while the second-best ensemble (DVERGE) only achieves $20.2\%$. Also on CIFAR-10, TRS ensemble achieves over $25\%$ robust accuracy against transfer attack when $\epsilon=0.06$, while all the other baseline ensembles achieve robust accuracy  lower than $6\%$. This implies that our proposed TRS ensemble has stronger generalization ability in terms of robustness against large $\epsilon$ adversarial attacks compared with other ensembles. We also put more details of the robust accuracy under blackbox attacks in \Cref{adx:sec-blackbox-results}.%. \zhuolin{need a ref here} %These results support our claim that TRS model could regularize inter model adversarial transferability by enforcing small gradient similarity and making the model to be smoother at the same time.
        
         \begin{figure}[!t]
            %   \vspace{-1.5em}
            \centering
            
            \includegraphics[width=0.4\textwidth]{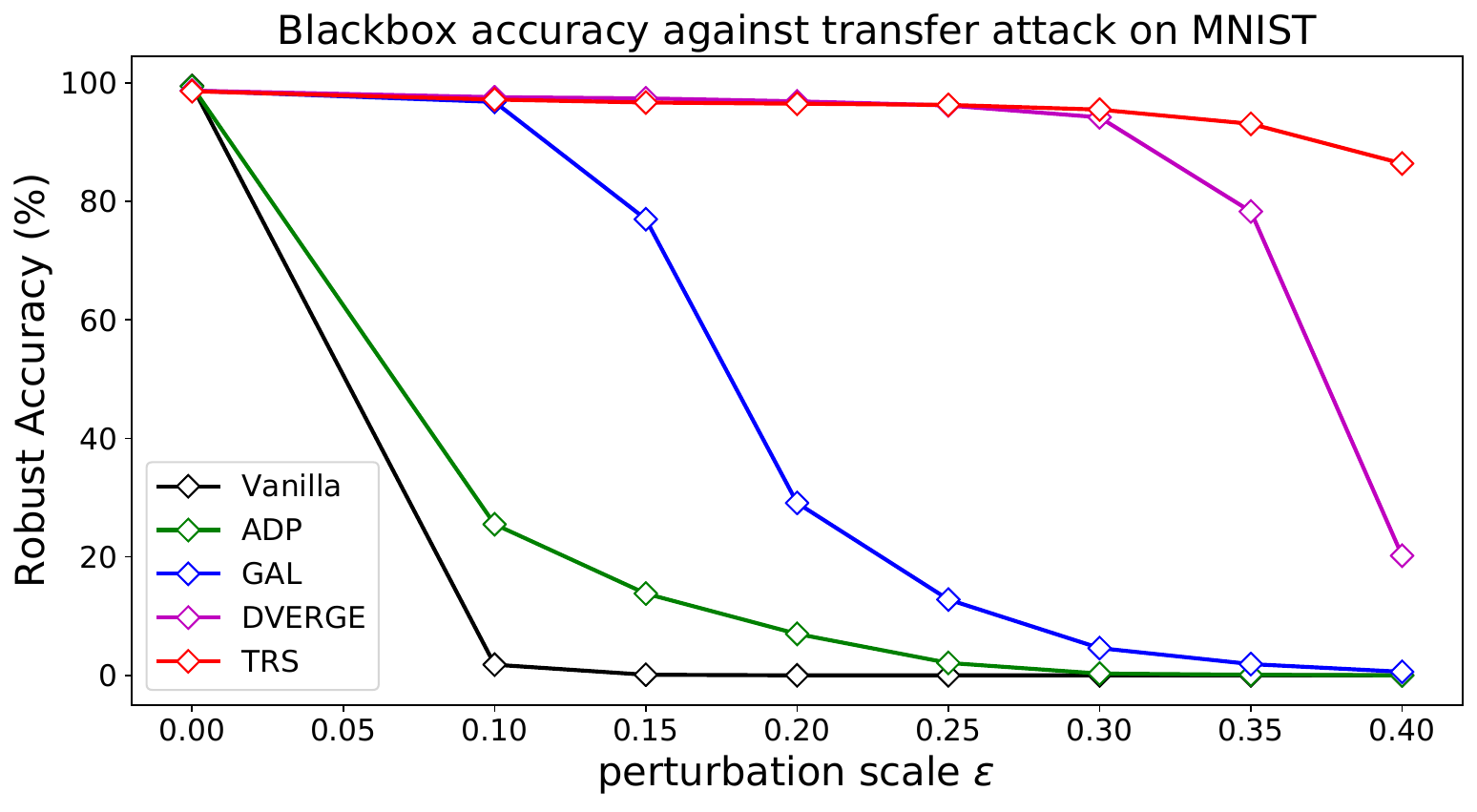}
            \includegraphics[width=0.4\textwidth]{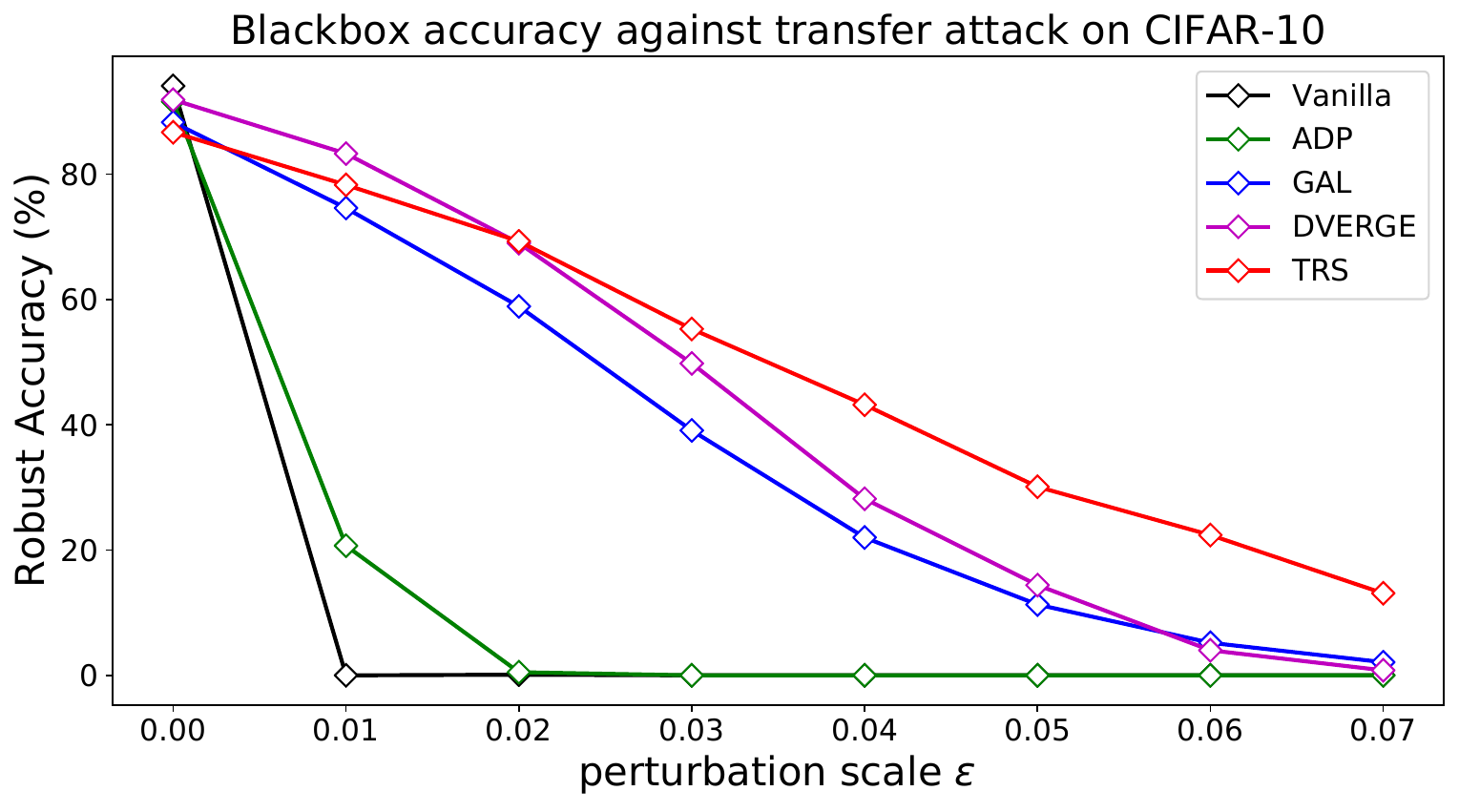}
                                                    \vspace{-1.0em}
            \caption{\small Robust accuracy under blackbox attacks with different $\ell_\infty$ perturbation budget $\epsilon$. (Left): MNIST; (Right): CIFAR-10.}
            \label{fig:blackbox-curve}
                                        % \vspace{-1em}
        \end{figure}
        
        \begin{figure*}[!t]
                      \vspace{-0.3em}
        \centering
            \includegraphics[width=0.155\textwidth]{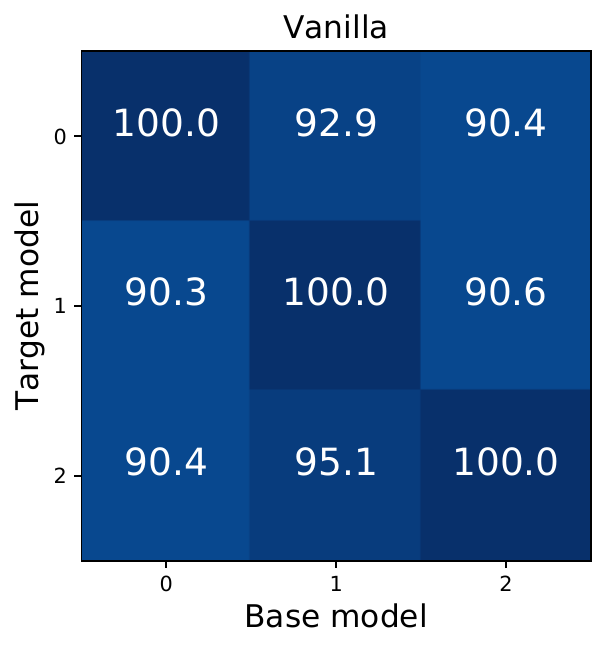}
            \includegraphics[width=0.155\textwidth]{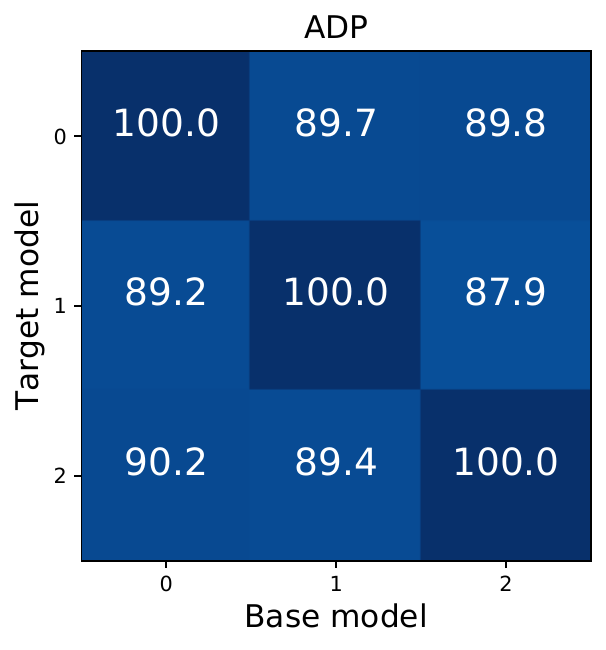}
            \includegraphics[width=0.155\textwidth]{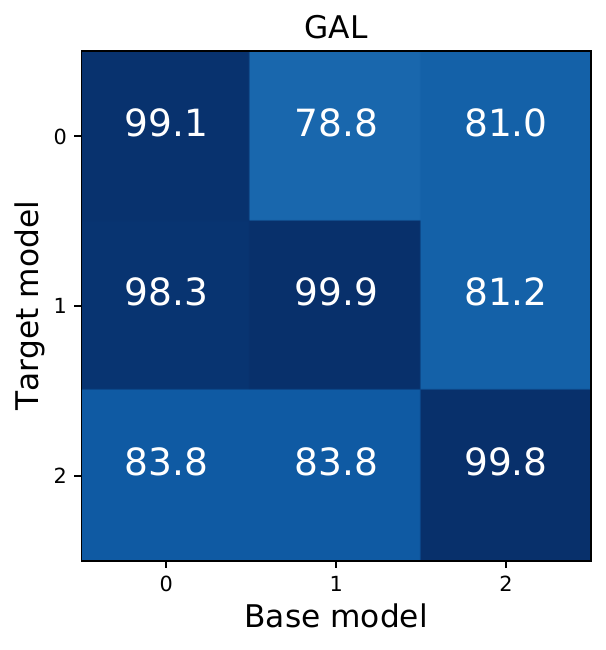}
            \includegraphics[width=0.155\textwidth]{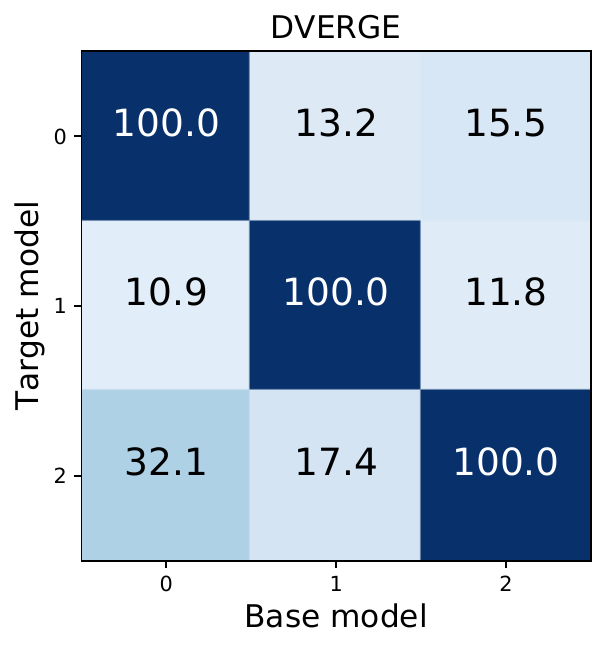}
            \includegraphics[width=0.155\textwidth]{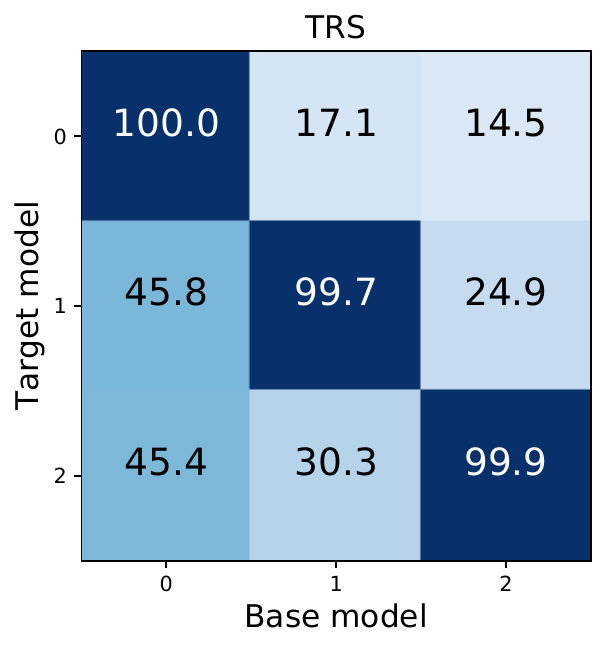}
            
            \includegraphics[width=0.155\textwidth]{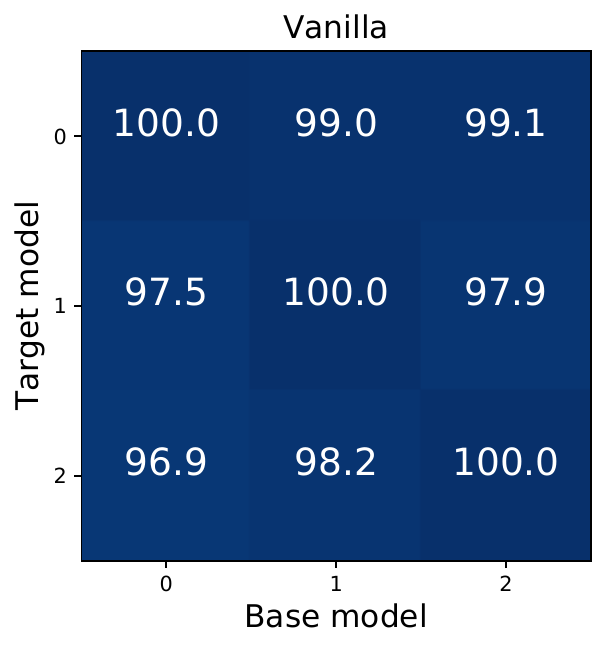}
            \includegraphics[width=0.155\textwidth]{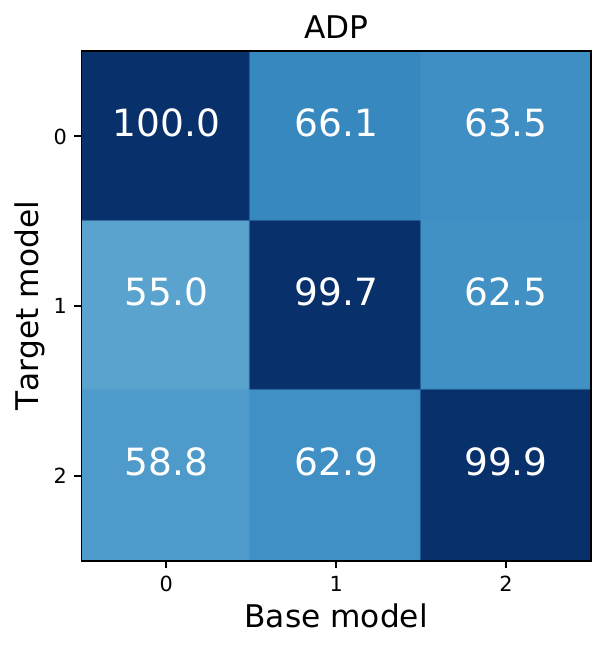}
            \includegraphics[width=0.155\textwidth]{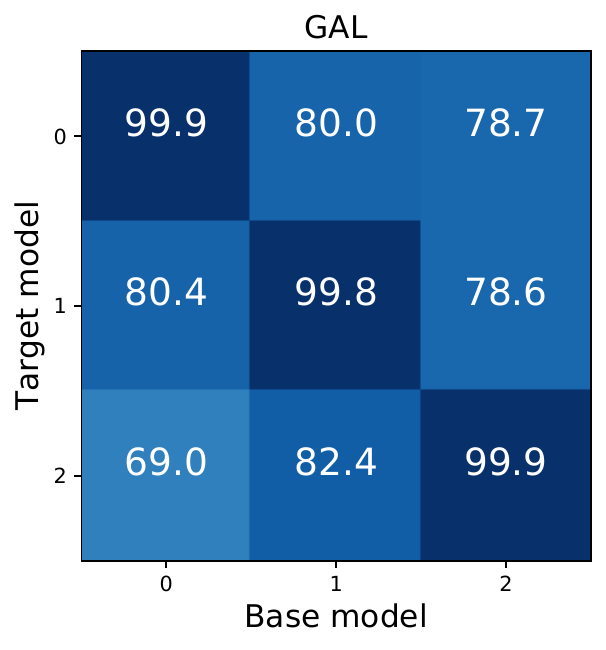}
            \includegraphics[width=0.155\textwidth]{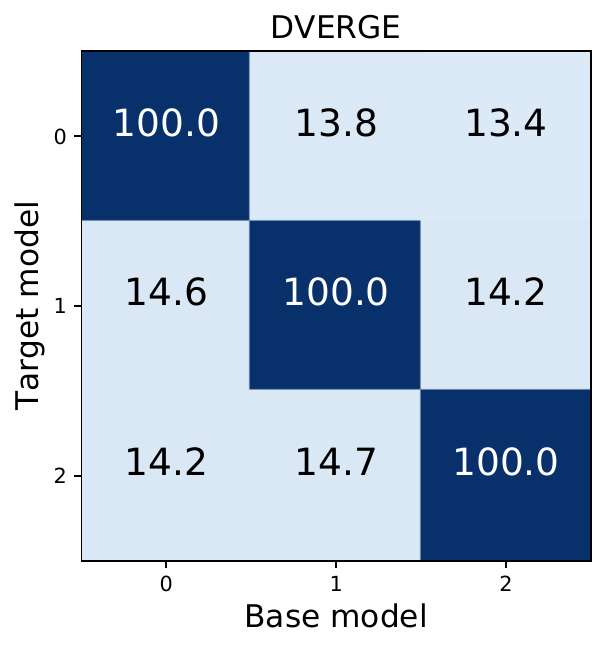}
            \includegraphics[width=0.155\textwidth]{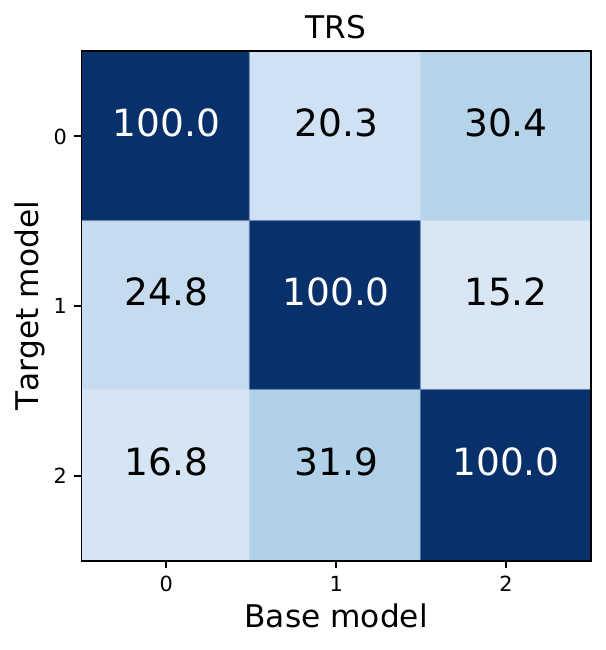}
                                        \vspace{-0.7em}
            \caption{\small Transferability analysis of  PGD attack on  MNIST (top) and CIFAR-10 (bottom). Each cell $(i,j)$ shows the \emph{attack success rate} of $i$-th model on the adversarial examples generated against the $j$-th model. We use $\epsilon=0.3$ for MNIST and $\epsilon=0.04$ for CIFAR-10. 
            %\zhuolin{baseline here should be Baseline? and we need to define baseline ensemble. How about Vanilla instead?}
            } %\bo{mnist needs to choose a reasonable eps}}
            % \includegraphics[width=0.48\textwidth]{figures/transferability2.png}
            % \caption{\small Transferability analysis for vanilla and base models from TRS-ensemble on MNIST ($\epsilon = 0.2$). Each cell (i,j) shows the \emph{classification accuracy} of $i$th model on the adversarial examples generated against the $j$th model. The first row presents vanilla models, and second row the base models from TRS-ensemble. It shows that the base models from TRS-ensemble achieves higher accuracy indicating lower transferability.}
            \label{fig:transferability}
                            % \vspace{-1.5em}
        \end{figure*}
        
        % \vspace{1mm}
        \noindent\textbf{Adversarial transferability.}   
        Figure~\ref{fig:transferability} shows the adversarial transferability matrix of different ensembles against $50$-steps PGD attack with $\epsilon=0.3$ for MNIST and $\epsilon=0.04$ for CIFAR-10. Cell $(i, j)$ where $i \neq j$ represents the transfer attack success rate evaluated on $j$-th base model by using the $i$-th base model as the surrogate model.
        Lower number in each cell indicates lower transferability and thus potentially higher ensemble robustness.
        The diagonal cell $(i, i)$ refers to $i$-th base model's attack success rate, which reflects the vulnerability of a single model.
        From these figures, we can see that while base models show their vulnerabilities against adversarial attack, only DVERGE and TRS ensemble could achieve low adversarial transferability among base models.
        %by adapting effective diverse training. 
        We should also notice that though GAL applied a similar gradient cosine similarity loss as our loss term $\mathcal{L}_{\text{sim}}$, GAL still can not achieve low adversarial transferability due to the lack of model smoothness enforcement, which is one of our key contributions in this paper.

    \noindent\textbf{Gradient similarity only vs. TRS.}
        \label{adx:sec-similarity-only}
    To further verify our theoretical analysis on the sufficient condition of transferability as model smoothness, we consider 
    % the case where we only apply 
    only applying
    % our 
    \emph{similarity loss} $\mathcal{L}_{\text{sim}}$ without \emph{model smoothness loss} $\mathcal{L}_{\text{smooth}}$ in TRS (i.e. $\lambda_b=0$). The result is shown as ``Cos-only'' method of Table~\ref{tab:result1}. We observe that the resulting whitebox robustness is much worse than standard TRS. This matches our  theoretical analysis that \textit{only minimizing the gradient similarity cannot guarantee low adversarial transferability among base models} and thus lead to low ensemble robustness. In \Cref{adx:separate-effects-section}, we investigate the impacts of $\mathcal{L}_{\text{sim}}$ and $\mathcal{L}_{\text{smooth}}$ thoroughly, and we show that though $\mathcal{L}_{\text{smooth}}$  contribute slightly more, both terms are critical to the final ensemble robustness.

    \noindent\textbf{$\ell_2$ regularizer only vs. Min-max model smoothing.} 
    %\xiaojun{a better name for `naive TRS'?} 
    To emphasis the importance of our proposed  min-max training loss on promoting the margin-wise model smoothness, we train a variant of TRS ensemble \cosll, where we directly apply the $\ell_2$ regularization on $\|\nabla_x \ell_{\cF}\|_2$ and $\|\nabla_x \ell_{\cG}\|_2$. The results are shown as ``Cos-$\ell_2$'' in Table~\ref{tab:result1}. We observe that Cos-$\ell_2$ achieves lower robustness accuracy compared with TRS, which implies the necessity of regularizing the gradient magnitude on not only the local training points but also  their neighborhood regions to ensure overall model smoothness.
% ===================== Conclusion =====================
% \vspace{-1em}
\section{Conclusion}
% \vspace{-1em}
    In this paper, we deliver an in-depth understanding of adversarial transferability. 
    % In doing so, 
    Theoretically, we provide both lower and upper bounds on transferability which shows that
    % between low risk classifiers based on gradient orthogonality and smoothness. 
    % % We demonsterate theoretically, that smoother models experience tighter lower and upper bounds.
    % We proved that 
    \textit{smooth} models together with \textit{low loss gradient similarity}  guarantee low transferability.
    Inspired by our analysis, we propose TRS ensemble training to empirically reduce transferability by reducing loss gradient similarity and promoting model smoothness, yielding a significant improvement on ensemble robustness.

\begin{ack}
This work is partially supported by the NSF grant No.1910100, NSF CNS 20-46726 CAR, the Amazon Research Award, and the joint CATCH MURI-AUSMURI.
\end{ack}

\newpage
{
\small
\bibliography{paper}
\bibliographystyle{IEEEtranSN}
}

\appendix
%%%%%%%%%%%%%%%%%%%%%%%%%%%%%%%%%%%%%%%%%%%%%%%%%%%%%%%%%%%%%%%%%%%%%%%%%%%%%%%%%%%%%%%%%
% \setcounter{page}{1}
In \Cref{adx:sec-untargeted attack}, we provide the lower and upper bounds of attack transferability for untargeted attacks. In \Cref{adx:sec-tv-transferability-proofs}, we extend our theoretical analysis to adversarial attacks bounded by distributional distance. In \Cref{adx:sec-lower-bound-proofs,apdxB}, we give the detailed proof of \Cref{thm:target,thm:target-upper-bound,thm:untarget,thm:untarget-upper-bound} characterizing the lower and upper bounds for both targeted and untargeted attack transferability. In \Cref{adx:sec-omitted-introduction}, we give the detailed introduction of our baseline ensembles and evaluated whitebox attacks, including their exact configuration. In \Cref{adx:sec-experiment-details}, we present all the training details for TRS ensemble and other baselines. In \Cref{adx:sec-blackbox-results}, we give the numerical blackbox robustness evaluation results on MNIST and CIFAR-10, corresponding to \Cref{fig:blackbox-curve} in main paper. In \Cref{adx:sec-stability-results}, we analyze the statistical stability of reported robust accuracy for TRS ensemble against attacks with random start, and TRS ensemble claims its stability by showing small standard deviation. In \Cref{adx:sec-ablation-study}, we visualize the decision boundaries of different robust ensembles and investigate how adversarial training would further improve the robustness of TRS ensemble. We also show TRS ensemble remains robust under large attack iterations through convergence analysis. In \Cref{adx:other-strong-attacks}, we evaluate the robustness of TRS ensemble against other three strong blackbox attacks, and TRS ensemble still remains robust. In \Cref{adx:sec-cifar100-results}, we conduct whitebox robustness evaluation on CIFAR-100 dataset and compare other state-of-the-art robust ensembles with our proposed TRS ensemble.

\section{Additional Theoretical Results for Untargeted Attacks}
    \label{adx:sec-untargeted attack}
    
    In this appendix, we present transferability lower and upper bounds for untargeted attack.
    All these bounds have similar forms as their targeted attack counterparts in the main text.

    \subsection{Lower Bound}

        \begin{Thm}[Lower Bound on Untargeted Attack Transferability]
        	\label{thm:untarget}

        	Assume both models $\cF$ and $\cG$ are $\beta$-smooth. 
            Let $\Att_U$ be an ($\alpha, \model$)-effective untargeted attack with perturbation ball $\lVert \delta \rVert_2 \le \epsilon$.
            The transferabiity can be lower bounded by

        % 	Under the same setting as Theorem \ref{thm:target}.
        %     % 	Consider an instance $x \in R^n$ with true label $y$ and adversarial target $y_t$.
        %     % 	The input $x$ has $n$ dimensions.
        % 	An ($\alpha, \model$)-effective (untargeted) adversarial attack $\xA = \Att_U(x)$ with perturbation ball $\lVert \delta \rVert_2 \le \epsilon$ is transferable to $\modelt$ with probability
        	$$
        	   % \begin{aligned}
            	   % & 
            	    \Pr{T_r(\model, \modelt, x) = 1} \ge (1-\alpha) - (\eta_\model + \eta_\modelt) 
            	    - %\\
            	   % & 
            	    \dfrac{\epsilon(1+ \alpha)-c_\model(1-\alpha) }{\epsilon-c_\modelt} 
                    - \dfrac{\epsilon(1-\alpha) }{\epsilon-c_\modelt} \sqrt{2 - 2 \underline{\mathcal{S}}(\el, \lt)},
                % \end{aligned}
            $$
        	where
        	$$
        	   % \begin{aligned}
        	   \resizebox{1.0\textwidth}{!}{$
        	        \displaystyle
        	        c_\model = \min_{(x,y)\in\supp(\cD)} \dfrac{ \underset{y'\in\cY: y'\neq y}{\min} \el(\Att_U(x), y') - \el(x, y) - \beta\epsilon^2 / 2 }{ \lVert \nabla_x \el(x, y) \rVert_2 }, 
        	        c_\modelt = \max_{(x,y)\in\supp(\cD)} \dfrac{ \underset{y'\in\cY: y'\neq y}{\min} \lt(\Att_U(x), y') - \lt(x, y) + \beta\epsilon^2 / 2 }{ \lVert \nabla_x \lt(x, y) \rVert_2 }.
        	   $}
        	   % \end{aligned}
        	$$
            Here $\eta_\cF$ and $\eta_\cG$ are the \emph{risks} of models $\cF$ and $\cG$ respectively. %, defined relative to the 0-1 loss.
            The $\supp(\cD)$ is the support of benign data distribution, i.e., $x$ is the benign data and $y$ is its associated true label.
        \end{Thm}
        
        The full proof is available in \Cref{aplowerbound}.
        The discussion of the theorem is in \Cref{sec:transferability-analysis}.
    
    \subsection{Upper Bound}

        \begin{Thm}[Upper Bound on Untargeted Attack Transferability]
            \label{thm:untarget-upper-bound}
            Assume both models $\cF$ and $\cG$ are $\beta$-smooth with gradient magnitude bounded by $B$, i.e., $\|\nabla_x \ell_{\cF}(x,y)\| \le B$ and $\|\nabla_x \ell_{\cG}(x,y)\| \le B$ for any $x \in \cX$, $y \in \cY$.
            Let $\Att_U$ be an $(\alpha,\cF)$-effective untargeted attack with perturbation ball $\|\delta\|_2 \le \epsilon$.
            When the attack radius $\epsilon$ is small such that 
            $\ell_{\min} - \epsilon B \left(1 + \sqrt{\frac{1+\overline{\cS}(\el,\lt)}{2}}\right) - \beta\epsilon^2 > 0$,
            the transferability can be upper bounded by
            
            % Consider an instance $x \in \bbR^n$ with true label $y$ and adversarial target $y_t$.
            % Assume both model $\cF$ and $\cG$ are $\beta$-smooth with gradients bounded by $B$.
            % An $(\alpha,\cF)$-effective (untargeted) adversarial attack $\xA=\Att_U(x)$ with perturbation ball $\|\delta\|_2 \le \epsilon$ is transferable to $\cG$ with probability
            $$
                % \resizebox{0.485\textwidth}{!}{
                % $
                    \Pr{T_r(\cF, \cG, x) = 1} \le 
                    \dfrac{\xi_\model + \xi_\modelt}{\ell_{\min} - \epsilon B \left(1 + \sqrt{\frac{1+\overline{\cS}(\el,\lt)}{2}}\right) - \beta\epsilon^2},
                % $
                % }
            $$
            where 
            $
                \displaystyle
                \ell_{\min} = \min_{\substack{x\in\cX, y'\in\cY:\\ (x,y) \in \supp(\cD), y' \neq y}} \, (\ell_\cF(x, y'), \ell_\cG(x, y')).
            $
            Here $\xi_\cF$ and $\xi_\cG$
            are the \emph{empirical risks} of models $\cF$ and $\cG$ respectively, defined relative to a differentiable loss.
            The $\supp(\cD)$ is the support of benign data distribution, i.e., $x$ is the benign data and $y$ is its associated true label.
        \end{Thm}
        
        The full proof is available in \Cref{apupperbound}.
        The discussion of the theorem is in \Cref{sec:transferability-analysis}.

\section{Discussion: Beyond \texorpdfstring{$\ell_p$}{Lp} Attack}
\label{adx:sec-tv-transferability-proofs}

Besides the widely used $\ell_p$ norm based adversarial examples, here we plan to extend our understanding to the distribution distance analysis.
    
    We no longer distinguish the targeted attack and untargeted attack.
    Therefore, we denote either of them by $\Att$.
    Accordingly, we revise the definition of $(\alpha,\model)$-effective attack~(\Cref{def:adv_effect}) to be $\Pr{\model(\Att(x) \neq y)} \ge 1 - \alpha$ where $y\in \cY$ is the true label of $x\in \cX$.
    
    % We define an \term{attack strategy} as a function $\Att(x)\in\cX$ on \term{target point} $x\in\cX$, which seeks an adversarial instance $\Attack{x}$, such that $\model(\Attack{x}) \ne \model(x)$.
    Moreover, we use $\cP_{\Attack{x}}$ to represent the distribution of $\Attack{x} \in \cX$ where $x$ is distributed according to $\cP_\cX$.
% }

    Now we define the distribution distance that we use to measure the adversarial distribution gap.
    \begin{Df}[Total variation distance;~\cite{chambolle2004algorithm}] For two probability distributions $\cP_\cX$ and $\cP_{\Attack{x}}$ on $\cX$, the total variation distance between them is defined as
    \[
    \|\cP_\cX - \cP_{\Attack{x}} \|_{TV} = \underset{C \subseteq \cX}{\sup} |\cP_\cX(C)-\cP_{\Attack{x}}(C)|.
    \]
    \end{Df}
    
    Informally, the total variation distance measures the largest change in probability over all events. 
    For discrete probability distributions, the total variation distance is the $\ell_1$ distance between the vectors in the probability simplex representing the two distributions. 
    % We would leverage the total variance distance to measure benign and adversarial data distributions for data induced transferability.
    
    \begin{Df}
    \label{attack_dis}
    Given $\rho \in(0,1)$, an attack strategy $\Attack{\cdot}$ is called \term{$\rho$-conservative},
    % \footnote{We use the total variation distance as it is a natural way to measure distances between distributions. Other notions of distance may also be applied.}
    if for $x \sim \cP_\cX$, $\|\cP_\cX-\cP_{\Attack{x}}\|_{TV} \leq \rho$.
    \end{Df}
    This definition formalizes the general objective of generating adversarial examples against deep neural networks: attack samples are likely to be observed, while they do not themselves arouse suspicion.

    \begin{Lem}
        \label{lemma:domain1}
        Let $f, g: \cX\to\cY$ be classifiers, $\delta, \rho, \epsilon\in(0,1)$ be constants, and $\Attack{\cdot}$ be an attack strategy. Suppose that $\Attack{\cdot}$ is $\rho$-conservative and $f,g$ have risk at most $\epsilon$. Then $$\Pr{\model(\Attack{x}) \ne \modelt(\Attack{x})}\leq 2\epsilon+ \rho$$ for a given random instance $x \sim \cP_\cX$. 
    \end{Lem}
    
    \begin{remark}
        This result provides theoretical backing for the intuition that the boundaries of low risk classifiers under certain dense data distribution are close~\cite{tramer2017space}. 
        It considers two classifiers that have risk at most $\epsilon$, which indicates their boundaries are close for benign data. It then shows that their boundaries are also close for the perturbed data as long as the attack strategy satisfies a conservative condition which constrains the drift in distribution between the benign and adversarial data. 
    \end{remark}

    \begin{proof}[Proof of Lemma~\ref{lemma:domain1}]
        Given $\Attack{\cdot}$ is \term{$\rho$-conservative}, by Definition~\ref{attack_dis} we know
        \begin{eqnarray*}
        \begin{aligned}
        & \left \vert \mathrm{Pr}_{\cP_\cX} \left(f(\Attack{x})=g(\Attack{x})\right) - \mathrm{Pr}_{\cX} \left(f(x)=g(x)\right) \right \vert \\
        = & \left \vert \mathrm{Pr}_{\Attack{x}} \left(f(x)=g(x)\right) - \mathrm{Pr}_{\cX} \left(f(x)=g(x)\right) \right \vert  \\
        \leq & \rho.
        \end{aligned}
        \end{eqnarray*} 
        Therefore, we have
        \[
        \Pr {f(\Attack{x})=g(\Attack{x})} \\
        \geq  \Pr {f(x)=g(x)} - \rho.
        \]
        From the low-risk conditions, the classifiers agree with high probability:
        \begin{eqnarray*}
        \begin{aligned}
         &\Pr{ f(\Attack{x}) \ne g(\Attack{x}) }\\
        \leq & \Pr {f(x) \ne g(x)} +\rho \\
        \leq &  1- \Pr {f(x)=y, g(x) = y} +\rho \;, \; \ \footnotemark \\
        \leq & 1 - ( 1- \Pr{ f(x)\ne y} -   \Pr{g(x) \ne y)} + \rho\\ 
        = & \epsilon + \epsilon + \rho  \\
        %\risk{f} + \risk{g} +\rho \\
        \leq &  2\epsilon + \rho \enspace,
        \end{aligned}
        \end{eqnarray*}
        where the third inequality follows from the union bound.
        \footnotetext{Here we assume $y$ is the ground truth label.}\footnote{Recall that for arbitrary events $A_1,\ldots,A_n$, the union bound implies
        $P\left(\bigcap_{i=1}^n A_i\right) \geq 1 - \sum_{i=1}^n P\left(\overline{A_i}\right)$.}
    \end{proof}
    
    \begin{Thm}
    	Let $\model, \modelt: \cX\to\cY$ be classifiers ($\cY \in \{-1,1\}$), $\delta, \rho, \epsilon\in(0,1)$ be constants, and $\Attack{\cdot}$ an attack strategy. Suppose that $\Attack{\cdot}$ is $\rho$-conservative and $\model,\modelt$ have risk at most $\epsilon$. Given random instance $x\in \cX$, if $\Attack{\cdot}$ is $(\delta,\model)$-effective, 
    % 	then it is also $(\delta+4\epsilon+\rho,\modelt)$-effective.
        then it is also $(\delta+2\epsilon+\rho, \modelt)$-effective.
        \label{thm:transfer_unbound_binary}
    \end{Thm}
    
    % \begin{proof}
    % From Lemma~\ref{lemma:domain1} and the union bound we have
    % \begin{eqnarray*}
    % \begin{aligned}
    % &\Pr{f(x) \ne f(\Attack{x})} \\
    % \geq & \Pr{g(x) \ne g(\Attack{x}), f(x)=g(x), f(\Attack{x}) = g(\Attack{x})} \\
    % \geq & 1 -  \Pr{g(x) = g(\Attack{X})}  
    % - \Pr{f(x) \ne g(x)} %\\
    % % & \qquad\qquad 
    % - \Pr{f(\Attack{x}) \ne g(\Attack{x}) } \\
    % \geq & 1 - \delta - 4\epsilon - \rho \enspace,
    % \end{aligned}
    % \end{eqnarray*}
    % as claimed.
    % \end{proof}
    
    The proof is shown below.
    This result formalizes the intuition that low-risk classifiers possess close decision boundaries in high-probability regions. 
    % The proof is deferred to \Cref{adx:sec-tv-transferability-proofs}. 
    In such settings, an attack strategy that successfully attacks one classifier would have high probability to mislead the other. 
    This theorem explains why we should expect successful transferability in practice: 
    defenders will naturally prefer low-risk binary classifiers. This desirable quality of classifiers is a potential liability.

% \begin{customthm}{5}[restated]
%     \label{transfer_unbound_binary}
% 	Let $\model, \modelt: \cX\to\cY$ be classifiers ($\cY \in \{-1,1\}$), $\delta, \rho, \epsilon\in(0,1)$ be constants, and $\Attack{\cdot}$ an attack strategy. Suppose that $\Attack{\cdot}$ is $\rho$-conservative and $\model,\modelt$ have risk at most $\epsilon$. Given random instance $x\in \cX$, if $\Attack{\cdot}$ is $(\delta,\model)$-effective, then it is also $(\delta+4\epsilon+\rho,\modelt)$-effective.
% \end{customthm}

\begin{proof}[Proof of \Cref{thm:transfer_unbound_binary}]
    From Lemma~\ref{lemma:domain1} and the union bound we have
    % \begin{eqnarray*}
    % \begin{aligned}
    % &\Pr{g(x) \ne g(\Attack{x})} \\
    % \geq & \Pr{f(x) \ne f(\Attack{x}), g(x)=f(x), g(\Attack{x}) = f(\Attack{x})} \\
    % \geq & 1 -  \Pr{f(x) = f(\Attack{x})}  
    % - \Pr{g(x) \ne f(x)}
    % % \\
    % % & \hspace{2em} 
    % - \Pr{g(\Attack{x}) \ne f(\Attack{x}) } \\
    % \geq & 1 - \delta - 4\epsilon - \rho \enspace,
    % \end{aligned}
    % \end{eqnarray*}
    \begin{eqnarray*}
    \begin{aligned}
    &\Pr{g(x) \ne y} \\
    \geq & \Pr{f(\Attack{x}) \neq y, g(\Attack{x}) = f(\Attack{x}) } \\
    \geq & 1 - \Pr{f(\Attack{x}) = y} - \Pr{ g(\Attack{x}) \neq f(\Attack{x}) } \\
    \geq & 1 - \delta - 2\epsilon-\rho,
    \end{aligned}
    \end{eqnarray*}
    as claimed.
\end{proof}

\allowdisplaybreaks

% \newpage
\section{Proof of Transferability Lower Bound~(Theorems \ref{thm:target} and \ref{thm:untarget})}\label{aplowerbound}
\label{adx:sec-lower-bound-proofs}
Here we present the proof of Theorem \ref{thm:target} and Theorem \ref{thm:untarget} stated in Section~\ref{sec:lower-bound} and \Cref{adx:sec-untargeted attack}.

The following lemma is used in the proof.
\begin{Lem}
	\label{lem:cosineBound}
	For arbitrary vector $\delta$, $x$, $y$, suppose $\lVert \delta \rVert_2 \le \epsilon$, $x$ and $y$ are unit vectors, i.e., $\lVert x \rVert_2 = \lVert y \rVert_2 = 1$. Let $m := \cos \langle x, y \rangle = \dfrac{x \cdot y}{\lVert x \rVert_2 \cdot \lVert y \rVert_2}$. Let $c$ denote any real number. Then
	\begin{equation*}
	\delta \cdot y > c + \epsilon \sqrt{2 - 2m} \;\;\Rightarrow\;\; \delta \cdot x > c.
	\end{equation*}
\end{Lem}

\begin{proof}
$\delta \cdot x = \delta \cdot y + \delta \cdot (x - y) > c + \epsilon \sqrt{2 - 2m} + \delta \cdot (x - y)$. By law of cosines, $\delta \cdot (x-y) \ge -\epsilon \sqrt{2 - 2\cos \langle x, y \rangle} = -\epsilon\sqrt{2 - 2m}$. Hence, $\delta \cdot x > c$. 
% $\Box$
\end{proof}

\begin{theorem*}[Lower Bound on Targeted Attack Transferability]

            Assume both models $\cF$ and $\cG$ are $\beta$-smooth. 
            Let $\Att_T$ be an ($\alpha, \model$)-effective targeted attack with perturbation ball $\lVert \delta \rVert_2 \le \epsilon$ and target label $y_t \in \cY$.
            The transferabiity can be lower bounded by
            $$
                % \begin{aligned}
            	   % & 
            	    \Pr{T_r(\model, \modelt, x, y_t) = 1}  \ge (1-\alpha) - (\eta_\model + \eta_\modelt) - 
            	   % & \hspace{2em} 
            	    \dfrac{\epsilon(1+\alpha) +  c_\model(1-\alpha)}{c_\modelt + \epsilon} - \dfrac{ \epsilon (1-\alpha)}{c_\modelt + \epsilon}\sqrt{2 - 2 \underline{\mathcal{S}}(\el, \lt)},
            % 	\end{aligned}
            $$
        	where 
        	$$
        	    \begin{aligned}
                	c_\model & = \max_{x \in \cX} \dfrac{ \underset{y\in \cY}{\min}\, \el(\Att_T(x), y) - \el(x, y_t) + \beta\epsilon^2 / 2 }{ \lVert \nabla_x \el(x, y_t) \rVert_2 }, \\
        	        c_\modelt & = \min_{x \in \cX} \dfrac{ \underset{y\in \cY}{\min}\, \lt(\Att_T(x), y) - \lt(x, y_t) - \beta\epsilon^2 / 2 }{ \lVert \nabla_x \lt(x, y_t) \rVert_2 }.
        	   \end{aligned}
        	$$
            Here $\eta_\cF,\eta_\cG$ are the \emph{risks} of models $\cF$ and $\cG$ respectively.

% 	Consider an instance $x \in \bbR^n$ with  true label $y$ and adversarial target $y_t$.
%     Assume both model $\cF$ and $\cG$ are $\beta$-smooth. 
% 	An ($\alpha, \model$)-effective (targeted) attack $\xA = \Att_T(x)$ with perturbation ball $\lVert \delta \rVert_2 \le \epsilon$ is transferable to $\modelt$ with bounded probability
% 	\begin{small}
% 	\begin{align*}
% 	\Pr{T_r(\model, \modelt, x, y_t) = 1}  \ge (1-\alpha) - (\eta_\model + \eta_\modelt) 
% 	 - \dfrac{\epsilon(1+\alpha) +  c_\model(1-\alpha)}{c_\modelt + \epsilon} \nonumber
% 	- \dfrac{ \epsilon (1-\alpha)}{c_\modelt + \epsilon}\sqrt{2 - 2 \underline{\mathcal{S}}(\el, \lt)},
% 	\end{align*}
% 	\end{small}
%     \vspace{-0.5em}
% 	\begin{small}
% 	\begin{align*}
% 	\text{where}\quad &c_\model = \max_{x \in \cX} \dfrac{ \min_y \el(\xA, y) - \el(x, y_t) + \beta\epsilon^2 / 2 }{ \lVert \nabla_x \el(x, y_t) \rVert_2 },\,
% 	c_\modelt = \min_{x \in \cX} \dfrac{ \min_y \lt(\xA, y) - \lt(x, y_t) - \beta\epsilon^2 / 2 }{ \lVert \nabla_x \lt(x, y_t) \rVert_2 }, \\
% 	&
%     \eta_\cF = \Pr{\cF(x) \neq y},\, \eta_\cG = \Pr{\cG(x) \neq y}.
% 	\end{align*}
% 	\end{small}
%     Here $\eta_\cF,\eta_\cG$ are the \emph{risks} of models $\cF$ and $\cG$ respectively.
\end{theorem*}

\begin{proof}

For simplifying the notations, we define $\xA := \cA_T(x)$, which is the generated adversarial example by $\cA_T$ when the input is $x$.

Define auxiliary function $f, g: \cX \mapsto \bbR$ such that
$$
\begin{aligned}
    f(x) = \dfrac{\min_{y \in \cY} \ell_\cF(\xA,y) - \ell_\cF(x,y_t) + \beta\epsilon^2/2}{\|\nabla_x \ell_\cF(x,y_t)\|_2}, \\
    g(x) = \dfrac{\min_{y \in \cY} \ell_\cG(\xA,y) - \ell_\cG(x,y_t) - \beta\epsilon^2/2}{\|\nabla_x \ell_\cG(x,y_t)\|_2}.
\end{aligned}
$$
The $f$ and $g$ are orthogonal to the confidence score functions of model $\cF$ and $\cG$.
Note that $c_\cF = \max_{x\in \cX} f(x)$ and $c_\cG = \min_{x \in \cX} g(x)$.

The transferability of concern satisfies:
\begin{align}
    & \Pr{T_r(\model, \modelt, x, y_t) = 1} \nonumber \\ 
  =& \Pr{\model(x) =
	y \cap \modelt(x) = y \cap \model(\xA) = y_t \cap
         \modelt(\xA) = y_t} \label{ap1} \\ 
  \ge & 1 - \Pr{\model(x) \neq y} - \Pr{\modelt(x) \neq y} %\nonumber \\
    % & \hspace{2em} 
    - \Pr{\model(\xA) \neq y_t} - \Pr{\modelt(\xA) \neq y_t} \label{ap2}   \\ 
  \ge & 1 - \eta_\model - \eta_\modelt - \alpha - \Pr{\modelt(\xA) \neq y_t}.
\end{align}
Eq. \ref{ap1} follows the definition~(Definition~\ref{def:trans}).
Eq. \ref{ap1} to Eq. \ref{ap2} follows from the union bound.
% \footnote{Union Bound: For
% 	arbitrary events $\mathcal{E}_i$, $\Pr{\bigcap_{i=1}^{n} \mathcal{E}_i}
% 	\ge 1 - \sum_{i=1}^{n} \overline{\mathcal{E}_i}$}.
From Eq. \ref{ap1} to Eq. \ref{ap2} definition of model risk~(Definition~\ref{def:risk}) and definition of adversarial effectiveness~(Definition~\ref{def:adv_effect}) are applied.

Now consider $\Pr{\model(\xA) \neq y_t}$ and $\Pr{\modelt(\xA) \neq y_t}$.
Given that model predicts the label for which $\el$ is minimized, $\model(\xA) \neq y_t \iff \el(x + \delta, y_t) > \min_y \el(x + \delta, y)$.
Similarly, $\modelt(\xA) \neq y_t \iff \lt(x + \delta, y_t) > \min_y \lt(x + \delta, y)$.

Following Taylor's Theorem with Lagrange remainder, we have
\begin{align}
    \el(x + \delta, y_t) & = \el(x, y_t) + \delta \nabla_x \el(x, y_t) + \dfrac{1}{2} {\xi}^{\top}\mathbf{H_\model}\xi , \label{ap4} \\
    \lt(x + \delta, y_t) & = \lt(x, y_t) + \delta \nabla_x \lt(x, y_t) + \dfrac{1}{2} {\xi}^{\top}\mathbf{H_\modelt}\xi.
    \label{ap5}
\end{align}
In Eq. \ref{ap4} and Eq. \ref{ap5}, $\xi = k\delta$ for some $k \in [0,1]$.
$\mathbf{H_\model}$ and $\mathbf{H}_\modelt$ are Hessian matrices of $\el$ and $\lt$ respectively.
Since $\ell_\cF(x+\delta, y_t)$ and $\ell_\cG(x+\delta, y_t)$ are $\beta$-smooth,
the maximum eigenvalues of $\mathbf{H}_\cF$ and $\mathbf{H}_\cG$ are bounded by $\beta$,
As the result, $|\xi^\top \mathbf{H}_\cF \xi| \le \beta \cdot \|\xi\|_2^2 \le \beta \epsilon^2$.
% Given that Hessian matrix is symmetric, $\mathbf{H}_\model$ and $\mathbf{H}_\modelt$ are diagonalizable.
% Thus, ${\xi}^{\top}\mathbf{H}_\model\xi = \sum_{i=1}^n \lambda_i y_i^2$, where $\lambda_i$ are eigenvalues and $\lVert y \rVert_2 \le \epsilon$.
% $|\lambda_i|$ is bounded by $\gamma_\model$.
% Hence, $-n\gamma_\model\epsilon^2 \le {\xi}^{\top}\mathbf{H_\model}\xi \le n\gamma_\model\epsilon^2$.
% Similarly, $-n\gamma_\modelt\epsilon^2 \le \xi^\top \mathbf{H}_\modelt \xi \le n\gamma_\modelt\epsilon^2$.
Applying them to Eq. \ref{ap4} and Eq. \ref{ap5}, we thus have
\begin{align}
& \el(x, y_t) + \delta \nabla_x \el(x, y_t) - \dfrac{1}{2}
\beta\epsilon^2  \le  \el(x + \delta, y_t) %\nonumber \\
% & 
\le  \el(x, y_t) + \delta \nabla_x \el(x, y_t) +\dfrac{1}{2} \beta\epsilon^2, \label{ap6} \\
& \lt(x, y_t) + \delta \nabla_x \lt(x, y_t) - \dfrac{1}{2} 
\beta\epsilon^2  \le  \lt(x + \delta, y_t) %\nonumber \\
% & 
\le  \lt(x, y_t) + \delta \nabla_x \lt(x, y_t) +\dfrac{1}{2} \beta\epsilon^2. \label{ap7}
\end{align}

Apply left hand side of Eq. \ref{ap6} to $\Pr{\model(\xA) \neq y_t} \le \alpha$~(from Definition~\ref{def:adv_effect}):
\begin{align*}
& \Pr{\model(\xA) \neq y_t} \\
= & \Pr{\el(x + \delta, y_t) > \min_y \el(x + \delta, y)} \\
\ge & \Pr{\el(x, y_t) + \delta \nabla_x \el(x, y_t) -\frac{1}{2}\beta\epsilon^2 >  \min_y \el(x + \delta, y)} \\
= & \Pr{ \delta \cdot \frac{\nabla_x \el(x, y_t)}{\lVert \nabla_x \el(x, y_t) \rVert_2} > f(x) }, \\
\Longrightarrow\  & \Pr{ \delta \cdot \frac{\nabla_x \el(x, y_t)}{\lVert \nabla_x \el(x, y_t) \rVert_2} > f(x) } \le \alpha. %\label{ap8}
\end{align*}

Similarly, we apply right hand side of Eq. \ref{ap7} to $\Pr{\modelt(\xA) = y_t}$:
\begin{align}
& \Pr{\modelt(\xA) \neq y_t} \nonumber \\
= & \Pr{\lt(x + \delta, y_t) > \min_y \lt(x + \delta, y)} \nonumber \\
\le & \Pr{\lt(x, y_t) + \delta \nabla_x \lt(x, y_t) +\frac{1}{2}\beta\epsilon^2 >  \min_y \lt(x + \delta, y)} \nonumber \\
= & \Pr{ \delta \cdot \frac{\nabla_x \lt(x, y_t)}{\lVert \nabla_x \lt(x, y_t) \rVert_2} > g(x) }. \label{ap16}
\end{align}

Knowing that $\lVert \delta \rVert_2 \le \epsilon$, from Lemma~\ref{lem:cosineBound} we have
\begin{align}
& \delta \cdot \frac{\nabla_x \lt(x, y_t)}{\lVert \nabla_x \lt(x, y_t) \rVert_2} > f(x) + \epsilon \sqrt{2 - 2\underline{\mathcal{S}}(\el, \lt) } \label{ap9} \\
\Longrightarrow \  & \delta \cdot \frac{\nabla_x \lt(x, y_t)}{\lVert \nabla_x \lt(x, y_t) \rVert_2} > f(x) + %\nonumber \\
% & \hspace{3em} 
\epsilon \sqrt{2 - 2\cos \langle \nabla_x \el(x, y_t), \nabla_x \lt(x, y_t) \rangle } \label{ap10} \\
\Longrightarrow \  & \delta \cdot \frac{\nabla_x \el(x, y_t)}{\lVert \nabla_x \el(x, y_t) \rVert_2} > f(x). \label{ap11}
\end{align}
From Eq. \ref{ap9} to Eq. \ref{ap10}, the infimum in definition of $\underline{\mathcal{S}}$~(Definition \ref{def:similarity}) indicates that 
    $$
        \underline{\mathcal{S}}(\el, \lt) \le \cos \langle \nabla_x \el(x, y_t), \nabla_x \lt(x, y_t) \rangle.
    $$
Hence, 
    $$
        f(x) + \epsilon \sqrt{2 - 2\underline{\mathcal{S}}(\el, \lt) } 
        \ge  f(x) + \epsilon \sqrt{2 - 2\cos \langle \nabla_x \el(x, y_t), \nabla_x \lt(x, y_t) \rangle }.
    $$
Eq. \ref{ap10} to Eq. \ref{ap11} directly uses Lemma \ref{lem:cosineBound}.
As a result,
\begin{equation*}
\begin{aligned}
    & \Pr{ \delta \cdot \frac{\nabla_x \lt(x, y_t)}{\lVert \nabla_x \lt(x, y_t) \rVert_2} > f(x) + \epsilon \sqrt{2 - 2\underline{\mathcal{S}}(\el, \lt) } } \\
    \le & \Pr{ \delta \cdot \frac{\nabla_x \el(x, y_t)}{\lVert \nabla_x \el(x, y_t) \rVert_2} > f(x) } \le \alpha. \label{ap12}
\end{aligned}
\end{equation*}
Note that $f(x) \le c_\cF$, we have
\begin{equation*}
    \Pr{ \delta \cdot \frac{\nabla_x \lt(x, y_t)}{\lVert \nabla_x \lt(x, y_t) \rVert_2} > c_\cF + \epsilon \sqrt{2 - 2\underline{\mathcal{S}}(\el, \lt) } }
    \le
    \alpha.
\end{equation*}
Now we consider the maximum expectation of $\delta \cdot \frac{\nabla_x \lt(x, y_t)}{\lVert \nabla_x \lt(x, y_t) \rVert_2}$.
Its maximum is $\max \, \lVert \delta \rVert_2 = \epsilon$.
Therefore, its expectation is bounded:
$$
    % \resizebox{\linewidth}{!}{
    % $
    \bbE \left[ \delta \cdot \dfrac{\nabla_x \ell_\cG(x, y_t)}{\| \nabla_x \ell_\cG(x, y_t) \|_2} \right] \le \epsilon\cdot\alpha + \left( c_\cF + \epsilon \sqrt{2 - 2\underline{\cS}(\ell_\cF, \ell_\cG)} \right) (1 - \alpha). 
    % $
    % }
$$
Now applying Markov's inequality, we get
\begin{align*}
    & \Pr{\delta \cdot \dfrac{\nabla_x \ell_\cG(x,y_t)}{\| \nabla_x \ell_\cG(x,y_t) \|_2} > c_\cG } \\
    \le & \dfrac{\epsilon \cdot \alpha + \left(c_\cF + \epsilon \sqrt{2 - 2\underline{\cS}(\ell_\cF, \ell_\cG)}\right) (1-\alpha) + \epsilon}{c_\cG + \epsilon} \\
    = & \dfrac{\epsilon(1 + \alpha) + \left( c_\cF + \epsilon \sqrt{2 - 2\underline{\cS}(\ell_\cF, \ell_\cG)}\right) (1-\alpha)}{c_\cG + \epsilon}.
\end{align*}
Since $g(x) \ge c_\cG$,
$$
    \small
    \begin{aligned}
        & \Pr{\delta \cdot \dfrac{\nabla_x \ell_\cG(x,y_t)}{\| \nabla_x \ell_\cG(x,y_t) \|_2} > g(x) } 
        \le  \Pr{\delta \cdot \dfrac{\nabla_x \ell_\cG(x,y_t)}{\| \nabla_x \ell_\cG(x,y_t) \|_2} > c_\cG } \\
        \le & \dfrac{\epsilon(1 + \alpha) + \left( c_\cF + \epsilon \sqrt{2 - 2\underline{\cS}(\ell_\cF, \ell_\cG)}\right) (1-\alpha)}{c_\cG + \epsilon}.
    \end{aligned}
$$
Combine with Eq.~\ref{ap11},
finally,
\begin{align*}
& \Pr{T_r(\model, \modelt, x, y_t) = 1} \nonumber \\
\ge & 1 - \eta_\model - \eta_\modelt - \alpha - \Pr{\modelt(\xA) \neq y_t}\\  %\label{ap13} \\
\overset{(i.)}{\ge} & 1 - \eta_\model - \eta_\modelt - \alpha - \Pr{ \delta \cdot \frac{\nabla_x \lt(x, y_t)}{\lVert \nabla_x \lt(x, y_t) \rVert_2} > g(x) }\\ % \label{ap14} \\
\ge & 1 - \eta_\model - \eta_\modelt - \alpha %\nonumber \\
% & \hspace{1em} 
- \dfrac{ \epsilon(1+\alpha) + \left(c_\model + \epsilon\sqrt{2 - 2\underline{\mathcal{S}}(\el, \lt)}\right)(1-\alpha) }{ c_\modelt + \epsilon } \nonumber \\ %\label{ap15} \\
= & (1-\alpha) - (\eta_\model + \eta_\modelt) - \dfrac{\epsilon(1+\alpha) +  c_\model(1-\alpha)}{c_\modelt + \epsilon} %\nonumber \\
% & \hspace{1em} 
- \dfrac{ \epsilon (1-\alpha)}{c_\modelt + \epsilon}\sqrt{2 - 2 \underline{\mathcal{S}}(\el, \lt)}. \nonumber %\label{ap17}.
\end{align*}
Here, $(i.)$ follows Eq. \ref{ap16}.
% And Eq. \ref{ap14} to Eq. \ref{ap15} pumps Eq. \ref{ap12}. 
% $\Box$
\end{proof}

\begin{theorem*}[Lower Bound on Untargeted Attack Transferability]

    	Assume both models $\cF$ and $\cG$ are $\beta$-smooth. 
        Let $\Att_U$ be an ($\alpha, \model$)-effective untargeted attack with perturbation ball $\lVert \delta \rVert_2 \le \epsilon$.
        The transferabiity can be lower bounded by
    	$$
    	   % \begin{aligned}
        	   % & 
        	    \Pr{T_r(\model, \modelt, x) = 1} \ge (1-\alpha) - (\eta_\model + \eta_\modelt) 
        	    - 
        	   % \\
        	   % & 
        	    \dfrac{\epsilon(1+ \alpha)-c_\model(1-\alpha) }{\epsilon-c_\modelt} 
                - \dfrac{\epsilon(1-\alpha) }{\epsilon-c_\modelt} \sqrt{2 - 2 \underline{\mathcal{S}}(\el, \lt)},
            % \end{aligned}
        $$
    	where
    	$$
    	    \begin{aligned}
    	        c_\model & = \min_{(x,y)\in\supp(\cD)} \dfrac{ \underset{y'\in\cY: y'\neq y}{\min} \el(\Att_U(x), y') - \el(x, y) - \beta\epsilon^2 / 2 }{ \lVert \nabla_x \el(x, y) \rVert_2 }, \\
    	        c_\modelt & = \max_{(x,y)\in\supp(\cD)} \dfrac{ \underset{y'\in\cY: y'\neq y}{\min} \lt(\Att_U(x), y') - \lt(x, y) + \beta\epsilon^2 / 2 }{ \lVert \nabla_x \lt(x, y) \rVert_2 }.
    	    \end{aligned}
    	$$
        Here $\eta_\cF$ and $\eta_\cG$ are the \emph{risks} of models $\cF$ and $\cG$ respectively. %, defined relative to the 0-1 loss.
        The $\supp(\cD)$ is the support of benign data distribution, i.e., $x$ is the benign data and $y$ is its associated true label.

% % 	\label{thm:untarget}
% 	Under the same setting as Theorem \ref{thm:target}.
% % 	Consider an instance $x \in R^n$ with true label $y$ and adversarial target $y_t$.
% % 	The input $x$ has $n$ dimensions.
% 	An ($\alpha, \model$)-effective (untargeted) adversarial attack $\xA = \Att_U(x)$ with perturbation ball $\lVert \delta \rVert_2 \le \epsilon$ is transferable to $\modelt$ with probability
% 	% \begin{equation}
% 	\begin{small}
% 	\begin{align*}
% 	\Pr{T_r(\model, \modelt, x) = 1} \ge (1-\alpha) - (\eta_\model + \eta_\modelt) 
% 	 - \dfrac{\epsilon(1+ \alpha)-c_\model(1-\alpha) }{\epsilon-c_\modelt} \nonumber 
%     - \dfrac{\epsilon(1-\alpha) }{\epsilon-c_\modelt} \sqrt{2 - 2 \underline{\mathcal{S}}(\el, \lt)},
% 	\end{align*}
% 	\end{small}
% 	% \end{equation}
% % 	where 
%     \vspace{-0.5em}
% 	\begin{small}
% 		\begin{align*}
% % 	\small
% 	\text{where}\quad &c_\model = \min_{(x,y)\in \cM} \dfrac{ \underset{y';y'\neq y}{\min} \el(\xA, y') - \el(x, y) - \beta\epsilon^2 / 2 }{ \lVert \nabla_x \el(x, y) \rVert_2 }, \,
% 	c_\modelt = \max_{(x,y)\in \cM} \dfrac{ \underset{y';y'\neq y}{\min} \lt(\xA, y') - \lt(x, y) + \beta\epsilon^2 / 2 }{ \lVert \nabla_x \lt(x, y) \rVert_2 }, \\
% 	&
% % 	\eta_\cF = \bbE_{x,y} \left[ \ell_{01,\cF}(x,y) \right],
% % 	\eta_\cG = \bbE_{x,y} \left[ \ell_{01,\cG}(x,y) \right].
%     \eta_\cF = \Pr{\cF(x) \neq y}, \, \eta_\cG = \Pr{\cG(x) \neq y}.
% 	\end{align*}
% 	\end{small}
%     Here $\eta_\cF$ and $\eta_\cG$ are the \emph{risks} of models $\cF$ and $\cG$ respectively.%, defined relative to the 0-1 loss.
\end{theorem*}

\begin{proof}
For simplifying the notations, we define $\xA := \cA_U(x)$, which is the generated adversarial example by $\cA_U$ when the input is $x$.
Define auxiliary function $f, g: \cM \to \bbR$ such that
$$
\begin{aligned}
    f(x,y) = \dfrac{ \underset{y'\in\cY:y'\neq y}{\min} \el(\xA, y') - \el(x, y) - \beta\epsilon^2 / 2 }{ \lVert \nabla_x \el(x, y_t) \rVert_2 }, \\
	g(x,y) = \dfrac{ \underset{y'\in\cY:y'\neq y}{\min} \lt(\xA, y') - \lt(x, y) + \beta\epsilon^2 / 2 }{ \lVert \nabla_x \lt(x, y_t) \rVert_2 }.
\end{aligned}
$$
The $f$ and $g$ are orthogonal to the confidence score functions of model $\cF$ and $\cG$. Note that 
$$
    c_\cF = \min_{(x,y)\in \supp(\cD)} f(x,y), c_\cG = \max_{(x,y)\in \supp(\cD)} g(x,y).
$$

The proof is similar to that of Theorem \ref{thm:target}.
\begin{align}
    & \Pr{T_r(\model, \modelt, x) = 1}\nonumber  \\ = & \Pr{\model(x) = y \cap \modelt(x) = y \cap \model(\xA) \neq y \cap \modelt(\xA) \neq y}\nonumber \\% \label{ap1'} \\
    \ge & 1 - \Pr{\model(x) \neq y} - \Pr{\modelt(x) \neq y} %\nonumber \\
    % & \hspace{2em} 
    - \Pr{\model(\xA) = y} - \Pr{\modelt(\xA) = y}\nonumber \\% \label{ap2'} \\
    = & 1 - \eta_\model - \eta_\modelt - \alpha - \Pr{\modelt(\xA) = y}. \label{ap20}
\end{align}
From Taylor's Theorem and \Cref{lem:cosineBound}, we observe that
\begin{gather}
	\Pr{\modelt(\xA) = y} \le \Pr{ \delta \cdot \dfrac{\nabla_x \lt(x, y)}{\lVert \nabla_x \lt(x, y)  \rVert_2} < c_\modelt }, \label{ap19} \\
	\Pr{ \delta \cdot \dfrac{\nabla_x \lt(x, y)}{\lVert \nabla_x \lt(x, y)  \rVert_2} < c_\model - \epsilon \sqrt{2 - 2 \underline{\mathcal{S}}(\el, \lt)} } \le %\nonumber \\
% 	\hspace{2em} 
	\Pr{\model(\xA) = y} = \alpha.
	\label{ap18}
\end{gather}
According to Markov's inequality, Eq. \ref{ap18} implies that
\begin{align}
	& \Pr{ \delta \cdot \dfrac{\nabla_x \lt(x, y)}{\lVert \nabla_x \lt(x, y)  \rVert_2} < c_\modelt } \le %\nonumber \\
% 	& 
	\dfrac{ \epsilon(1 + \alpha) - \left(c_\model - \epsilon\sqrt{2 - 2\underline{\mathcal{S}}(\el, \lt)}\right)(1-\alpha) }{ \epsilon - c_\modelt }.% \\
	%& = \dfrac{ \epsilon(2 - \alpha) - \left(c_\model - \epsilon\sqrt{2 - 2\mathcal{S}(\el, \lt)}\right)\alpha }{ \epsilon - c_\modelt } 
	\label{ap21}
\end{align}
We conclude the proof by combining Eq. \ref{ap19} with Eq. \ref{ap21} and plugging it into Eq. \ref{ap20}.
% \begin{align}
% 	\Pr{T_r(\model, \modelt, x) = 1} & \ge 1 - \eta_\model - \eta_\modelt - \alpha - \Pr{\modelt(\xA) = y} \\
% 	& \ge (1-\alpha) - (\eta_\model + \eta_\modelt) - \dfrac{ \epsilon(1 + \alpha) - \left(c_\model - \epsilon\sqrt{2 - 2\underline{\mathcal{S}}(\el, \lt)}\right)(1-\alpha) }{ \epsilon - c_\modelt } \\
% 	& = (1-\alpha) - (\eta_\model + \eta_\modelt) - \dfrac{\epsilon(1 + \alpha) - c_\cF (1 - \alpha )}{\epsilon-c_\modelt} - \dfrac{\epsilon(1-\alpha)}{\epsilon-c_\modelt} \sqrt{2 - 2 \underline{\mathcal{S}}(\el, \lt)}.
% \end{align}
% This completes the proof. 
% % $\Box$
\end{proof}

% \section{new proof of upper bound}
% \begin{align*}
%   &\Pr{T_r(\model, \modelt, x, y_t) = 1} \\
% & = \Pr[\model(x) =
% y \cap \modelt(x) = y \cap \model(x^*) = y_t \cap \modelt(x^*) =
% y_t]  \\
% & \le \Pr{\model(x^*) = y_t \cap \modelt(x^*) = y_t}\\
% % & = 2 - (1 - \alpha) - \Pr{\modelt(x^*) \ne y_t} \\
% % & = 1+ \alpha - \Pr{\modelt(x^*) \ne y_t} \\
%   & =   \alpha+\Pr{\modelt(x^*)=y_t} \\
%  & \le  \alpha + \dfrac{ \epsilon(2 - \alpha) - \left(c_\model - \epsilon\sqrt{2 - 2\mathcal{S}(\el, \lt)}\right)\alpha }{ \epsilon - c_\modelt }\\
% \end{align*}

\section{Proof of Transferability Upper Bound~(Theorems \ref{thm:target-upper-bound} and \ref{thm:untarget-upper-bound})} \label{apdxB}
\label{apupperbound}
%Let $g'$ be the classifier $g$ after gaussian filtering.
%\begin{align*}
% \frac{\Pr{Tr(f,g')]}{\Pr{Tr(f,g)]} &= \frac{\Pr[f(x)=g'(x), %f(x^*)\neq f(x), g'(x^*) \neq g'(x)}}{\Pr[f(x) = g(x), f(x^*)\neq %f(x), g(x^*)\neq g(x)}} \\
% &\le \frac{\Pr{f(x)=g'(x), f(x^*)\neq f(x), g'(x^*) \neq %g'(x)]}{\Pr[f(x) = g(x), f(x^*)\neq f(x), g(x^*)\neq g(x), g(x) = %g'(x)}}

%Comment: The proof relies on the fact that we need some confidence on the correct class. Otherwise, we could have a very low risk classifier but with low confidence everywhere (low confidence meaning it is very close to the confidence of runner-up class), and that the attacker could choose to attack in the direction halfway between $\nabla_f$ and $\nabla_g$ and succeed with very high probability. That is why CE loss is used here, since it helps to prove a confidence of the correct label. Are there ways to link risk with confidence, or other ways to do this...?

Here we present the proof of Theorem~\ref{thm:target-upper-bound} and Theorem~\ref{thm:untarget-upper-bound} as stated in Section~\ref{sec:model-upperbound} and \Cref{adx:sec-untargeted attack}.

The following lemma is used in the proof.
\begin{Lem}\label{lem:cosdissim}
    Suppose two unit vectors $x, y$ satisfy $x\cdot y \le S$, then for any $\delta$, we have $\min(\delta\cdot x, \delta\cdot y) \le \|\delta\|_2\sqrt{(1+S)/2}$. 
\end{Lem}

% \para{Proof}
\begin{proof}
% For sake of contradiction, suppose $\delta \cdot x \ge \|\delta\|_2\sqrt{(1+S)/2}$, $\delta \cdot y \ge \|\delta\|_2\sqrt{(1+S)/2}$. 
Denote $\alpha$ to be the angle between $x$ and $y$, then $\cos \alpha \le S$, or $\alpha \ge \arccos{S}$. If $\alpha_x, \alpha_y$ are the angles between $\delta$ and $x$ and between $\delta$ and $y$ respectively, then we have $\max(\alpha_x, \alpha_y) \ge \alpha / 2 \ge \arccos{S}/2 $.
By the half-angle formula, $\cos(\alpha / 2) \le \cos\left(\frac{\arccos{S}}{2}\right) = \sqrt{\frac{1+S}{2}}$.
Thus, $\min(\delta\cdot x, \delta \cdot y) \le \|\delta\|_2\cos(\alpha/2) \le \|\delta\|_2\sqrt{(1+S)/2}$. 
% $\Box$
\end{proof}

\begin{theorem*}[Upper Bound on Targeted Attack Transferability]

    Assume both model $\cF$ and $\cG$ are $\beta$-smooth with gradient magnitude bounded by $B$, i.e., $\|\nabla_x \ell_{\cF}(x,y)\| \le B$ and $\|\nabla_x \ell_{\cG}(x,y)\| \le B$ for any $x \in \cX, y \in \cY$.
    Let $\Att_T$ be an $(\alpha,\cF)$-effective targeted attack with perturbation ball $\|\delta\|_2 \le \epsilon$ and target label $y_t\in \cY$.
    When the attack radius $\epsilon$ is small such that 
            $\ell_{\min} - \epsilon B \left(1 + \sqrt{\frac{1+\overline{\cS}(\el,\lt)}{2}}\right) - \beta\epsilon^2 > 0$,
            the transferability can be upper bounded by
    $$
        % \resizebox{0.485\textwidth}{!}{
        % $
           \Pr{T_r(\cF, \cG, x, y_t) = 1} \le 
            \dfrac{\xi_\model + \xi_\modelt}{\ell_{\min} - \epsilon B \left(1 + \sqrt{\frac{1+\overline{\cS}(\el,\lt)}{2}}\right) - \beta\epsilon^2},
            % \label{eq:upper-bound}
        % $
        % }
    $$
    where 
    $
        \ell_{\min} = \min_{x \in \cX} \, (\ell_\model(x, y_t), \ell_\modelt(x, y_t)).
    $
    Here $\xi_\cF$ and $\xi_\cG$
    are the \emph{empirical risks} of models $\cF$ and $\cG$ respectively, defined relative to a differentiable loss.

%     Consider an instance $x \in \bbR^n$ with true label $y$ and adversarial target $y_t$.
%     Assume both model $\cF$ and $\cG$ are $\beta$-smooth with gradients bounded by $B$.
%     An $(\alpha,\cF)$-effective (targeted) attack $\xA = \Att_T(x)$ with perturbation ball $\|\delta\|_2 \le \epsilon$ is transferable to $\cG$ with bounded probability 
%     \begin{small}
%         \begin{align}
%           &\Pr{T_r(\cF, \cG, x, y_t) = 1} \le 
%           \frac{\xi_\model + \xi_\modelt}{\ell_{\min} - \epsilon B \left(1 + \sqrt{\dfrac{1+\overline{\cS}(\el,\lt)}{2}}\right) - \beta\epsilon^2},
%             % \label{eq:upper-bound}
%         \end{align}
%     \end{small}
%     \vspace{-0.5em}
% 	\begin{small}
% 		\begin{align*}
%         	\text{where} \quad & \ell_{\min} = \min_{x \in \cX} (\ell_\model(x, y_t), \ell_\modelt(x, y_t)),
%         	\, \xi_\cF = \bbE_{x,y}\left[ \ell_\cF(x,y) \right],\, 
%         \xi_\cG = \bbE_{x,y}\left[ \ell_\cG(x,y) \right].
%     	\end{align*}
%     \end{small}
%     Here $\xi_\cF$ and $\xi_\cG$
%     are the \emph{empirical risks} of models $\cF$ and $\cG$ respectively, defined relative to a differentiable loss.
\end{theorem*}

\begin{proof}
We let $x^\cA := \cA_T(x)$ be the generated adversarial example when the input is $x$.
Since $\model(x)$ outputs label for which $\ell_\model$ is minimized, we have
\begin{align}
    \model(x) = y & \Longrightarrow \ell_\model(x, y_t) > \ell_\model(x, y) \label{Beq01}
\end{align}
and similarly
\begin{align}
    \model(\xA) = y_t &\Longrightarrow \ell_\model(\xA, y) > \ell_\model(\xA, y_t), \\
    \modelt(x) = y &\Longrightarrow \ell_\modelt(x, y_t) > \ell_\modelt(x, y),\\
    \modelt(\xA) = y_t &\Longrightarrow \ell_\modelt(\xA, y) > \ell_\modelt(\xA, y_t). \label{Beq04}
\end{align}

Since $\ell_\cF(x,y)$ and $\ell_\cG(x,y)$ are $\beta$-smooth,
\begin{align*}
    \ell_\model(x, y) + \delta\cdot\nabla_x \ell_\model(x, y) + \frac{\beta}{2}\|\delta\|^2 \ge \ell_\model(\xA, y),
\end{align*}
which implies
\begin{equation}
    % \small
    \begin{aligned}
        \delta\cdot\nabla_x \ell_\model(x, y) &\ge \ell_\model(\xA, y) - \ell_\model(x, y) - \frac{\beta}{2}\|\delta\|^2 \\
        &\ge \ell_\model(\xA, y_t) - \ell_\model(x, y) - \frac{\beta}{2}\|\delta\|^2 =: c_\model'.
    \end{aligned}
     \label{Beq05}
\end{equation}
Similarly for $\modelt$,
\begin{equation}
    % \small
    \delta\cdot\nabla_x \ell_\modelt(x, y) \ge \ell_\modelt(\xA, y_t) - \ell_\modelt(x, y) - \frac{\beta}{2}\|\delta\|^2 =: c_\modelt'. \label{Beq06}
\end{equation}

Thus, 
\begin{align}
    &\Pr{\model(x) = y , \modelt(x) = y, \model(\xA) = y_t, \modelt(\xA) = y_t} \nonumber\\
    \le & \mathrm{Pr} \left(\ell_\model(x, y_t) > \ell_\model(x, y), \ell_\model(\xA, y) > \ell_\model(\xA, y_t), %\right. \nonumber \\
    % & \left. 
    \ell_\modelt(x, y_t) > \ell_\modelt(x, y), \ell_\modelt(\xA, y) > \ell_\modelt(\xA, y_t) \right) \label{Beq1}\\ 
    \le& \Pr{\delta\cdot{\nabla_x \ell_\model(x,y)}{} \ge c_\model' , \, \delta\cdot{\nabla_x \ell_\modelt(x,y)}{} \ge c_\modelt' } \label{Beq2}\\
    \le& \mathrm{Pr}\left( \left(c_\model' \le {\epsilon}{\sqrt{(1+\overline{\cS}(\el,\lt)) / 2}}\|\nabla_x \ell_\model(x, y)\|_2\right) \, \bigcup %\right. \nonumber \\
    % & \left. 
    \left(c_\modelt' \le {\epsilon}{\sqrt{(1+\overline{\cS}(\el,\lt))/2}}\|\nabla_x \ell_\modelt(x,y)\|_2\right) \right) \label{Beq3}\\ 
    \le& \Pr{c_\model' \le {\epsilon}\sqrt{(1+\overline{\cS}(\el,\lt)) / 2}\|\nabla_x \ell_\model(x, y)\|_2} %\nonumber \\
    % & 
    + \Pr{c_\modelt' \le {\epsilon}{\sqrt{(1+\overline{\cS}(\el,\lt))/2}}\|\nabla_x \ell_\modelt(x,y)\|_2}, \label{Beq4}
\end{align}
where Eq.~\ref{Beq1} comes from Eqs.~\ref{Beq01} to \ref{Beq04},  Eq.~\ref{Beq2} comes from Eq.~\ref{Beq05} and Eq.~\ref{Beq06}.  
The Eq.~\ref{Beq3} is a result of \Cref{lem:cosdissim}: either 
$$\delta\cdot \frac{\nabla_x \ell_\model(x,y)}{\|\nabla_x \ell_\model(x,y)\|_2} \le \|\delta\|_2\sqrt{(1+\overline{\cS}(\el,\lt))/2}$$ 
or 
$$\delta\cdot \frac{\nabla_x \ell_\modelt(x,y)}{||\nabla_x \ell_\modelt(x,y)||} \le \|\delta\|_2\sqrt{(1+\overline{\cS}(\el,\lt))/2}.$$

We observe that by $\beta$-smoothness condition of the loss function,
$$
    % \small
    \begin{aligned}
        c_\model' &= \ell_\model(\xA, y_t) - \ell_\model(x, y) - \frac{\beta}{2}\|\delta\|_2^2\\
        &\ge \ell_\model(x, y_t) + \delta\cdot \nabla_x \ell_\model(x, y_t) - \frac{\beta}{2}\|\delta\|_2^2  - \ell_\model(x,y) - \frac{\beta}{2}\|\delta\|_2^2.
    \end{aligned}
$$
Thus,
\begin{equation}
    % \small
    \begin{aligned}
        &\Pr{c_\model' \le {\epsilon}\sqrt{(1+\overline{\cS}(\el,\lt)) / 2}\|\nabla_x \ell_\model(x, y)\|_2} \\
        \le & \Pr{\ell_\model(x, y_t) - \ell_\model(x, y) \le {\epsilon}B (1+\sqrt{(1+\overline{\cS}(\el,\lt)) / 2}) + \beta\epsilon^2} \\
        \le & \Pr{\ell_\model(x, y) \ge \ell_\model(x, y_t) - {\epsilon}B(1+\sqrt{(1+\overline{\cS}(\el,\lt)) / 2} - \beta\epsilon^2}\\
        \le & \frac{\xi_\model}{\underset{x\in\cX}{\min}\,\ell_\model(x, y_t) - {\epsilon}B \left(1+\sqrt{(1+\overline{\cS}(\el,\lt)) / 2}\right) - \beta\epsilon^2}. 
    \end{aligned}
    \label{eq:AppUpper05}
\end{equation}
% 
% 
%\begin{align*}
%    &\le \Pr{\log p_f(y|x) - \log p_f(y_t|x) \le \log T}\\
%    &\le \Pr{\frac{p_f(y|x)}{p_f(y_t | x)} \leq T}\\
%    &\le \Pr{p_f(y|x) \le \frac{T}{T + 1}}\\
%    &\le \frac{1 - 2^{-l_f^{CE}}}{ 1 - \frac{T}{T + 1}} \\
%    &\le (T+1)(1 - 2^{-l_f^{CE}} )
%\end{align*}
% 
Similarly for $\modelt$,
\begin{equation}
    \small
    \begin{aligned}
        &\Pr{c_\modelt' \le {\epsilon}\sqrt{(1+\overline{\cS}(\el,\lt)) / 2}\|\nabla_x \ell_\modelt(x, y)\|_2} \\
        \le & \frac{\xi_\modelt}{\underset{x\in\cX}{\min}\,\ell_\modelt(x, y_t) - {\epsilon}B\left(1+\sqrt{(1+\overline{\cS}(\el,\lt)) / 2}\right) - \beta\epsilon^2}.
    \end{aligned}
    \label{eq:AppUpper06}
\end{equation}
We conclude the proof by combining the above two equations into Eq.~\ref{Beq4}.
% 
% Combining the two and inject them into \Cref{Beq4}, we get
% \begin{align*}
%     \Pr{T_r(\model, \modelt, x, y_t) = 1} \le \frac{\xi_\model + \xi_\modelt}{\ell_{\min} - \epsilon B (1+\sqrt{(1+\overline{\cS}(\el,\lt))/2}) - \beta\epsilon^2}.
% \end{align*}
\end{proof}

\begin{theorem*}[Upper Bound on Untargeted Attack Transferability]

    Assume both model $\cF$ and $\cG$ are $\beta$-smooth with gradient magnitude bounded by $B$, i.e., $\|\nabla_x \ell_{\cF}(x,y)\| \le B$ and $\|\nabla_x \ell_{\cG}(x,y)\| \le B$ for any $x \in \cX, y \in \cY$.
    Let $\Att_U$ be an $(\alpha,\cF)$-effective untargeted attack with perturbation ball $\|\delta\|_2 \le \epsilon$.
    When the attack radius $\epsilon$ is small such that 
            $\ell_{\min} - \epsilon B \left(1 + \sqrt{\frac{1+\overline{\cS}(\el,\lt)}{2}}\right) - \beta\epsilon^2 > 0$,
            the transferability can be upper bounded by
    $$
        % \resizebox{0.485\textwidth}{!}{
        % $
            \Pr{T_r(\cF, \cG, x) = 1} \le 
            \dfrac{\xi_\model + \xi_\modelt}{\ell_{\min} - \epsilon B \left(1 + \sqrt{\frac{1+\overline{\cS}(\el,\lt)}{2}}\right) - \beta\epsilon^2},
        % $
        % }
    $$
    where 
    $
        \displaystyle
        \ell_{\min} = \min_{\substack{x\in\cX, y'\in\cY:\\ (x,y) \in \supp(\cD), y' \neq y}} \, (\ell_\cF(x, y'), \ell_\cG(x, y')).
    $
    Here $\xi_\cF$ and $\xi_\cG$
    are the \emph{empirical risks} of models $\cF$ and $\cG$ respectively, defined relative to a differentiable loss.
    The $\supp(\cD)$ is the support of benign data distribution, i.e., $x$ is the benign data and $y$ is its associated true label.

%     Consider an instance $x \in \bbR^n$ with true label $y$ and adversarial target $y_t$.
%     Assume both model $\cF$ and $\cG$ are $\beta$-smooth with gradients bounded by $B$.
%     An $(\alpha,\cF)$-effective (untargeted) adversarial attack $\xA=\Att_U(x)$ with perturbation ball $\|\delta\|_2 \le \epsilon$ is transferable to $\cG$ with probability
%     \begin{small}
%         \begin{align}
%           &\Pr{T_r(\cF, \cG, x, y_t) = 1} \le 
%           \frac{\xi_\model + \xi_\modelt}{\ell_{\min} - \epsilon B \left(1 + \sqrt{\dfrac{1+\overline{\cS}(\el,\lt)}{2}}\right) - \beta\epsilon^2},
%             % \label{eq:upper-bound}
%         \end{align}
%     \end{small}
%     \vspace{-0.5em}
% 	\begin{small}
% 		\begin{align*}
%         	\text{where} \quad & \ell_{\min} = \underset{x, y': (x,y) \in \cM, y' \neq y}{\min} (\ell_\cF(x, y'), \ell_\cG(x, y')),
%         	\, \xi_\cF = \bbE_{x,y}\left[ \ell_\cF(x,y) \right],\, 
%         \xi_\cG = \bbE_{x,y}\left[ \ell_\cG(x,y) \right].
%     	\end{align*}
%     \end{small}
%     Here $\xi_\cF$ and $\xi_\cG$
%     are the \emph{empirical risks} of models $\cF$ and $\cG$ respectively, defined relative to a differentiable loss.
\end{theorem*}

\begin{proof}
        The proof follows the proof for the targeted attack case.
        Accordingly, Eq.~\ref{Beq05} and Eq.~\ref{Beq06} are modified to
        \begin{equation}
            % \small
            \begin{aligned}
                \delta\cdot\nabla_x \ell_\model(x, y) &\ge \ell_\model(\xA, y) - \ell_\model(x, y) - \frac{\beta}{2}\|\delta\|^2 \\
                &\ge \ell_\model(\xA, y_a) - \ell_\model(x, y) - \frac{\beta}{2}\|\delta\|^2 =: c_\model'.
            \end{aligned}
             \label{Beq05-2}
        \end{equation}
        \begin{equation}
            % \small
            \delta\cdot\nabla_x \ell_\modelt(x, y) \ge \ell_\modelt(\xA, y_b) - \ell_\modelt(x, y) - \frac{\beta}{2}\|\delta\|^2 =: c_\modelt' \label{Beq06-2}
        \end{equation}
        where $y_a$ and $y_b$ are the predicted labels of model $\model$ and $\modelt$ for $\xA$ under a transferable untargeted attack respectively.
        Both $y_a$ and $y_b$ are not equal to $y$.
        Then, instead of $\min_{x\in\cX}\ell_{\model / \modelt}(x, y_t)$ we use 
        $$
            \min_{x\in\cX, y'\in\cY: (x,y) \in \supp(\cD), y' \neq y} \ell_{\model / \modelt} (x, y')
        $$ 
        in Eq.~\ref{eq:AppUpper05} and Eq.~\ref{eq:AppUpper06} and henceforth. 
        % $\Box$
\end{proof}

\section{Additional Details for Baseline Ensembles and Whitebox Attacks}
\label{adx:sec-omitted-introduction}
In this section, we present a detailed introduction for our baseline ensembles and evaluated whitebox attacks. We introduce the baseline ensemble methods as follows:

        \begin{itemize}[leftmargin=*]
            \item \textbf{Boosting}~\cite{mason1999boosting,schapire1990strength} is a natural way of model ensemble training, which builds different weak learners in a sequential manner improving diversity in handling different task partitions. Here we consider two variants of boosting algorithms: 1) \textbf{AdaBoost}~\cite{hastie2009multi}, where the final prediction will be the weighted average of all the weak learners: weight $\alpha_i$ for $i$-th base model is decided by the accumulated error $e_i$ as $\alpha_i = \log \frac{1-e_i}{e_i} + \log(K-1)$. Here $K$ refers to the number of categories in a classification task. As we can see, higher weight will be placed on stronger learners. 
            2) \textbf{GradientBoost}~\cite{friedman2001greedy}, which is a general ensemble training method by identifying weaker learners based on gradient information and generating the ensemble by training base models step by step with diverse learning orientations within pseudo-residuals $r = -\frac{\partial \ell(f(x), y)}{\partial f(x)}$ computed from the current ensemble model $f$ on input $x$ with ground truth label $y$. 
            
            \item \textbf{CKAE}~\cite{kornblith2019similarity} develops diverse ensembles based on CKA measurement, which is recently shown to be effective to measure the orthogonality between representations. For two representations $K$ and $L$, $\text{CKA}(K, L) = \frac{\text{HSIC}(K, L)}{\sqrt{\text{HSIC}(K, K)\text{HSIC}(L, L)}}$, $\text{HSIC}(K, L) = \frac{1}{(n-1)^2}\text{tr}(KHLH)$, where $n$ is the number of samples and $H$ the centering matrix. For an ensemble consisting of base models $\{\cF_i\}$, we regard the representation of $\cF_i$ as its loss gradient vectors on batch samples and then minimize pair-wise CKA between base models $\cF_i, \cF_j$'s representations to encourage ensemble diversity.
            \item \textbf{ADP}~\cite{pang2019improving} is proposed recently as an effective regularization-based training method to reduce adversarial transferability among base models within an ensemble by maximizing the volume spanned by base models' non-maximal output vectors. Specifically, for a ensemble consisting of base models $\{\cF_i\}_{i=1}^N$ and input $x$ with ground truth label $y$, the ADP regularizer is defined as $\mathcal{L}_{\text{ADP}}(x,y)=\alpha \cdot H(\text{mean}(\{\cF_i(x)\}_{i=1}^N)) + \beta \cdot \log(\mathbb{ED})$, where $H(\cdot)$ is the Shannon Entropy Loss and $\mathbb{ED}$  the square of the spanned volume. Nevertheless, the ADP ensemble has been shown to be vulnerable against attacks that run for enough iterations until converged~\cite{tramer2020adaptive}. We will also discuss this similar observation in our empirical robustness evaluation. 
            \item \textbf{GAL}~\cite{kariyappa2019improving} promotes the diverse properties of the ensemble model by only minimizing the actual cosine similarities between pair-wise base models' loss gradient vectors. For $N$ base models $\{\cF_i\}_{i=1}^N$ within an ensemble and input $x$ with ground truth label $y$, the GAL regularizer is defined as: $\mathcal{L}_{\text{GAL}} = \log(\sum_{1 \leq i < j \leq N} \exp(CS(\nabla_x \ell_{\cF_i}, \nabla_x \ell_{\cF_j}))$ where $CS(\cdot, \cdot)$ refers to the actual cosine similarity measurement and $\nabla_x \ell_{\cF_i}$ the loss gradient of base model $\cF_i$ on $x$. It could serve as a baseline to empirically verify our theoretical analysis: when the loss gradients of base models are similar, the smoother the base models are, the less transferable they are.
            % Such intuition is also illustrated in Figure~\ref{fig:transfer}.
            \item \textbf{DVERGE}~\cite{yang2020dverge} reduces the transferability among base models by utilizing Cross-Adversarial-Training: For a ensemble consisting of base models $\{\cF_i\}$ and input $x$ with ground truth label $y$, each base model $\cF_i$ is trained with the non-robust feature instances~\cite{ilyas2019adversarial} generated against another base model. Specifically, DVERGE minimizes $\sum_{j\neq i}\ell(\cF_i(x'_{\cF_j}(x_s, x)), y_s)$ for every $\cF_i$ iteratively, where $x'_{\cF_j}(x_s, x)$ represents the non-robust features against $\cF_j$ based on the randomly chosen input ($x_s$, $y_s$). $\ell(\cdot, \cdot)$ is the cross-entropy loss function.
            \end{itemize}
            
We consider the following attacks for whitebox robustness evaluation. Here we define $(x, y)$ to be the input $x$ with label $y$ and $\xA$ to be the notion of adversarial example generated from $x$. $\ell(\cF(x), y)$ refers to the loss between model output $\cF(x)$ and label $y$, and $\epsilon$ is the $\ell_\infty$ perturbation magnitude bound for different attacks. %\zhuolin{note to talk about $\epsilon$ appropriately}
        
        \begin{itemize}[leftmargin=*]
            \item \emph{Fast Gradient Sign Method} (FGSM)~\cite{goodfellow2014explaining} is a simple yet effective attack strategy which generates adversarial example $\xA = x + \nu$ by assigning $\nu = \epsilon \cdot \mathrm{sgn}(\nabla_x \ell(\cF(x), y))$. 
            
            \item \emph{Basic Iterative Method} (BIM)~\cite{madry2017towards} is an iterative attack method which adds adversarial perturbations step by step: $x_{i+1} = \mathrm{clip}(x_i + \alpha \cdot \nabla_{x_i} \ell(\cF(x_i), y))$, with initial starting point $x_0=x$. Function $\mathrm{clip}(\cdot)$ projects the perturbed instance back to the $\ell_\infty$ ball within the perturbation range $\epsilon$, and $\alpha$ refers to the step size. 
            
            \item \emph{Momentum Iterative Method} (MIM)~\cite{dong2018boosting} can be regarded as the variant of BIM by utilizing the gradient momentum during the iterative attack procedure. Within iteration $i+1$, we update new gradient as $g_{i+1} = \mu g_i + \frac{\ell(\cF(x_i), y)}{\|\nabla_x \ell(\cF(x_i), y)\|_1}$ and set $x_{i+1} = \mathrm{clip}(x_i + \alpha \cdot g_{i+1})$ while $\mu$ refers to the momentum coefficient and $\alpha$ the step size.
            
            \item \emph{Projected Gradient Descent} (PGD)~\cite{madry2017towards} can be regarded as the variant of BIM by sampling $x_0$ randomly within the $\ell_p$ ball around $x$ within radius $\epsilon$. After initialization, it follows the standard BIM procedure by setting $x_{i+1} = \mathrm{clip}(x_i + \alpha \cdot \nabla_{x_i} \ell(\cF(x_i), y))$ on $i$-th attack iteration.%and moving along the gradient direction until convergence.
            % It will project the instance back to the $\ell_p$ ball. 
            
            % \item \emph{Jacobian-based Saliency Map Attack} (JSMA)~\cite{papernot2016limitations} is a greedy attack algorithm that perturbs the pixels with high values on the saliency map at each iteration. The procedure will be terminated while the attack becomes successful or the modified pixels exceed the tolerable threshold.
            
            \item \emph{Auto-PGD} (APGD)~\cite{croce2020reliable} is a step-size free variant of PGD by configuring the step-size according to the overall iteration budgets and the progress of the current attack. Here we consider APGD-CE and APGD-DLR attack which use CrossEntropy (CE) and Difference of Logits Ratio (DLR)~\cite{croce2020reliable} loss as their loss function correspondingly.
            
            \item \emph{Carlini \& Wanger Attack} (CW)~\cite{carlini2017towards} accomplishes the attack by solving the optimization problem: $\xA := \min_{x'} \|x'-x\|_{2}^{2} + c\cdot f(x', y)$, where $c$ is a constant to balance the perturbation scale and attack success rate, and $f$ is the adversarial attack loss designed to satisfy the sufficient and necessary condition of different attacks. For instance, the untargeted attack loss is represented as $f(x', y) = \max(\cF(x)_y - \cF(x)_{i \neq y} , -\kappa)$ while $\kappa$ is a confidence variable with value $0.1$ as default.%\zhuolin{I checked the actual implementation of cw attack within advertorch and it still use the plus operator but reverse the $f(x', y)$ internally. Here I plan to keep plus operator but emphasis the untargeted attack goal design for $f$}
            
            \item \emph{Elastic-net Attack} (EAD)~\cite{chen2018ead} follows the similar optimization of CW Attack while considering both $\ell_2$ and $\ell_1$ distortion: $\xA := \min_{x'} \|x'-x\|_{2}^{2} + \beta\|x'-x\|_{1} + c\cdot f(x', y)$. Here $\beta, c$ refer to the balancing parameters and $f(x', y) = \max(\cF(x)_y - \cF(x)_{i \neq y} , -\kappa)$ under untargeted attack setting. We set $\beta=0.01, \kappa=0.1$ as default.%\zhuolin{note to clarify $f$ here as well.} 
            
        \end{itemize}
        
            In our experiments, we set $50$ attack iterations with step size $\alpha=\nicefrac{\epsilon}{5}$ for BIM, MIM attack and PGD attack with $5$ random starts. For CW and EAD attacks, we set the number of attack iterations as $1000$ and evaluate them with different constant $c$ for different datasets. 
\section{Training Details}
\label{adx:sec-experiment-details}
We adapt ResNet-20~\cite{he2016deep} as the base model architecture and Adam optimizer~\cite{kingma2014adam} in all of our experiments.

\noindent\textbf{TRS training algorithm.} We show the one-epoch TRS training algorithm pseudo code in \Cref{algo:trs-training}. We apply the mini-batch training strategy and train the TRS ensemble for $M$ epochs ($M=120$ for MNIST and $M=200$ for CIFAR-10) in our experiments. To decide the $\delta$ within the local min-max procedure, we use the \textbf{Warm-up} strategy by linearly increasing the local $\ell_\infty$ ball's radius $\delta$ from small initial $\delta_0$ to the final $\delta_M$ along with the increasing of training epochs.

    \begin{algorithm}
    \begin{algorithmic}[1]
    \STATE $\delta_m \leftarrow \delta_0 + (\delta_M - \delta_0) \cdot m / M$
    \FOR{$b=1,\cdots,B$}
    \STATE $(x, y) \leftarrow$ training instances from $b$-th mini-batch
    \STATE $\mathcal{L}_{\text{Reg}} \leftarrow 0$
    \STATE $\mathcal{L}_{\text{ECE}} \leftarrow 0$
    \FOR{$i=1,\cdots,N$}
    \FOR{$j=i+1,\cdots,N$}
    %\STATE $g_i \leftarrow \nabla_{x}\ell_{\mathcal{F}_i}(x, y)$
    %\STATE $g_j \leftarrow \nabla_{x}\ell_{\mathcal{F}_j}(x, y)$
    %\STATE $\mathcal{L}_{\text{Reg}} \leftarrow \mathcal{L}_{\text{Reg}} + \mathcal{L}_{\text{TRS}}(\mathcal{F}_i, \mathcal{F}_j, x)$ \# vanilla TRS training
    %\STATE $\hat{x} \leftarrow \text{PGD}(x, \delta_m, \mathcal{F}_i, \mathcal{F}_j)$
    \STATE $\mathcal{L}_{\text{Reg}} \leftarrow \mathcal{L}_{\text{Reg}} + \mathcal{L}_{\text{TRS}}(\mathcal{F}_i, \mathcal{F}_j, x, \delta_m)$ %\# GST-based TRS training
    \ENDFOR
    \ENDFOR
    \FOR{$i=1,\cdots,N$}
    \STATE $\mathcal{L}_{\text{ECE}} \leftarrow \mathcal{L}_{\text{ECE}} + \mathcal{L}_{\text{CE}}(\mathcal{F}_i(x),y)$
    \ENDFOR
    \STATE $\mathcal{L}_{\text{Reg}} \leftarrow \mathcal{L}_{\text{Reg}} / \binom{N}{2}$
    \STATE $\mathcal{L}_{\text{ECE}} \leftarrow \mathcal{L}_{\text{ECE}} / N$
    \FOR{$i=1,\cdots,N$}
    \STATE $\nabla_{\mathcal{F}_i} \leftarrow \nabla_{\mathcal{F}_i}[\mathcal{L}_{\text{ECE}}+ \mathcal{L}_{\text{Reg}}]$
    \STATE $\mathcal{F}_i \leftarrow \mathcal{F}_i - lr\cdot \nabla_{\mathcal{F}_i}$
    \ENDFOR
    \ENDFOR
    \caption{TRS training framework in epoch $m$ for an ensemble with $N$ base models $\{\mathcal{F}_i\}$, %PGD($x, \delta_m, \mathcal{F}_i, \mathcal{F}_j$) refers to the Projection Gradient Descent function to find the $\hat{x}$ with maximum $\|\nabla_{\hat{x}} \ell_{\cF_i}\|_2 + \|\nabla_{\hat{x}} \ell_{\cF_j}\|_2$ within the $\ell_\infty$ ball around instance $x$ of radius $\delta_m$ and 
    with the total number of training epochs $M$.} %and $\delta_0, \delta_M$ to the initial and final $\ell_\infty$ ball's radius of local min-max optimization procedure.}
    \label{algo:trs-training}
                                            % \vspace{-2em}
    \end{algorithmic}
    % \vspace{-0.1em}

    \end{algorithm}

\noindent\textbf{Baseline training details.} For ADP and GAL, we follow the exact training configuration mentioned in their paper in both MNIST and CIFAR-10 experiments. For DVERGE, we set the same feature distillation $\epsilon=0.07$ with step size as $0.007$ for CIFAR-10 as they mentioned in their paper but set $\epsilon=0.5$ with step size as $0.05$ for MNIST since they did not conduct any MNIST experiments in their paper. We set training epochs as 120 for MNIST and 200 for CIFAR-10 and CIFAR-100 in baseline training.

\noindent\textbf{TRS training details.} For MNIST, we set the initial learning rate $\alpha=0.001$ and train our TRS ensemble for 120 epochs by decaying the learning rate by $0.1$ at $40$-th and $80$-th epochs. For CIFAR-10 and CIFAR-100 we set the initial learning rate $\alpha=0.001$ and train our TRS ensemble for 200 epochs by decaying the learning rate by $0.1$ at $100$-th and $150$-th epochs. For PGD Optimization within $\mathcal{L}_{\text{smooth}}$ approximation, we set step size $\tilde{\alpha} = \delta/3$ and the total number of steps $T$ as $6$ for both MNIST and CIFAR-10 experiments. We also leverage the ablation study about the convergence of PGD optimization w.r.t the robustness of TRS ensemble by varying $\tilde{\alpha}$ and $T$ in \ref{adx:sec-pgd-inner-convergence}.

By configuring the default TRS ensemble training setting, we evaluated the average epoch training time for TRS and compared it to other baselines (ADP, GAL, DVERGE) on RTX 2080 single GPU device. Results are shown in \Cref{tab:average-time}.

\begin{table}[!htbp]
\centering
\caption{Comparison on average epoch training time (s) between TRS training and other baseline training methods, evaluated on RTX 2080 single GPU device.}

\begin{tabular}{c|c|c|c|c}
\toprule
Avg epoch training time (s) & ADP   & GAL    & DVERGE & TRS     \\ \hline
MNIST                       & 29.22 & 106.81 & 184.42 & 302.24  \\
CIFAR-10                    & 33.22 & 139.10 & 349.61 & 1291.55 \\ \bottomrule
\end{tabular}
\label{tab:average-time}
\end{table}

Our results show that though ADP, GAL require less training time, they can not achieve even comparable robustness with TRS as shown in our paper. Compared with DVERGE, TRS requires longer training time but maintains higher robustness under almost all attack scenarios.
\section{Numerical Results of Blackbox Robustness Evaluation}
\label{adx:sec-blackbox-results}
Table~\ref{tab:adx-blackbox-mnist} and~\ref{tab:adx-blackbox-cifar10} show the detailed robust accuracy number of different ensembles against blackbox transfer attack with different perturbation scale $\epsilon$, which corresponds to the \Cref{fig:blackbox-curve}. As we can see, TRS ensemble shows its competitive robustness to DVERGE on small $\epsilon$ setting but much better stability of robustness on large $\epsilon$ setting although it slightly sacrifices benign accuracy on clean data.
    \begin{table}[htbp]
        \centering
            \caption{Robust accuracy (\%) of different approaches against \textbf{blackbox transfer attack} with different perturbation scales $\epsilon$ on MNIST dataset.}
            \scalebox{0.92}{
        \begin{tabular}{c|c|c|c|c|c|c|c|c}
        \toprule
        $\epsilon$ & clean & 0.10 & 0.15 & 0.20 & 0.25 & 0.30 & 0.35 & 0.40 \\ \hline
        Vanilla   & 99.5  & 1.8  & 0.1  & 0.0  & 0.0  & 0.0  & 0.0  & 0.0  \\
        ADP        & 99.4  & 25.5 & 13.8 & 7.0  & 2.1  & 0.3  & 0.1  & 0.0  \\
        GAL        & 98.7  & 96.8 & 77.0 & 29.1 & 12.8 & 4.6  & 1.9  & 0.6  \\
        DVERGE     & 98.7  & \textbf{97.6} & \textbf{97.4} & \textbf{96.9} & 96.2 & 94.2 & 78.3 & 20.2 \\
        %COS-ONLY     & 99.1  & 56.8 & 36.9 & 23.2 & 12.9 & 6.7 & 4.4 & 2.1 \\
        \rowcolor{tabgray} TRS        & 98.6  & 97.2 & 96.7 & 96.5 & \textbf{96.3} & \textbf{95.5} & \textbf{93.1} & \textbf{86.4} \\  \bottomrule
        \end{tabular}}
        \label{tab:adx-blackbox-mnist}
        % \end{table}
        \vspace{1em}

        % \begin{table}[tbp]
        \centering
        \caption{Robust accuracy (\%) of different approaches against \textbf{blackbox transfer attack} with different perturbation scales $\epsilon$ on CIFAR-10 dataset.}%\zhuolin{we need a baseline here maybe}}
        \scalebox{0.92}{
        \begin{tabular}{c|c|c|c|c|c|c|c|c}
        \toprule
        $\epsilon$ & clean & 0.01 & 0.02 & 0.03 & 0.04 & 0.05 & 0.06 & 0.07 \\ \hline
        Vanilla & 94.1 & 10.0 & 0.1 & 0.0 & 0.0 & 0.0 & 0.0 & 0.0 \\
        ADP        & 91.6  & 20.7 & 0.5 & 0.0 & 0.0 & 0.0 & 0.0 & 0.0 \\
        GAL        & 88.3 & 74.6 & 58.9 & 39.1 & 22.0 & 11.3 & 5.2 & 2.1 \\
        DVERGE     & 91.9  & \bf 83.3 & 69.0 & 49.8 & 28.2 & 14.4 & 4.0 & 0.8 \\
        %COS-ONLY   & 92.8 & 36.0 & 5.2 & 0.5 & 0.0 & 0.0 & 0.0 & 0.0 \\

        \rowcolor{tabgray} TRS & 86.7 & 82.3 & \bf 76.1 & \bf 65.8 & \bf 55.0 & \bf 45.5 & \bf 35.8 & \bf 26.7 \\
         \bottomrule
        \end{tabular}}
        \label{tab:adx-blackbox-cifar10}
        \end{table}
        
\begin{table}[!t]
\centering
\caption{Robust accuracy (\%) of TRS ensemble trained with different hyper-parameter settings against various whitebox attacks on MNIST dataset.}
\scalebox{0.82}{
\begin{tabular}{c|c|c|c|c|c|c|c|c|c}
\toprule
\multicolumn{2}{c|}{$\lambda_a$}            & \multicolumn{4}{c|}{100}                           & \multicolumn{4}{c}{500}                           \\ \hline
\multicolumn{2}{c|}{$\lambda_b$}            & \multicolumn{2}{c|}{2.5} & \multicolumn{2}{c|}{10} & \multicolumn{2}{c|}{2.5} & \multicolumn{2}{c}{10} \\ \hline
\multicolumn{2}{c|}{$\delta_M$}             & 0.3             & 0.4    & 0.3    & 0.4            & 0.3             & 0.4    & 0.3        & 0.4        \\ \hline
\rowcolor{tabgray}     & $\epsilon=0.1$  & \textbf{95.6}   & 94.6   & 90.6   & 94.8          & \textbf{95.6}   & 93.0   & 95.2       & 93.6       \\
                           \rowcolor{tabgray}\multirow{-2}{*}{FGSM}& $\epsilon=0.2$  & 91.7            & 83.4   & 89.7   & 87.3           & \textbf{92.0}   & 84.0   & 88.0       & 85.0       \\ \hline
\multirow{2}{*}{BIM (50)} & $\epsilon=0.1$  & \textbf{93.3}   & 82.9   & 75.7   & 92.5           & 88.2            & 83.9   & 92.6       & 90.9       \\
                          & $\epsilon=0.15$ & \textbf{85.7}   & 69.7   & 61.3   & 84.1           & 73.1            & 61.3   & 82.2       & 83.3       \\ \hline
\rowcolor{tabgray}& $\epsilon=0.1$  & \textbf{93.0}   & 79.1   & 74.3   & 92.2           & 86.3            & 83.3   & 91.7       & 90.6       \\
                          \rowcolor{tabgray}\multirow{-2}{*}{PGD (50)} & $\epsilon=0.15$ & \textbf{85.1}   & 62.6   & 57.4   & 82.6           & 69.9            & 58.2   & 80.0       & 82.9       \\ \hline
\multirow{2}{*}{MIM (50)} & $\epsilon=0.1$  & \textbf{92.9}   & 81.6   & 75.1   & 92.0           & 87.7            & 83.5   & 91.7       & 91.2       \\
                           & $\epsilon=0.15$ & \textbf{85.1}   & 68.2   & 60.2   & 83.7           & 74.0            & 62.4   & 82.4       & 83.4       \\ \hline
\rowcolor{tabgray}    & $c=0.1$         & 98.1            & 96.6   & 96.4   & 97.5           & \textbf{98.4}   & 97.2   & 98.1       & 97.8       \\
                          \rowcolor{tabgray}\multirow{-2}{*}{CW}   & $c=1.0$         & 92.6            & 92.6   & 89.1   & \textbf{95.9}  & 86.1            & 77.4   & 88.2       & 95.1       \\ \hline
\multirow{2}{*}{EAD}& $c=1.0$         & 23.3            & 14.3   & 9.2    & \textbf{24.1}  & 22.5            & 2.6    & 3.4        & 23.9       \\
                          & $c=5.0$         & 1.4             & 0.9    & 0.1    & \textbf{2.3}   & 0.0             & 0.0    & 0.2        & 1.7        \\ \hline
\rowcolor{tabgray} & $\epsilon=0.1$  & \textbf{92.1}   & 78.5   & 72.8   & 91.5           & 85.9            & 82.8   & 91.1       & 90.2       \\
                          \rowcolor{tabgray}\multirow{-2}{*}{APGD-DLR}& $\epsilon=0.15$ & \textbf{83.4}   & 62.1   & 57.0   & 82.3           & 69.6            & 57.9   & 79.8       & 82.4       \\ \hline
\multirow{2}{*}{APGD-CE}& $\epsilon=0.1$  & \textbf{91.7}   & 78.1   & 72.1   & 91.2           & 85.2            & 82.5   & 90.8       & 89.7       \\
                            & $\epsilon=0.15$ & \textbf{82.8}   & 61.3   & 56.5   & 81.9           & 69.3            & 57.6   & 79.4       & 81.7       \\ \bottomrule
\end{tabular}}
\label{tab:big-table-mnist}
\end{table}

% \usepackage{multirow}

\begin{table}[!htbp]
\centering
\caption{\{Min, Max, Mean, Std\} of Robust accuracy (\%) of TRS ensemble against $10$ times whitebox attacks simulation with different random seeds on MNIST and CIFAR-10 datasets.}
\begin{tabular}{c|c|c|c|c|c|c}
\toprule
\multicolumn{2}{c|}{Robust Accuracy}                  & param.          & Min & Max & Mean & Std \\ \hline
\multirow{6}{*}{MNIST}    & \multirow{2}{*}{PGD}      & $\epsilon=0.1$  &  92.8   & 93.2    &  93.1    &  0.143    \\
                          &                           & $\epsilon=0.15$ &  84.9   & 85.1 &  85.1    &  0.067   \\ \cline{2-7} 
                          & \multirow{2}{*}{APGD-DLR} & $\epsilon=0.1$  &  92.1   &  92.3   &  92.2    &  0.083   \\
                          &                           & $\epsilon=0.15$ &  83.2   & 83.5    &  83.4    &   0.114  \\ \cline{2-7} 
                          & \multirow{2}{*}{APGD-CE}  & $\epsilon=0.1$  &  91.7   &   92.0  &  91.9    &  0.102   \\
                          &                           & $\epsilon=0.15$ &  82.5   & 82.9    &  82.7    & 0.120 \\ \hline\hline
\multirow{6}{*}{CIFAR-10} & \multirow{2}{*}{PGD}      & $\epsilon=0.01$ & 50.4 & 50.5 & 50.4 & 0.049 \\
                          &                           & $\epsilon=0.02$ & 14.8 & 15.8 & 15.2 & 0.293 \\ \cline{2-7} 
                          & \multirow{2}{*}{APGD-DLR} & $\epsilon=0.01$ & 50.0 & 50.5 & 50.2 & 0.151    \\
                          &                           & $\epsilon=0.02$ & 15.2 & 16.0 & 15.6 & 0.234 \\ \cline{2-7} 
                          & \multirow{2}{*}{APGD-CE}  & $\epsilon=0.01$ & 48.6 & 48.9 & 48.8 & 0.090 \\ 
                          &                           & $\epsilon=0.02$ & 15.3 & 16.0 & 15.6 & 0.199 \\ \bottomrule
\end{tabular}
\label{tab:stabiltiy-robustness}
\end{table}

\section{Statistical Stability Analysis on Robust Accuracy}
\label{adx:sec-stability-results}

For attacks with random-start (PGD, APGD-DLR, APGD-CE) mentioned in \Cref{tab:result1}, we run each of them $10$ times with different random seeds and evaluate them on TRS ensemble to present the statistical indicators (Min, Max, Mean, Std) of robust accuracy in \Cref{tab:stabiltiy-robustness}. We can conclude that our reported robust accuracy shows statistical stability given the standard deviation is smaller than 0.3 under all the scenarios.

\section{Ablation Studies}
        \label{adx:sec-ablation-study}
    % \noindent\textbf{Running time Analysis}
    % We evaluate the epoch-average running time for TRS and other baseline methods and results are shown in Table~\ref{tab:x}. We can see that TRS training could achieve similar running time comparing to GAL ensemble method but reached much better robustness and transferability results. Our method also shown its efficiency compared to the current state-of-the-art DVERGE method. 
    %In this section we visualize the decision boundary of the trained ensembles to gain more insights about the model properties, together with the exploration on the impacts of different loss terms in TRS.
    
    \subsection{Decision Boundary Analysis}
    \label{adx:sec-decision-boundary}
    %   Figure~\ref{fig:boundary-cifar} and~\ref{fig:trs-boundary-mnist} shows TRS model's decision boundary around several testing images. Unlike DVERGE's decision boundary plots~\ref{fig:dverge-boundary-mnist}, each instance could be close to the decision boundary following the negative gradient direction in TRS model. That is, TRS model suffers a lot from the single step adversarial attack. While against multi-steps adversarial attacks, TRS could achieve high robust accuracy with the quite smoothing decision boundary.
      We visualize the decision boundary of the GAL, DVERGE and TRS ensembles for MNIST and CIFAR-10  in Figure~\ref{fig:boundary}. The dashed line is the negative gradient direction and the horizontal direction is randomly chosen which is orthogonal to the gradient direction. From the decision boundary of GAL ensemble, we can see that controlling only the gradient similarity will lead to a very non-smooth model decision boundary and thus harm the model robustness.
    %   As the ablation study, Figure~\ref{fig:gal-boundary-mnist} shows the decision boundary of GAL models. We should notice that by applying the gradient similarity regularization term only will lead to rocky model decision boundary and harm model's robustness.
      From the comparison of DVERGE and TRS ensemble, we find that DVERGE ensemble tends to be more robust along the gradient direction especially on CIFAR-10, (i.e. the distance to the boundary is larger and sometimes even larger than along the other random direction). 
      This may be due to the reason that DVERGE is essentially performing adversarial training for different base models and therefore it protects the adversarial (gradient) direction.
      Thus, DVERGE performs better against weak attacks which only consider the gradient direction (e.g. FGSM on CIFAR-10).
      On the other hand, we find that TRS training yields a smoother model along different directions
      than DVERGE, which leads to more consistent predictions within a large neighborhood of an input.
    %   In addition, from TRS's decision boundary, we find that the distance to the boundary towards the random direction is usually larger, which 
     Thus, the TRS ensemble has higher robustness in different directions against strong attacks such as PGD attack.
    %   Along other directions we observe that our model has a better robustness, since the distance to the boundary toward the random direction is usually large.
    %   In addition, we can also see that TRS yields a smoother model which has more consistent predictions along gradient direction.
    %   Therefore, our approach performs better against the attacker when he aims to apply stronger attacks to find adversarial examples within small radius, such as PGD attack.
      
            \begin{figure}[!htbp]
            \centering
            \includegraphics[width=0.22\textwidth]{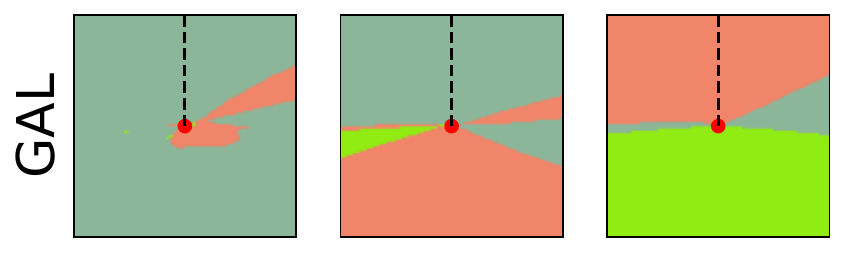}
            \includegraphics[width=0.22\textwidth]{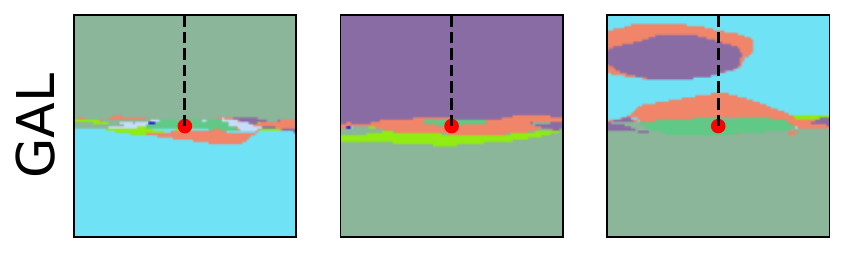}
            
            \includegraphics[width=0.22\textwidth]{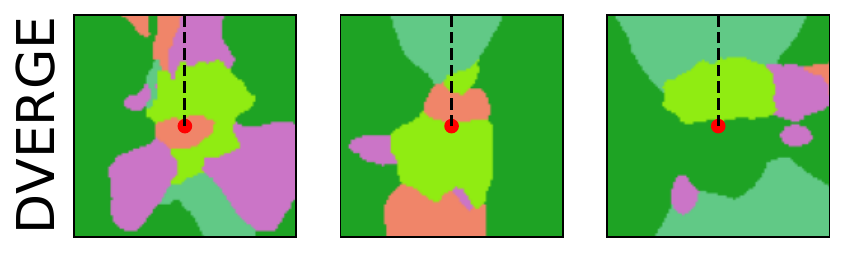}
            \includegraphics[width=0.22\textwidth]{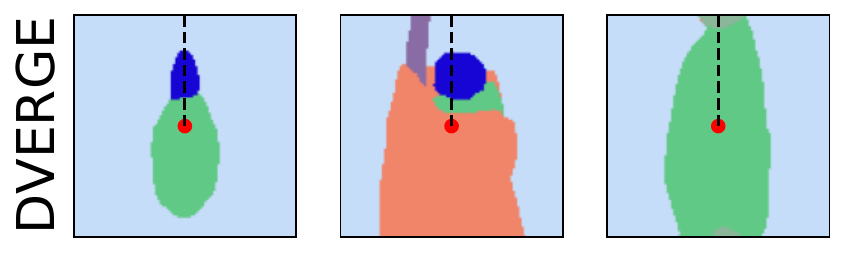}
            
            \includegraphics[width=0.22\textwidth]{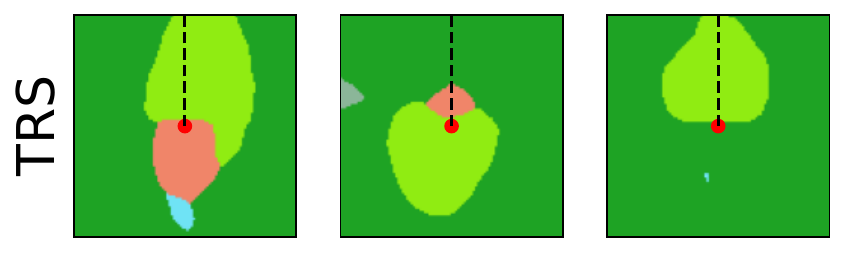}
            \includegraphics[width=0.22\textwidth]{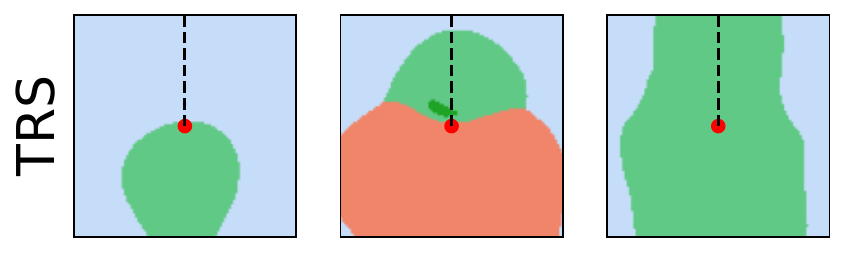}
            
            \caption{\small The decision boundary of different models around testing images on (left) MNIST and (right) CIFAR-10 dataset. Same color indicates the same model prediction. The dash lines shows the negative gradient direction, which is used in the gradient-based attacks.}
            \label{fig:boundary}
        \end{figure}

    \subsection{TRS with Adversarial Training}
        \label{adx:sec-adv-training}
           While the TRS regularizer could reduce adversarial transferability among base models by enforcing low similarity on loss gradients and promoting model smoothness, we explore whether Adversarial Training~\cite{madry2017towards}, which aims to reduce base models' vulnerability, is able to further improve the robustness of TRS or not. 
           We first apply adversarial training to train an ensemble model (AdvT) containing 3 base models as in TRS ensemble. During training, we use
           $\ell_\infty$ adversarial perturbation $\delta_{\text{adv}}$ ($\|\delta_{\text{adv}}\|_\infty\leq 0.2$ for MNIST and $\|\delta_{\text{adv}}\|_\infty\leq 0.03$ for CIFAR-10).
           To combine TRS with adversarial training (TRS+AdvT), we combine TRS regularizer Loss $\mathcal{L}_{\text{TRS}}$ with Adversarial Training Loss $\mathcal{L}_{\text{AdvT}} = \max_{\|x-x'\|_\infty\leq \delta_{\text{adv}}} \ell_\cF(x', y)$ on input $x$ with label $y$ with the same weight, and train the ensemble $\mathcal{F}$ jointly. 
    We evaluate both \underline{whitebox} and \underline{blackbox} robustness of TRS+AdvT and AdvT ensembles. During the evaluation, we consider the \emph{Conditional Robust Accuracy} evaluated on adversarial examples generated based on {correctly classified clean samples} to eliminate the influence of model benign accuracy. Other settings are the same as we have introduced in Section 4.1. 
    
    Table~\ref{tab:whitebox-advt} shows the robustness of both AdvT and TRS+AdvT  under whitebox and blackbox attacks on different datasets. As we can see, TRS+AdvT ensemble outperforms the traditional adversarial training based ensemble consistently especially when $\epsilon$ is large.
    % , which indicates the significance of reducing inter model adversarial transferability besides model self-robustness. %\zhuolin{one things left: 1. should we include TRS into this table as well?}
    
    \newcolumntype{g}{>{\columncolor{tabgray}}c}
    \begin{table}[tbp]
    \centering
    \vspace{-1em}
    \caption{ Conditional Robust Accuracy (\%) of Adversarial Training based ensemble (AdvT) and TRS+AdvT ensemble against (Top) \textbf{whitebox attacks} and (Down) \textbf{blackbox attack} with different perturbation scales $\epsilon$.}
    \scalebox{0.78}{
    \begin{tabular}{c|g|g|g|g|g|g|g|g|g}
    \toprule
    \rowcolor{white} \multicolumn{2}{c|}{Attacks}          & \multicolumn{2}{c|}{FGSM} & \multicolumn{2}{c|}{BIM (50)} & \multicolumn{2}{c|}{PGD (50)} & \multicolumn{2}{c}{MIM (50)} \\ \hline
    \rowcolor{white}    & $\epsilon$ & 0.10        & 0.20        & 0.10          & 0.15          & 0.10          & 0.15          & 0.10          & 0.15          \\ \cline{2-10} 
    \rowcolor{white}                           & AdvT       & 98.4        & 97.3        & 98.2          & 97.5          & 98.2          & 97.2          & 98.2          & 97.6          \\
                         \multirow{-3}{*}{MNIST}     &  TRS+AdvT   & \textbf{99.1}        & \textbf{98.0}        & \textbf{99.0}          & \textbf{98.2}          & \textbf{98.9}          & \textbf{98.0}          & \textbf{99.0}          & \textbf{98.1}          \\ \hline \hline
    \rowcolor{white}  & $\epsilon$ & 0.02        & 0.04        & 0.01          & 0.02          & 0.01          & 0.02          & 0.01          & 0.02          \\ \cline{2-10} 
    \rowcolor{white}                           & AdvT       & \textbf{79.3}        & \textbf{60.0}        & 88.5          & 76.1          & 88.4          & 76.1          & 88.5          & 76.3          \\
                     \multirow{-3}{*}{CIFAR-10}         & TRS+AdvT   & 79.2        & 58.0        & \textbf{90.7}          & \textbf{76.7}          & \textbf{90.7}          & \textbf{76.6}          & \textbf{90.9}          & \textbf{76.9}          \\ \bottomrule
    \end{tabular}
    }
        \vspace{1.2em}
    \label{tab:whitebox-advt}
    % \end{table}

 %   \begin{table}[tbp]
    \centering
  %  \caption{Conditional Robust Accuracy (\%) under \textbf{blackbox transfer attack} with different perturbation scales $\epsilon$ of Adversarial Training (AdvT) based ensemble and TRS+AdvT ensemble.}
    \scalebox{0.85}{
    \begin{tabular}{c|g|g|g|g|g|g|g|g}
    \toprule
    \rowcolor{white}  \multirow{3}{*}{MNIST}    & $\epsilon$ & 0.10                      & 0.15                      & 0.20                      & 0.25                      & 0.30                      & 0.35                      & 0.40                     \\ \cline{2-9} 
    \rowcolor{white}                          & AdvT       & 98.9                      & 98.7                      & 98.6                      & 98.4                      & 98.4                      & 91.6                      & 8.1                      \\
    \multirow{-3}{*}{MNIST}                          &  TRS+AdvT   & \textbf{99.4}                      & \textbf{99.3}                      & \textbf{99.1}                      & \textbf{99.1}                      & \textbf{98.9}                      & \textbf{98.7}                      & \textbf{98.5}                     \\ \hline  \hline
    \rowcolor{white}  & $\epsilon$ & 0.01 & 0.02 & 0.03 & 0.04 & 0.05 & 0.06 & 0.07 \\ \cline{2-9} 
    \rowcolor{white}                          & AdvT       & 98.4 & 96.2 & 93.9 & 91.5 & 89.0 & 84.9 & 81.1 \\
     \multirow{-3}{*}{CIFAR-10}                         & TRS+AdvT   & \textbf{98.8} & \textbf{97.7} & \textbf{94.9} & \textbf{92.6} & \textbf{89.9} & \textbf{86.3} & \textbf{81.6} \\ \bottomrule
    \end{tabular}}
    \label{tab:blackbox-advt}
    \end{table}
    
    \subsection{Impacts of \texorpdfstring{$\mathcal{L}_{\text{sim}}$ and $\mathcal{L}_{\text{smooth}}$}{Lsim and Lsmooth}}
    \label{adx:separate-effects-section}
    To better understand the exact effects of regularizing $\mathcal{L}_{\text{sim}}$ and $\mathcal{L}_{\text{smooth}}$, we conduct ablation studies by regularizing $\mathcal{L}_{\text{sim}}$ or $\mathcal{L}_{\text{smooth}}$ only on both MNIST and CIFAR-10 datasets. Results are shown in \Cref{tab:separate-effects}.
    
    We can see that, though training with $\mathcal{L}_{\text{smooth}}$ only could lead to high robustness, TRS ensemble could achieve even higher robustness against strong multi-step attacks by concerning similarity loss $\mathcal{L}_{\text{sim}}$ at the same time. This indicates that both model smoothness and model diversity are important, though $\mathcal{L}_{\text{smooth}}$ would take the majority. 
    % By combining $\mathcal{L}_{\text{sim}}$ and $\mathcal{L}_{\text{smooth}}$ for training, we can achieve the highest robustness.
    
% \usepackage{multirow}
\begin{table}[!htbp]
\centering
\caption{Robust accuracy (\%) of TRS ensemble trained by regularizing $\mathcal{L}_{\text{sim}}$ or $\mathcal{L}_{\text{smooth}}$ only, or together, against various white-box attacks on MNIST and CIFAR-10 datasets.}
\scalebox{0.86}{
\begin{tabular}{c|c|c|c|c|c|c|c}
\toprule
\multicolumn{2}{c|}{Robust Accuracy}                                                       & FGSM            & BIM             & PGD             & MIM             & CW            & EAD          \\ \hline
\multirow{4}{*}{MNIST}    & param.                                                         & $\epsilon=0.2$  & $\epsilon=0.15$ & $\epsilon=0.15$ & $\epsilon=0.15$ & $c=1.0$       & $c=10.0$     \\ \cline{2-8} 
                          & $\mathcal{L}_{\text{sim}}$ only                                & 30.7            & 0.0             & 0.0             & 0.0             & 58.6          & 0.5          \\
                          & $\mathcal{L}_{\text{smooth}}$ only                             & \textbf{93.1}   & 82.5            & 80.7            & 82.6            & 86.2          & 1.2          \\
                          & $\mathcal{L}_{\text{sim}} + \mathcal{L}_{\text{smooth}}$ (TRS) & 91.7            & \textbf{85.7}   & \textbf{85.1}   & \textbf{85.1}   & \textbf{92.6} & \textbf{1.4} \\ \hline\hline
\multirow{4}{*}{CIFAR-10} & param.                                                         & $\epsilon=0.04$ & $\epsilon=0.02$ & $\epsilon=0.02$ & $\epsilon=0.02$ & $c=1.0$       & $c=5.0$      \\ \cline{2-8} 
                          & $\mathcal{L}_{\text{sim}}$ only                                & \textbf{35.0}   & 0.0             & 0.0             & 0.0             & 17.6          & 0.0          \\
                          & $\mathcal{L}_{\text{smooth}}$ only                             & 9.3             & 13.9            & 13.8            & 15.0            & 43.0          & 0.0          \\
                          & $\mathcal{L}_{\text{sim}} + \mathcal{L}_{\text{smooth}}$ (TRS) & 24.9            & \textbf{15.8}   & \textbf{15.1}   & \textbf{17.2}   & \textbf{58.1} & \textbf{0.1} \\ \bottomrule
\end{tabular}}
\label{tab:separate-effects}
\end{table}
    
    \subsection{Robust Accuracy Convergence Analysis}
    \label{adx:sec-convergence}
    
            We observe that when the number of attack iterations is large, both ADP and GAL regularizer trained ensembles achieve much lower robust accuracy against iterative attacks (BIM, PGD, MIM) than the reported robustness in the original papers which is estimated under a small number of attack iterations. This case implies the non-convergence of iterative attack evaluation mentioned in their papers, which is also confirmed by~\cite{tramer2020adaptive}. In contrast, both DVERGE and TRS still remain highly robust against iterative attacks with large iterations. To show the stability of our model's robust accuracy, we evaluate it against PGD attack with $500$ and $1000$ attack iterations. Results are shown in Table~\ref{tab:convergence} where TRS ensemble's robust accuracy only slightly drops after increasing the attack iterations, and outperforms DVERGE by a large margin.
            
          \begin{table}[!t]
        \centering
                                              
        \caption{ Convergence of PGD attack on different ensembles.}
        \scalebox{0.93}{
        \begin{tabular}{c|c|c|c|c|c|c}
        \toprule
        \multicolumn{2}{c|}{Settings}                                & iters & ADP & GAL & DVERGE & TRS  \\ \hline
        \multirow{6}{*}{MNIST}    & \multirow{3}{*}{$\epsilon=0.10$} & 50    & 4.5 & 4.1 & 69.2   & \bf 93.0 \\
                                  &                                  & 500   & 1.6 & 1.1 & 66.5   & \bf 92.8 \\
                                  &                                  & 1000  & 1.6 & 1.0 & 66.3   & \bf 92.6 \\ \cline{2-7} 
                                  & \multirow{3}{*}{$\epsilon=0.15$} & 50    & 1.0 & 0.6 & 28.8   & \bf 85.1 \\
                                  &                                  & 500   & 0.5 & 0.1 & 25.0   & \bf 83.6 \\
                                  &                                  & 1000  & 0.4 & 0.1 & 24.8   & \bf 83.5 \\ \hline \hline
        \multirow{6}{*}{CIFAR-10} & \multirow{3}{*}{$\epsilon=0.01$} & 50    & 9.0 & 8.3 & 37.1   & \bf 50.5 \\
                                  &                                  & 500   & 3.5 & 7.8 & 35.8   & \bf 50.3 \\
                                  &                                  & 1000  & 2.9 & 7.8 & 35.7   & \bf 50.2 \\ \cline{2-7} 
                                  & \multirow{3}{*}{$\epsilon=0.02$} & 50    & 0.1 & 0.6 & 10.5   & \bf 15.1 \\
                                  &                                  & 500   & 0.0 & 0.3 & 9.0    & \bf 14.5 \\
                                  &                                  & 1000  & 0.0 & 0.3 & 8.8    & \bf 14.5 \\ \bottomrule
        \end{tabular}}
        \label{tab:convergence}
                                           
        \end{table}
        \subsection{Convergence of PGD Optimization within \texorpdfstring{$\mathcal{L}_{\text{smooth}}$}{Lsmooth}
        Approximation}
            \label{adx:sec-pgd-inner-convergence}
             \begin{table}[!htbp]
                \centering
        \caption{ Robustness of TRS ensemble against various white-box attacks by varying PGD step size $\tilde{\alpha}$ and total number of steps $T$ for solving the inner-maximization within $\mathcal{L}_{\text{smooth}}$ on MNIST dataset.}
        \scalebox{0.9}{
            \begin{tabular}{c|c|c|c|c|c|c}
            \toprule
            Robust acc on MNIST                   & FGSM           & BIM             & PGD             & MIM             & CW            & EAD          \\ \hline
            param.                                & $\epsilon=0.2$ & $\epsilon=0.15$ & $\epsilon=0.15$ & $\epsilon=0.15$ & $c=1.0$       & $c=10.0$     \\ \hline
            DVERGE                                & 91.6           & 47.7            & 28.8            & 44.6            & 79.2          & 0.0          \\ \hline
            TRS ($\tilde{\alpha}=\delta,T=1$)     & 90.5           & 76.0            & 70.1            & 73.3            & 89.6          & 0.1          \\
            TRS ($\tilde{\alpha}=\delta/3,T=6$)   & 91.7           & 85.7            & 85.1            & 85.1            & 92.6          & \textbf{1.4} \\
            TRS ($\tilde{\alpha}=\delta/10,T=20$) & \textbf{92.1}  & \textbf{87.2}   & \textbf{85.5}   & \textbf{86.1}   & \textbf{92.8} & \textbf{1.4} \\ \bottomrule
            \end{tabular}}
            \label{tab:pgd-convergence}
        \end{table}

        Since the computation cost of training a TRS ensemble partially relies on the complexity of PGD procedure on solving the inner-maximization task of $\mathcal{L}_{\text{smooth}}$, we conduct ablation study on analyzing the trade-off between the computation cost (by varying PGD steps $T$) and the resulting robustness of TRS ensemble on MNIST dataset. Specifically, we consider the following settings of PGD step size $\tilde{\alpha}$ and the number of steps $T$:
        \begin{itemize}
            \item [(1)] $\tilde{\alpha}=\delta,T=1$
            \item [(2)] $\tilde{\alpha}=\delta/3,T=6$
            \item [(3)] $\tilde{\alpha}=\delta/10,T=20$
        \end{itemize}

        Results are shown in \Cref{tab:pgd-convergence}. As we can see, the robustness of TRS ensemble consistently improves with the increasing of $T$ and converges. We should also notice that, even for $T=1$, TRS ensemble is more robust than the strongest baseline DVERGE against various strong attacks. Due to the positive correlation between $T$ and the training cost, we should choose suitable $T$ balancing the training cost and model robustness. For MNIST, the default setting ($T=6$) could be a good choice.

\section{Robustness of TRS Ensemble against Other Strong Blackbox Attacks}
\label{adx:other-strong-attacks}
% \usepackage{multirow}
\begin{table}[!htbp]
\centering
\caption{Robust accuracy (\%) of different approaches against strong blackbox transfer attack on MNIST and CIFAR-10 datasets.}
\begin{tabular}{c|c|c|c|c}
\toprule
\multicolumn{2}{c|}{Robust Accuracy}                 & ILA            & DI2-FGSM      & IRA           \\ \hline
\multirow{4}{*}{MNIST ($\epsilon=0.3$)}     & ADP    & 5.4           & 9.9           & 4.8           \\
                                            & GAL    & 3.0           & 8.6           & 7.1           \\
                                            & DVERGE & 89.5          & 91.6          & 82.0          \\
                                            & TRS    & \textbf{91.2} & \textbf{93.7} & \textbf{84.4} \\ \hline\hline
\multirow{4}{*}{CIFAR-10 ($\epsilon=0.05$)} & ADP    & 1.2           & 1.6           & 1.4           \\
                                            & GAL    & 32.2          & 36.2          & 29.2          \\
                                            & DVERGE & 35.9          & 38.3          & 32.4          \\
                                            & TRS    & \textbf{46.2} & \textbf{50.0} & \textbf{45.1} \\ \bottomrule
\end{tabular}
\label{tab:other-blackbox-attacks}
\end{table}

We also conduct additional blackbox robustness evaluation against the following three strong blackbox attacks which focus on attack transferability between surrogate model and target model:
\begin{itemize}[leftmargin=*]

\item \emph{Intermediate Level Attack} (ILA)~\cite{huang2019enhancing} enhances the blackbox attack transferability by taking the perturbation on an intermediate layer of surrogate model into account.
\item \emph{DI2-FGSM}~\cite{xie2019improving} can be viewed as a variant of BIM by applying input transformation randomly at each attack iteration to promote diverse input patterns.
\item \emph{Interaction Reduced Attack} (IRA)~\cite{wang2020unified} integrates an additional interaction loss term after analyzing the negative correlation between attack transferability and interaction between adversarial units. 

\end{itemize}
We use the open-source code mentioned in their original papers and generate blackbox adversarial examples from a surrogate ensemble model consisting of three ResNet20 submodels for both MNIST and CIFAR-10 datasets. We compare the robustness of TRS ensemble with other baseline ensemble. Results are shown in \Cref{tab:other-blackbox-attacks}.

We can find that, TRS ensemble consistently demonstrates the highest robustness compared to other baseline ensembles, which indicates solid blackbox robustness of TRS ensemble against various types of blackbox attacks.

\section{Robustness of TRS Ensemble on CIFAR-100 Dataset}
\label{adx:sec-cifar100-results}
% See Table~\ref{tab:result-cifar100}. \zhuolin{xiaojun maybe you can add more details and illustrations here?}
Besides MNIST and CIFAR-10 datasets, we also evaluate our proposed TRS ensemble on the CIFAR-100 dataset. The base model structure and training parameter configuration remain the same as in CIFAR-10 experiments. The whitebox robustness evaluation results are shown in Table~\ref{tab:result-cifar100}. From the results, we can see that the robustness of TRS model is better than other methods against all attacks except FGSM, which is similar with our observations in CIFAR-10. This shows that our TRS algorithm still achieves a good performance on classification tasks with large number of classes.

\begin{table*}[!htbp]
    \centering
    \caption{\small Robust accuracy$(\%)$ of different ensembles against whitebox attacks on CIFAR-100. ``para.''  refers to the attack parameter ($\epsilon$ is the $\ell_\infty$ perturbation budget for the attack and $c$ the constant to balance the attack stealthiness and effectiveness). }

    \begin{tabular}{c|c|c|c|c|c}
    \toprule
    \bf CIFAR-100 & para. & ADP  & GAL  & DVERGE & \bf TRS  \\ \hline
     \rowcolor{tabgray}  &  $\epsilon=0.02$ & 11.5 & 28.7 & \bf 29.7 & 19.3 \\
                              \rowcolor{tabgray}\multirow{-2}{*}{FGSM}   &  $\epsilon=0.04$ & 6.4 & 2.7 & \bf 25.4 & 9.5 \\ \hline
    \multirow{2}{*}{BIM (50)} & $\epsilon=0.01$ & 0.5 & 7.6 & 12.1 & \bf 22.9 \\
                              & $\epsilon=0.02$ & 0.0 & 1.5 & 2.9 & \bf 5.4 \\ \hline
    \rowcolor{tabgray}& $\epsilon=0.01$ & 0.4 & 5.4 & 11.3 & \bf 23.0 \\
                              \rowcolor{tabgray}\multirow{-2}{*}{PGD (50)} & $\epsilon=0.02$ & 0.0 & 1.1 & 2.0 & \bf 5.3 \\ \hline
    \multirow{2}{*}{MIM (50)} & $\epsilon=0.01$ & 0.5 & 5.7 & 13.1 & \bf 23.4 \\
                              & $\epsilon=0.02$ & 0.0 & 0.5 & 2.6 & \bf 6.2 \\ \hline
    \rowcolor{tabgray}   & $c=0.01$& 11.3 & 32.0 & 44.8 & \bf 45.7 \\
                              \rowcolor{tabgray}\multirow{-2}{*}{CW}    & $c=0.1$ & 0.5 & 10.7 & 20.3 & \bf 26.9 \\ \hline
    \multirow{2}{*}{EAD}      & $c=1.0$ & 0.0 & 0.0 & 1.0 & \bf 5.7 \\
                              & $c=5.0$ & 0.0 & 0.0 & 0.0 & \bf 0.3 \\ \hline
    \rowcolor{tabgray} & $\epsilon=0.01$ & 0.2 & 4.3 & 11.8 & \bf 22.2 \\
                              \rowcolor{tabgray}\multirow{-2}{*}{APGD-LR} & $\epsilon=0.02$ & 0.0 & 0.6 & 2.1 & \bf 5.3 \\ \hline
    \multirow{2}{*}{APGD-CE}  & $\epsilon=0.01$ & 0.2 & 4.2 & 11.3 & \bf 20.7 \\
                              & $\epsilon=0.02$ & 0.0 & 0.4 & 1.7 & \bf 4.8 \\ \bottomrule
    \end{tabular}
        \label{tab:result-cifar100}
\end{table*}

\end{document}